\documentclass{article} 
\usepackage{iclr2027_conference,times}

\usepackage{amsmath,amsfonts,bm}

\def\eqref#1{equation~\ref{#1}}

\def\1{\bm{1}}

\DeclareMathAlphabet{\mathsfit}{\encodingdefault}{\sfdefault}{m}{sl}
\SetMathAlphabet{\mathsfit}{bold}{\encodingdefault}{\sfdefault}{bx}{n}

\newcommand{\normlone}{L^1}

\usepackage{hyperref, url, algorithm, amsthm, algpseudocode, bbm, graphicx, subcaption}

\newtheorem{theorem}{Theorem}
\newtheorem{assumption}{Assumption}

\makeatletter
\newcommand{\dgplabel}[2]{%
  \begingroup
  \def\@currentlabel{#2}%
  \label{#1}%
  \endgroup
}
\makeatother

\title{Conformal Prediction and Conditional Coverage for Tabular Foundation Models}

\author{Sungwoo Park\thanks{Equal contribution.}
\quad
Sunghee Park\footnotemark[1]
\quad
Won Chang\thanks{Corresponding author.} \\
Department of Statistics \\
Seoul National University \\
Seoul, Republic of Korea \\
\texttt{\{park.sw,psh2002,wonchang\}@snu.ac.kr}
}

\iclrfinalcopy
\fancypagestyle{preprint}{%
  \fancyhf{}
  \fancyfoot[L]{\footnotesize Preprint.}
  \fancyfoot[C]{\thepage}
  \renewcommand{\headrulewidth}{0pt}
  \renewcommand{\footrulewidth}{0pt}
}
\hypersetup{pdftitle={Conformal Prediction and Conditional Coverage for Tabular Foundation Models}, pdfauthor={Sungwoo Park, Sunghee Park, Won Chang}}
\begin{document}

\maketitle
\thispagestyle{preprint}

\begin{abstract}
Tabular foundation models (TFMs) provide predictive distributions for regression, but their prediction regions can exhibit undercoverage or overcoverage even when point predictions are accurate. We introduce C-USIM (Conditionally-Uniformized Score Integration Method), a lightweight application of highest predictive density split conformal prediction that accommodates multimodal predictions. Given calibration and test outputs, it requires no additional training or model inference. It provides finite-sample marginal validity under our assumptions. We bound conditional--marginal coverage gaps using distribution-estimation error and score discreteness, and examine coverage heterogeneity through percentile rank--score plots. Experiments with TabPFN and TabICL show improved marginal coverage accuracy and lower average conditional and group coverage errors. Under a fixed data budget, allocating more observations to calibration can reduce marginal coverage error despite less accurate point predictions.

\end{abstract}

\section{Introduction}
Tabular foundation models (TFMs) built on prior-data fitted networks (PFNs) approximate posterior predictive distributions (PPD) through pretraining on synthetic data and in-context learning from a labeled context \citep{Muller2024-zz}. Mismatch between pretraining and downstream distributions, together with approximation errors, can overestimate or underestimate these predictions despite accurate point estimates, motivating calibration and coverage assessment across inputs.

We introduce C-USIM (Conditionally-Uniformized Score Integration Method), a split conformal procedure that calibrates highest predictive density (HPD) regions from TFM outputs. HPD regions can comprise disjoint intervals for multimodal predictions. Uncalibrated plug-in HPD can under- or overcover when the predictive distribution is inaccurate. C-USIM uses the established density-rank score and split-conformal calibration principle of HPD-split \citep{izbickicdhpd}, setting its threshold to an empirical quantile of held-out scores. Calibration requires no additional training, model inference, or separate density or score-correction model.

Under exchangeability, calibration guarantees finite-sample marginal coverage of at least the target but can leave systematic undercoverage in some covariate regions and overcoverage in others \citep{vovkconformal}. We analyze conditional--marginal coverage gaps through distribution-estimation error and score discreteness, using percentile rank--score plots to show coverage improvements after calibration and remaining variation across inputs.

Under a fixed label budget, we compare reserving labels for calibration with using all labels as context for uncalibrated plug-in HPD. TabPFN and TabICL experiments show improved marginal coverage accuracy, lower mean absolute conditional coverage error across synthetic mechanisms, and lower aggregate group coverage errors in most real-world comparisons. We further examine the training--calibration ratio within C-USIM and find that allocating more labels to calibration can improve coverage accuracy despite less accurate point predictions. Additional comparisons using identical TFM predictive outputs and data splits show that C-USIM achieves shorter mean total prediction-set length than CQR and CHR in the bimodal settings with both models, while all three methods maintain mean coverage close to the target.

\begin{figure}[!t]
\centering
\includegraphics[width=\linewidth]{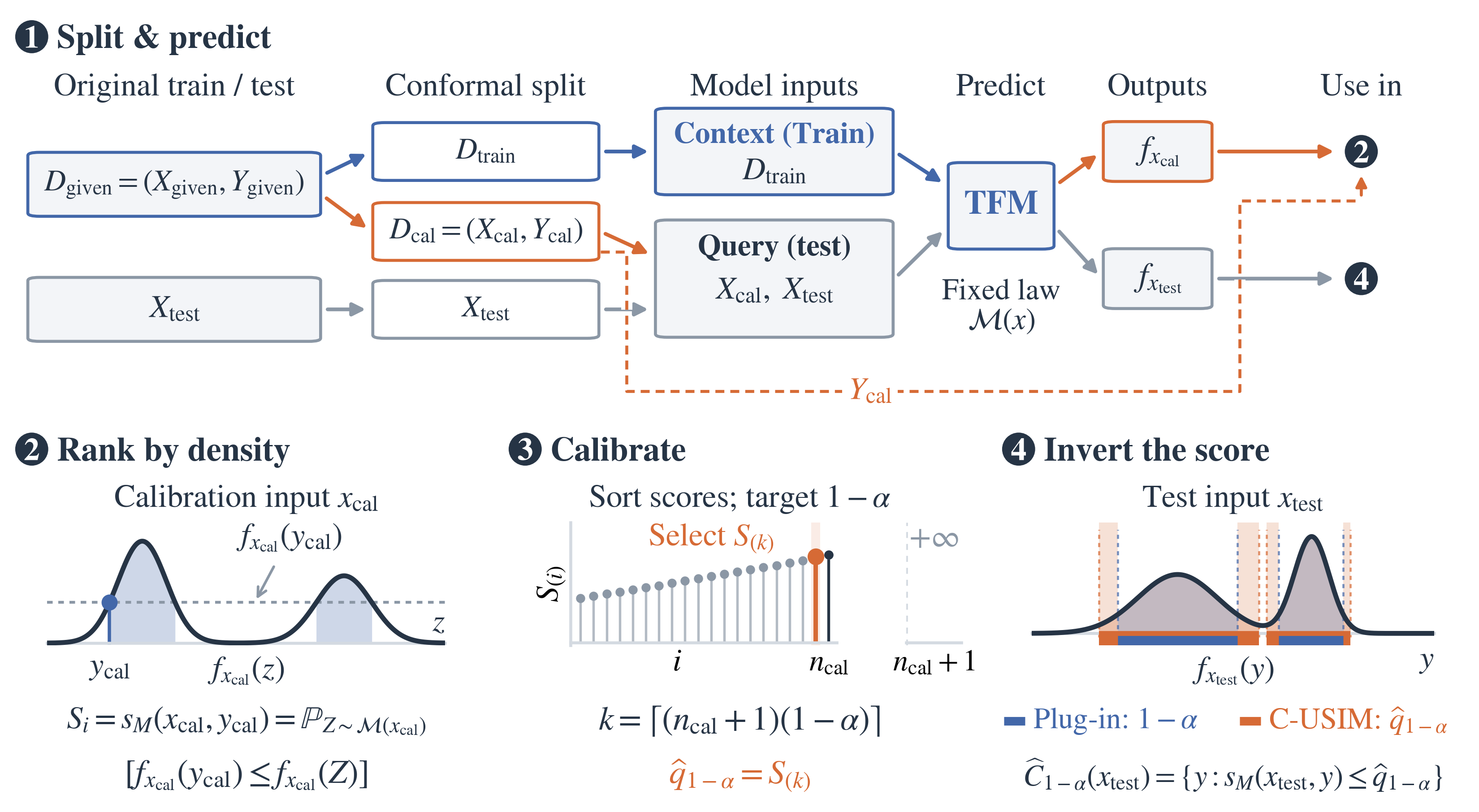}
\caption{C-USIM overview with $n_{\mathrm{cal}}$ calibration scores and target coverage $1-\alpha$. Using a pretrained TFM, the procedure predicts densities, computes density-rank scores, calibrates a threshold, and constructs prediction regions without fitting an additional model. See Section~\ref{sec:method} for details.}
\label{fig:c-usim-overview}
\end{figure}

We make three contributions.
\textbf{First,} we formulate C-USIM as an application of the established HPD-split framework to TFM bin-probability and quantile outputs, calibrating potentially disjoint prediction regions without additional model fitting or inference.
\textbf{Second,} under the stated assumptions, we bound conditional--marginal coverage gaps using bin-probability estimation error and score discreteness, and relate a shared threshold to coverage at individual inputs through our percentile rank--score representation.
\textbf{Third,} our experiments under a fixed label budget provide evidence of coverage improvements over plug-in HPD and guidance on context--calibration allocation: more calibration observations can improve marginal coverage accuracy despite less accurate point predictions.

\section{Preliminaries}
\subsection{Prior-data Fitted Network}
Given a training dataset $D_{\mathrm {train}} = \{(x_1, y_1), \cdots, (x_{n_{\mathrm {train}}}, y_{n_{\mathrm {train}}})\}$, we estimate the conditional distribution of the response $Y$ at $X = x_{\mathrm{test}}$. In the Bayesian framework for supervised learning, let $\phi$ denote a data-generating mechanism in a space $\Phi$, sampled from the prior $p(\phi)$, with datasets $D$ generated independently conditional on $\phi$. The PPD $p(\cdot ~|~ x_{\mathrm {test}}, D_{\mathrm {train}})$ marginalizes over the posterior of $\phi$:

$$
p(y ~|~x, D)\propto \int_{\Phi} p(y~|~x, \phi)p(D~|~\phi)p(\phi)d\phi
$$

Prior-data Fitted Networks (PFNs) use Transformers pretrained on synthetic datasets to approximate this Bayesian inference, predicting the PPD from a training context. In univariate regression, where $y\in \mathbb{R}$, TabPFN learns probability masses over response bins using cross-entropy loss \citep{Grinsztajn2026-ei}, whereas TabICL estimates conditional quantiles using pinball loss \citep{Qu2026-rv}.

\subsection{Plug-in HPD Regions}
\label{sec:plug-in-hpd}
Let $f_x$ denote an estimated conditional density of $Y$ given $X=x$. For $0<\alpha<1$, the plug-in HPD region is defined by the highest-density-region construction \citep{hyndman1996computing}:
\begin{align*}
    \widehat{C}^{\mathrm{plug}}_{1-\alpha}(x)
    := \{y:f_x(y)\ge\lambda_\alpha(x)\}, &&
    \lambda_\alpha(x)
    := \sup\left\{t>0:\int_{\{z:f_x(z)\ge t\}}f_x(z)\,dz\ge 1-\alpha\right\}.
\end{align*}
The cutoff $\lambda_\alpha(x)$ retains at least $1-\alpha$ predictive mass at the largest possible density threshold. For piecewise-constant densities, whole positive-density levels are included in decreasing order until their cumulative mass first reaches or exceeds the target, retaining all boundary ties. These potentially disconnected regions require no calibration, and their predictive mass does not guarantee coverage under the true data distribution.

\subsection{Split Conformal Prediction}

Split conformal prediction (CP) guarantees finite-sample marginal coverage under exchangeability \citep{vovkconformal}. After fitting a model on training data, a held-out calibration set $\{(X_i,Y_i)\}_{i=1}^{n_{\mathrm{cal}}} $ and a nonconformity rule $s(x,y)$ yield scores $ \{S_i=s(X_i,Y_i)\}_{i=1}^{n_{\mathrm{cal}}}$. For target coverage $1-\alpha$, define $ \widehat{C}_{1-\alpha}(x) := \left\{ y:s(x,y)\le \widehat{q}_{1-\alpha} \right\} $ where \(\widehat{q}_{1-\alpha}\) is the \(\lceil(n_{\mathrm{cal}}+1)(1-\alpha)\rceil\)th order statistic of \(\{S_i\}_{i=1}^{n_{\mathrm{cal}}}\cup\{\infty\}\). However, nontrivial distribution-free guarantees of exact conditional coverage are generally impossible in finite samples without additional assumptions \citep{barber2021limits}. 

For a real-valued random variable $W$ with CDF $F_W(w)$, the probability integral transform (PIT) $U:=F_W(W)$ is uniform when $F_W$ is continuous: $U\sim\mathrm{Unif}(0,1)$. A score is \textit{pivotal} if its conditional law is independent of the covariate. For $S=s(X,Y)$ with conditional CDF $F_{S\mid X=x}$, if $F_{S\mid X=x}$ is continuous for every $x$, the oracle conditional PIT correction $ \tilde{s}(x,y):=F_{S\mid X=x}\!\left(s(x,y)\right) $ satisfies $ \tilde{s}(X,Y)\mid X=x\sim\mathrm{Unif}(0,1). $ Since this conditional distribution is the same across covariates, the corrected score is pivotal. The region $\{y:\tilde{s}(x,y)\le 1-\alpha\}$ therefore has conditional coverage $1-\alpha$ for every $x$ \citep{laplante2026postprocessingconformalpredictionapproach}.

\section{Related work}
Conformal prediction provides distribution-free marginal coverage guarantees \citep{gammerman1998learning, saunders1999transduction, vovkconformal}, with extensive work on regression \citep{shafer2008tutorial, lei2014distribution}. HPD-split uses density-based scores to construct adaptive prediction regions for heteroscedastic, skewed, or multimodal conditional distributions \citep{izbickicdhpd}. Its HPD-based score is also studied in the C-HDR formulation alongside a broader family of CDF-based conformity scores \citep{dheur2025unified}.

Conformalized quantile regression (CQR) calibrates residuals around estimated lower and upper conditional quantiles to form adaptive prediction intervals \citep{romano2019conformalizedquantileregression}. Conformalized histogram regression (CHR) calibrates a nested family of shortest contiguous intervals from conditional histograms using region probability mass as the conformity score; its interval construction can affect length and computational cost \citep{sesia2021conformalpredictionusingconditional}. When applied to the same TFM predictive distribution, CQR and CHR return contiguous intervals, which include the intervening low-density region whenever they cover separated modes. C-USIM instead permits disjoint high-density regions, allowing such gaps to be excluded and adapting prediction-set geometry to multimodal predictive densities. This flexibility allows C-USIM to handle broader families of distributions, a capability that is particularly important for TFMs.

JAPAN thresholds densities estimated by normalizing flows, while DSPS combines conditional normalizing flows with density-based ranking and conformal calibration \citep{english2026japan, luo2025density}. PIT-CP-MDN and PIT-CP-CNF fit additional models for conditional score distributions using mixture density networks and conditional normalizing flows, respectively, to seek approximate conditional coverage for different base scores \citep{laplante2026postprocessingconformalpredictionapproach}.

Prior work has already applied CP to TFMs using absolute-error scores for regression and adaptive prediction-set scores for classification \citep{leeuwen2025conformal, costa2026high, kenfack2025towards}. C-USIM applies the established HPD-split construction and its PIT interpretation directly to TFM bin-probability and quantile outputs, without fitting an additional density or score-correction model. Our contributions concern conditional-coverage analysis, percentile rank--score diagnostics, and evaluation of context--calibration allocation under a fixed label budget.

\section{C-USIM: Conditionally-Uniformized Score Integration Method}
\label{sec:method}

C-USIM constructs prediction regions by reconstructing densities from TFM outputs, computing density-rank scores, and calibrating a threshold on held-out observations.

\subsection{Setup and notation}
For a random variable $W$, we denote its cumulative distribution function and probability density function, when it exists, by $F_W$ and $f_W$, respectively.

In this paper, we consider $d$-dimensional covariates and a univariate response. We denote the training and calibration datasets by $D_{\mathrm{train}}$ and  $D_{\mathrm{cal}}$, consisting of $n_{\mathrm{train}}$ and $n_{\mathrm{cal}}$ covariate-response pairs in $\mathbb{R}^d \times \mathbb{R}$, respectively, drawn independently from a common distribution. We also denote a PFN model by $\mathcal{M}(\cdot, D_{\mathrm{train}}): \mathbb{R}^d \to \mathcal{P}(\mathbb{R})$, which uses $D_{\mathrm{train}}$ as context and maps each covariate $x$ to an estimator of the conditional distribution of $Y$ given $X=x$, assuming that this conditional distribution exists. Unless otherwise stated, we omit the second argument and write $\mathcal{M}(x)$ for $\mathcal{M}(x, D_{\mathrm{train}})$.

\begin{assumption}[Row-permutation equivariance]
\label{ass:row-per-equiv}
For fixed $D_{\mathrm{train}}$, let $\mathcal{M}(\mathbf{x})$ denote the ordered predictive distributions returned jointly for the calibration and test covariates $\mathbf{x}=(x_1,\ldots,x_m)$. We assume $ \mathcal{M}(\pi\mathbf{x})=\pi\mathcal{M}(\mathbf{x}) $ for every row permutation $\pi$.
\end{assumption}

\subsection{General Setting of C-USIM}
\label{sec:c-usim-construction}
For $x\in \mathbb{R}^d$, our goal is to construct $\widehat{C}_{1-\alpha}(x)$ from density predictions with the following properties:
\begin{itemize}
\item $\mathbb{P}_{(X, Y)} (Y\in \widehat{C}_{1-\alpha}(X)) \ge 1-\alpha$, i.e., marginal coverage achieves the target.
\item The total length of the prediction region $\widehat{C}_{1-\alpha}(X)$ is as small as possible.
\item Conditional coverage $\mathbb{P}_{Y}(Y\in \widehat{C}_{1-\alpha}(x)|X = x)$ concentrates around $1-\alpha$.
\end{itemize}

Under the stated assumptions, split-conformal calibration provides a finite-sample guarantee for the first property. Conditional coverage accuracy and prediction-set length are performance objectives assessed through our analysis and experiments.

C-USIM uses the CP framework with the nonconformity score $s_{\mathcal{M}}$ obtained by applying the model-based PIT to the negated predictive density. Specifically, we take $s_0(x,y)=-f_{\mathcal{M}(x)}(y)$ as the base score and evaluate its CDF under $Z\sim\mathcal{M}(x)$ at $s_0(x,y)$. Since $s_0(x,Z)\le s_0(x,y)$ is equivalent to $f_{\mathcal{M}(x)}(Z)\ge f_{\mathcal{M}(x)}(y)$, this gives
\begin{align*}
    s_{\mathcal{M}}(x, y) := \mathbb{P}_{Z\sim \mathcal{M}(x)}(f_{\mathcal{M}(x)} (y) \le f_{\mathcal{M}(x)}(Z))
\end{align*}

\begin{algorithm}[!t]
\caption{C-USIM prediction-set construction}
\label{alg:c-usim}
\begin{algorithmic}[1]
\Require Pretrained TFM $\mathcal{M}$, labeled data $D_{\mathrm{given}}$,
calibration size $n_{\mathrm{cal}}$, test covariates $\mathbf{x}_{\mathrm{test}}$,
and miscoverage level $\alpha\in(0,1)$.
\Ensure Prediction sets $\widehat{C}_{1-\alpha}(x)$ for each test covariate $x$.
\State Split $D_{\mathrm{given}}=D_{\mathrm{train}}\,\dot\cup\,D_{\mathrm{cal}}$
with $|D_{\mathrm{cal}}|=n_{\mathrm{cal}}$.
\State Use $D_{\mathrm{train}}$ as labeled context for $\mathcal{M}$.
\State Query the calibration and test covariates jointly, as in Assumption~\ref{ass:row-per-equiv}.
\State Construct densities $f_x:=f_{\mathcal{M}(x)}$ from the outputs (Appendix~\ref{app:density-estimation}).
\For{each $(X_i,Y_i)\in D_{\mathrm{cal}}$}
    \State $S_i\gets\mathbb{P}_{Z\sim\mathcal{M}(X_i)}\!\left(f_{X_i}(Y_i)\le f_{X_i}(Z)\right)$.
\EndFor
\State $k\gets\lceil(n_{\mathrm{cal}}+1)(1-\alpha)\rceil$.
\State $\widehat{q}_{1-\alpha}\gets$ the $k$-th smallest value in
$\{S_i\}_{i=1}^{n_{\mathrm{cal}}}\cup\{\infty\}$.
\State \Return $\widehat{C}_{1-\alpha}(x)=\{y:s_{\mathcal{M}}(x,y)\le\widehat{q}_{1-\alpha}\}$
for each $x\in\mathbf{x}_{\mathrm{test}}$.
\end{algorithmic}
\end{algorithm}

Algorithm~\ref{alg:c-usim} gives the complete prediction-set construction. Density-level ordering, cumulative probability masses, and an empirical calibration quantile determine the regions without additional model fitting or inference.

\paragraph{Models used in this work} For all analyses and subsequent experiments, we use TabPFN-v3\footnote{\url{https://github.com/priorlabs/tabpfn}} and TabICL-regressor-v2\footnote{\url{https://github.com/soda-inria/tabicl}} with the default inference settings of their respective implementations. Hereafter, we refer to these models simply as TabPFN and TabICL, respectively. For both models, we assume the permutation-equivariance condition in Assumption~\ref{ass:row-per-equiv} for calibration and test rows, which is natural for their Transformer architectures with suitably masked test attention.

\section{Theoretical Analysis of Conditional Coverage for C-USIM}
\label{sec:conditional-coverage}

C-USIM uses a density-rank score designed to be pivotal in the oracle setting: with the true conditional density and continuity of the induced density-score distribution, this construction yields the same uniform score distribution across inputs. In practice, distribution-estimation error and the discrete scores induced by finite-resolution TFM outputs can prevent pivotality. Split-conformal calibration provides finite-sample marginal validity under the stated assumptions, while exact conditional coverage is not generally guaranteed. We therefore analyze conditional--marginal coverage gaps and use percentile rank--score plots to examine coverage heterogeneity across inputs.

In this section and its proofs (Appendix~\ref{proof}), we use a stronger assumption than Assumption~\ref{ass:row-per-equiv}.

\begin{assumption}[Test-set invariance]
\label{ass:preserve-indep}
For fixed $D_{\mathrm{train}}$, consider any two test-covariate datasets $\mathbf{x}_1=(x_{11},\ldots,x_{1m})$ and $\mathbf{x}_2=(x_{21},\ldots,x_{2n})$. For any indices $p$ and $q$ satisfying $x_{1p}=x_{2q}$, we assume that
\begin{align*}
\mathcal{M}(x_{1p}, D_{\mathrm{train}}; \mathbf{x}_1)
=
\mathcal{M}(x_{2q}, D_{\mathrm{train}}; \mathbf{x}_2),
\end{align*}
where $\mathcal{M}(x, D_{\mathrm{train}}; \mathbf{x})$ denotes the response distribution returned for $x$ when prediction is performed on the test dataset $\mathbf{x}$.
\end{assumption}

We assume that the calibration and test observations are i.i.d.; under Assumption~\ref{ass:preserve-indep}, their scores are therefore i.i.d. evaluations of a fixed scoring rule.


\subsection{Conditional Coverage Gap Bound}
\label{sec:conditional-gap}
The conditional coverage gap of conformal prediction at $x$, with respect to the scoring rule $s$, is defined as the equation below \citep{laplante2026postprocessingconformalpredictionapproach}:
\begin{align*}
    \Delta(x) := \sup_{\alpha \in (0, 1)} \left|\mathbb{P}\left(Y\in \widehat{C}_{1-\alpha}(x)~ |~ X =x\right) - \mathbb{P}\left(Y\in \widehat{C}_{1-\alpha}(X)\right)\right|
\end{align*}

We proved that this can be bounded in terms of the discrete $\normlone$ error of the bin probability vectors.

\begin{theorem}[Conditional coverage gap bound for the PIT-transformed score]
\label{thm:pitl1}
Fix $x\in\mathbb{R}^d$, and let
\begin{align*}
    -\infty = \hat{y}_0^{(x)} < \hat{y}_1^{(x)} < \cdots < \hat{y}_{r+1}^{(x)}
    = \infty.
\end{align*}
Suppose that the nonconformity score $s_{\mathcal{M}}(x,\cdot)$ is constant on each interval $[\hat{y}_{i-1}^{(x)},\hat{y}_i^{(x)})$, with constant value $S_i^{(x)}$, and that these values are distinct across intervals. Let
\begin{align*}
    \hat p_i^{(x)} := \mathbb{P}_{Z_x\sim\mathcal{M}(x)} \left(Z_x\in [\hat{y}_{i-1}^{(x)},\hat{y}_i^{(x)}) \right), \quad \tilde p_i^{(x)}
:=
\mathbb{P}
\left(
Y\in
[\hat{y}_{i-1}^{(x)},\hat{y}_i^{(x)})
\mid X=x
\right),
\end{align*}
and define
\begin{align*}
B(x)
:=
\frac{1}{2}
\left\|
\tilde p^{(x)}-\hat p^{(x)}
\right\|_1
+
\max_{1\le i\le r+1}
\hat p_i^{(x)}.
\end{align*}
Here, $\tilde{p}^{(x)} = (\tilde{p}_1^{(x)}, \tilde{p}_2^{(x)}, \dots, \tilde{p}_{r+1}^{(x)}), ~ \hat{p}^{(x)} = (\hat{p}_1^{(x)}, \hat{p}_2^{(x)}, \dots, \hat{p}_{r+1}^{(x)})\in \mathbb R^{r+1}.$

Let $\tilde s$ be the PIT-transformed scoring rule induced by $s_{\mathcal{M}}$, and let $\Delta_{\tilde s}(x)$ denote the conditional coverage gap of conformal prediction constructed with $\tilde s$. Then,
\begin{align*}
    \Delta_{\tilde s}(x)
    \le
    \mathbb{E}_{X}\!\left[B(X)\right] + B(x).
\end{align*}
\end{theorem}

This result characterizes a trade-off between model complexity and the performance of CP with a PIT-corrected score. As $r$ increases, the partition becomes finer and can reduce the discretization term $\max_{1\le i\le r+1}\hat p_i^{(x)}$. However, estimating the resulting higher-dimensional probability vector may increase the $L^1$ error term. Conversely, using a smaller $r$ can yield a more stable probability estimate but incurs a larger discretization error.

Since this analysis yields a global bound for $\alpha\in(0, 1)$, it does not characterize the distribution of conditional coverage over $X$. To extend the discussion, in the following section, we use the percentile rank--score plot to visualize and examine the distribution of conditional coverage across inputs empirically.

\subsection{Conditional Coverage Through Percentile Rank--Score Plots}
\label{sec:percentile-score}

\begin{assumption}[Continuous score distribution]
\label{ass:conti-dist}
For the results in this section, we assume that the C-USIM score $s_{\mathcal M}(x,Y^{(x)})$ has a continuous distribution, where $Y^{(x)}\sim Y\mid X=x$ for every covariate $x$.
\end{assumption}

Let $ T(x, y) := s_{\mathcal M}(x, y),  P(x, y) := F_{s_\mathcal M (X, Y) |X = x}(T(x, y))$ denote the score and its conditional percentile rank, respectively. Let $Z$ be a conditionally independent copy of $Y$ given $X$, i.e., $(Z\mid X = x ) \overset{d}{=} (Y\mid X=x)$ and $Z\perp\!\!\!\perp Y\mid X$.

Under Assumption~\ref{ass:conti-dist}, $P(x, y)= \mathbb{P}(s_\mathcal M (x, Z) < s_\mathcal M (x, y) ~|~ X = x) = \mathbb{P}(Z\in \hat{C}_{T(x, y)}(x) | X = x)$. Hence, $P(x,y)$ can be interpreted as the true conditional coverage obtained when $T(x,y)$ is used as the score threshold. Moreover, by PIT, $P(X, Y) ~|~ X =x \sim \operatorname{Unif}(0,1)$ and consequently $P(X, Y)\sim\operatorname{Unif}(0,1)$.

For a fixed score threshold $t$, define
\begin{align*}
P_t(x) := F_{s_{\mathcal M}(X,Y)\mid X = x}(t) = \mathbb P\left(
s_{\mathcal M}(X,Y)\le t \mid X=x
\right) =
\mathbb P\left(
Y\in\widehat C_t(X)
\mid X = x
\right).
\end{align*}
Thus, $P_t(x)$ is the conditional coverage at covariate vector $x$ obtained by applying $t$ as the score threshold. Thus, $P_t (X)$ describes the distribution of conditional coverage associated with $\widehat C_t$.

In C-USIM with target coverage $1-\alpha$, let $\{T_i\}_{i=1}^{n_{\mathrm {cal}} + 1}$ denote the calibration scores, where $T_i = T(X_i, Y_i)$ for $i=1,\ldots,n_{\mathrm {cal}}$ and $T_{n_{\mathrm {cal}}+1}=\infty$. The calibration threshold $\hat{\tau}$ is then set as
\begin{align*}
\hat{\tau}:=T_{(k)},\qquad
k=\left\lceil (n_{cal}+1)(1-\alpha)\right\rceil.
\end{align*}

This is the $k$-th order statistic of $\{T_i\}_{i=1}^{n_{\mathrm {cal}} + 1}$. Then, the distribution of C-USIM's conditional coverage $C$ can be written as $C = P_{\hat{\tau}}(X)$.

Under Assumptions~\ref{ass:preserve-indep} and~\ref{ass:conti-dist}, if $n_{\mathrm {cal}} \rightarrow \infty$, the threshold $\hat{\tau}$ converges in probability to $q_{1-\alpha}$, which is the $(1-\alpha)$ quantile of the marginal score distribution, and $P_{\hat{\tau}}(X)$ converges in distribution to $P_{q_{1-\alpha}}(X)$. The conditional coverage distribution of C-USIM in the ideal case, where the score distribution is uniform across covariate values, is as follows.

\begin{theorem} [Conditional coverage distribution in the ideal case] \label{thm:idealcc}
Under Assumptions~\ref{ass:preserve-indep} and~\ref{ass:conti-dist}, suppose that $\alpha > (n_{\mathrm {cal}}+1)^{-1}$ and $T(X, Y) = f(P(X, Y))$ for some strictly increasing function $f$ almost surely. Then C-USIM's conditional coverage $C$ follows $\operatorname{Beta}(k, n_{\mathrm {cal}}+1-k)$.
\end{theorem}

For each test covariate $X_i$, consider the percentile rank--score curve $(P(X_i,y), T(X_i,y))$. For a given threshold $t$, $P_t(X_i)$ is the horizontal coordinate at which this curve intersects the horizontal line $T=t$. More generally, $P_t(X)$ describes the conditional-coverage distribution at threshold $t$. Figure~\ref{fig:TabPFNPLotCDExample} illustrates these percentile rank--score plots (a) and the resulting conditional-coverage distributions (b) in an example setting.

\begin{figure}[htbp]
\centering
\includegraphics[width=.96\linewidth]{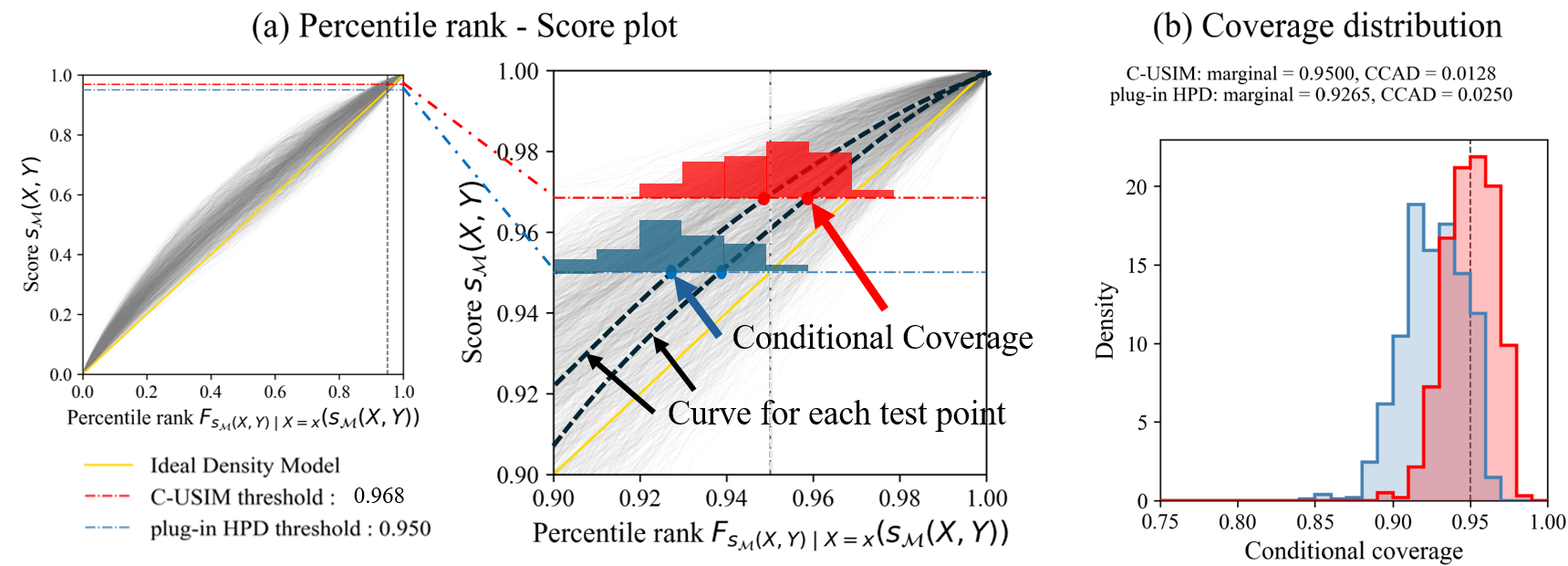}
\caption{ Percentile rank--score plots over test covariates (plot (a)), obtained using TabPFN on the sine with shifted exponential noise function (\ref{dgp:1d-1}). Plot (b) approximates $P_t(X_i)$ at the nominal score threshold $t=1-\alpha$ (blue) and the ideal C-USIM threshold $t=q_{1-\alpha}$ (red), which is estimated by the empirical $(1-\alpha)$-quantile of the pooled Monte Carlo scores. The histogram summarizes interpolated estimates of $P_t(X_1),\dots,P_t(X_{n_{test}})$, approximating the distribution of $P_t(X)$. See Appendices~\ref{app:synthetic-dgps} and~\ref{app:percentile-score} for details. }
\label{fig:TabPFNPLotCDExample}
\end{figure}

The diagonal $T=P$ indicates agreement between the score and its conditional percentile rank given $X=x$. At the nominal threshold, curves above (below) the diagonal indicate undercoverage (or overcoverage). For fixed predictive distributions, the calibration shifts the horizontal line while leaving the curves unchanged. In this example, C-USIM moves this line toward $q_{1-\alpha}\approx 0.968$. If the model satisfies Assumption~\ref{ass:preserve-indep}, this is the point to which its calibrated threshold converges as the calibration sample size increases. Meanwhile, plug-in HPD uses the fixed score threshold $0.95$. Figure~\ref{fig:TabPFNPLotCDExample}(a) shows the intersections, and Figure~\ref{fig:TabPFNPLotCDExample}(b) summarizes their horizontal coordinates ($P_{0.968}(X_i)$ for ideal C-USIM and $P_{0.95}(X_i)$ for plug-in HPD).

This representation also explains the roles of training and calibration data. Increasing the training sample aims to improve the predictive distributions and bring the curves closer to $T=P$, whereas more calibration data can stabilize the estimated threshold. For TabPFN and TabICL, we observed diminishing gains beyond a certain training size. In this regime, reallocating observations from training to calibration improved coverage more than further enlarging the training set.  More illustrations of percentile rank--score plots and conditional coverage distributions using TabPFN and TabICL are attached in Appendix \ref{app:percentile-score}. Section~\ref{sec:split-ratio} examines this trade-off under a fixed label budget.

\section{Experiments}
\label{sec:experiments}
We compare C-USIM with plug-in HPD on synthetic and real-world regression tasks, evaluating coverage accuracy and prediction-set length.

\paragraph{Calibrating method.} Throughout all experiments, we obtain distribution estimates for the calibration and test covariates by providing the row-merged dataset $D_{\mathrm{cal}} \dot\cup D_{\mathrm{test}}$, following Algorithm~\ref{alg:c-usim}. Because TabPFN and TabICL produce outputs in different formats, we use model-specific reconstruction rules to convert them into the finite piecewise-uniform conditional densities described in Appendix~\ref{app:density-estimation}. For the plug-in HPD baseline, we include the entire boundary density level required to attain at least \(95\%\) model mass, following Section~\ref{sec:plug-in-hpd}. In contrast, C-USIM uses the exact score sublevel set determined by its calibration order statistic, following Algorithm~\ref{alg:c-usim}. Appendix~\ref{app:settings} provides the common settings, random seeds, and grouping scheme.

\paragraph{Evaluation measures.} We report marginal coverage and total prediction-set length alongside conditional coverage absolute deviation (CCAD) for synthetic data and CEC-X for real data. For synthetic data with known true conditional densities, CCAD measures mean absolute conditional coverage error across inputs. For real data, where these densities are unknown, CEC-X averages absolute empirical coverage errors across covariate-defined groups, weighted by their test sizes. These measures correspond to the three goals in Section~\ref{sec:c-usim-construction}: marginal validity, conditional coverage close to the target, and small prediction sets. Appendix~\ref{app:evaluation-measures} provides the formal definitions for these.

\paragraph{Fixed-context comparison.} Our main comparison uses a total budget of 1,536 labels: plug-in HPD uses all labels as context, while C-USIM uses 512 for context and 1,024 for calibration. To examine the effect of calibrating the cutoff with the predictive densities held fixed, we also compare C-USIM with plug-in HPD using the same 512-observation context. Only C-USIM uses the additional 1,024 calibration responses in this comparison. We report this fixed-context comparison as a supplementary analysis in Appendix~\ref{app:plugin512}, since our primary comparison concerns context--calibration allocation under a common total label budget.

\subsection{Function-based synthetic examples}
\label{sec:syn-ex}

We compare plug-in HPD and C-USIM using the settings in Table~\ref{tab:experiment-settings} in Appendix~\ref{app:settings}. We consider three one-dimensional, three multidimensional, and three bimodal synthetic examples, described in Appendix~\ref{app:synthetic-dgps}. These examples assess coverage with oscillatory or discontinuous response functions, skewed noise, covariate-dependent scale mixtures, and separated modes. For each model, we report averages across all nine mechanisms.

\begin{figure}[!t]
\centering
\includegraphics[width=\linewidth]{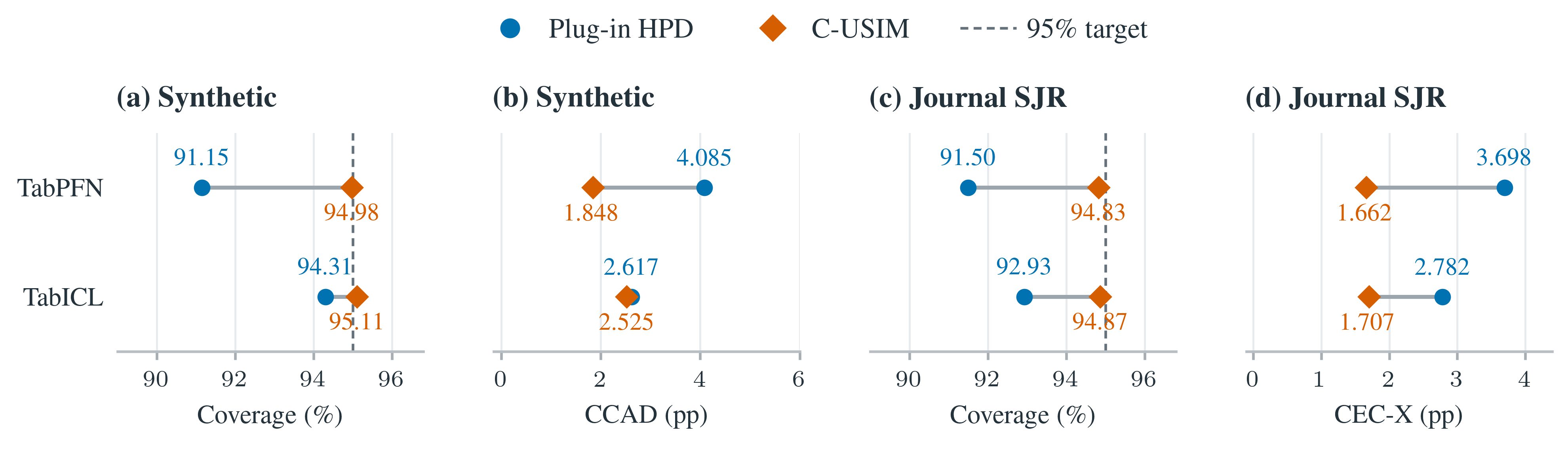}
\caption{Synthetic and Journal SJR overview. (a,b) Synthetic seed means, averaged equally over all nine mechanisms. (c,d) Journal SJR seed means of marginal coverage and CEC-X, respectively; (d) uses a representative covariate grouping. See Table~\ref{tab:experiment-settings} for experiment settings and Appendix~\ref{app:detailed-results} for detailed results.}
\label{fig:experiment-overview}
\end{figure}

Figure~\ref{fig:experiment-overview}(a,b) summarizes results from all nine synthetic mechanisms. C-USIM brings mean marginal coverage closer to the target and reduces mean CCAD for both models, with a larger reduction for TabPFN. The lower mean CCAD shows that the gains extend to average conditional coverage accuracy across these mechanisms. Under the same total label budget, these results support reserving labels for calibration when constructing prediction regions from TFM outputs. Appendix~\ref{app:synthetic-results} gives the case-specific seed distributions, numerical results, and set lengths.

We additionally compare C-USIM with CQR, CHR, and CTI using identical TFM predictive outputs and data splits (Appendix~\ref{app:calibrated-baselines}, Table~\ref{tab:synthetic-calibration-rev}). Across the bimodal settings with both TabPFN and TabICL, C-USIM produces substantially shorter prediction sets than CQR and CHR, reducing mean total set length by 15.4--56.4\% and 15.6--57.0\%, respectively, while all four calibrated methods maintain mean coverage close to the 95\% target. Mean set lengths are comparable to CTI in these bimodal settings. However, with $K=100$, CTI yields markedly inflated mean lengths on TabPFN \ref{dgp:1d-2} and \ref{dgp:md-1}: 5.588 and 22.112, versus 1.290 and 2.100 for C-USIM. The appendix examines this resolution sensitivity and CTI's generally non-pivotal score. C-USIM uses the existing TFM density representation without an additional resolution parameter or further model fitting.

\subsection{Real-world regression examples}

We evaluate C-USIM on three real-world datasets: Journal SJR, JP Anime, and Allstate Claims Severity. Journal SJR, from the CARTE benchmark \citep{kim2024carte}, contains publication metadata and journal H-indices from SCImago.\footnote{Dataset source: \url{https://huggingface.co/datasets/inria-soda/carte-benchmark}.} Descriptions and results for the other two datasets are provided in Appendix~\ref{app:additional-realdata}. We use $K$-means to construct multiple covariate groupings for CEC-X evaluation. Detailed settings and grouping procedures are provided in Table~\ref{tab:experiment-settings} and Appendix~\ref{app:settings}.

Figure~\ref{fig:experiment-overview}(c,d) shows that C-USIM brings mean marginal coverage closer to the target and lowers CEC-X for both models. Mean group error also decreases for both models. With the same total label budget as plug-in HPD, C-USIM thus improves both marginal coverage accuracy and average coverage accuracy within groups. Appendix~\ref{app:sjr-results} provides detailed group results and comparisons.

\subsection{Split ratio}
\label{sec:split-ratio}

\begin{figure}[!t]
\centering
\includegraphics[width=\linewidth]{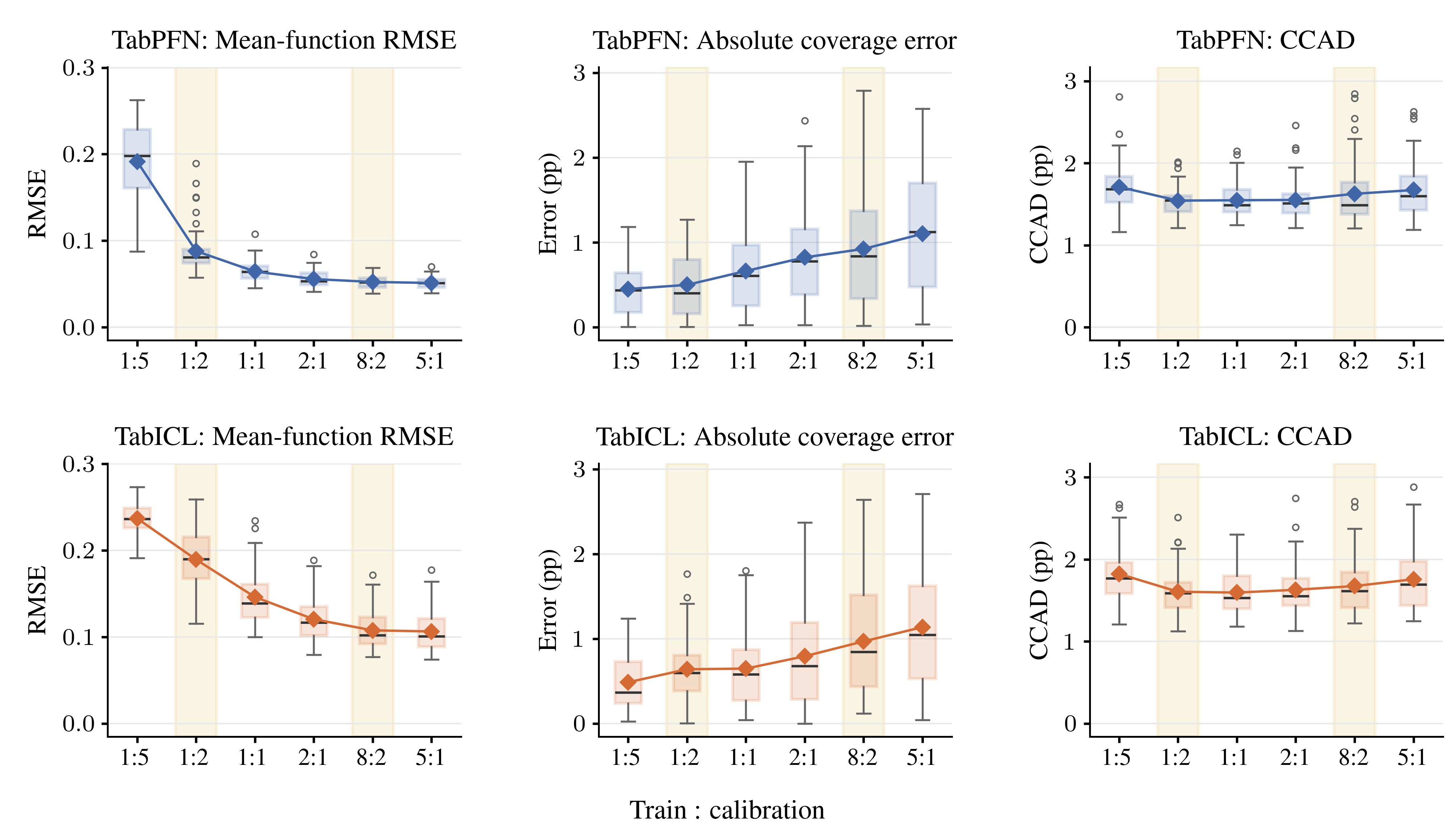}
\caption{Split-ratio sensitivity on \ref{dgp:md-2}; settings follow Table~\ref{tab:experiment-settings}. Rows correspond to TabPFN and TabICL. Columns show mean-function RMSE, absolute marginal coverage error, and CCAD; lower values are better. Standard boxplots summarize variation across seeds; connected diamonds mark the means.}
\label{fig:split-ratio}
\end{figure}

We examine how the training--calibration split affects point-prediction accuracy and coverage stability in C-USIM. An $8{:}2$ training--calibration split is a customary starting point for point prediction, but \citet{das2026optimal} show that the preferred split for prediction-interval length depends on the learning regime. To assess coverage stability, we use the rank--score representation in Section~\ref{sec:percentile-score}: training context shapes the input-specific curves, whereas, for a fixed predictor, more calibration data can stabilize the horizontal line $T=\hat{\tau}$. When additional training yields little improvement in the curves' concentration around $T=P$, allocating more labels to calibration may improve coverage stability.

Figure~\ref{fig:split-ratio} shows how the training--calibration allocation affects coverage stability for C-USIM on \ref{dgp:md-2} from Section~\ref{sec:syn-ex}, using the fixed label budget and settings in Table~\ref{tab:experiment-settings}. The first two columns show clear trade-offs between RMSE and absolute coverage error, with a smaller calibration size leading to better point prediction accuracy while increasing the absolute coverage error. For TabPFN, a $2{:}1$ or $1{:}1$ split seems to strike a good balance between prediction accuracy and coverage stability. For TabICL, there is no such balancing point, showing clear trade-offs between the two goals. CCAD also increases from the $1{:}2$ split through $5{:}1$. The reduction in marginal coverage error holds across all evaluated case--model pairs (Appendix~\ref{app:split-ratio}), supporting larger calibration sets when stable coverage is a priority.

\section{Limitations and future work}
The distribution estimates produced by current TFMs may depend on internal model settings, the calibration and test covariates jointly processed in a prediction call, and the ordering of the training rows; Assumption~\ref{ass:preserve-indep} rules out the dependence on the other jointly processed test covariates. Although these factors do not necessarily violate the exchangeability assumption required for conformal prediction, they make it difficult to characterize the generalized behavior of conditional coverage distributions of C-USIM. Mismatch between a TFM's pretraining prior and downstream tasks may also affect the accuracy of these estimates and, consequently, conditional coverage and prediction-set efficiency. A systematic investigation of this mismatch is beyond the scope of our work, but remains an important direction for research on TFM predictive distributions. Such investigations would require a clear definition of mismatch and controlled evaluation, while accounting for differences in access to pretraining data-generation code and sampling configurations.

C-USIM can be extended to other scoring criteria by computing model-based PIT transformations from the TFM predictive distribution. For instance, if the conditional response distribution is symmetric, a PIT-corrected symmetric density score ($s(x, y)=\mathbb{P}_{Z\sim\mathcal{M}(x)} (f_{\mathcal{M}(x)}(y)  + f_{\mathcal{M}(x)}(2\hat y^{\mathrm{med}} - y)  \le f_{\mathcal{M}(x)}(Z)  + f_{\mathcal{M}(x)}(2\hat y^{\mathrm{med}} - Z) )$), where $\hat y ^{\mathrm{med}}$ is the estimated median response, may be considered. Future work could establish other generalized pivotal scores and evaluate their conditional coverage and efficiency on synthetic and real-world data.

\section{Conclusion}
This work combines conditional--marginal coverage-gap analysis, percentile rank--score diagnostics, and context--calibration allocation experiments for TFM predictive distributions. C-USIM supports this study through a lightweight application of established HPD-split calibration, with marginal validity under the stated exchangeability assumptions. It accommodates multimodal predictions through potentially disjoint regions, using pretrained TFM calibration and test outputs without further training or inference. Experiments with TabPFN and TabICL show that, under a fixed label budget, C-USIM improves marginal coverage accuracy over plug-in HPD, reduces CCAD in most evaluated synthetic cases, and lowers aggregate group coverage errors in most comparisons on real data cases. The split-ratio results further suggest that allocations for conformal prediction may differ from the training-heavy splits motivated by point prediction.

\subsection*{AI use statement}

In this work, we used generative AI tools for implementing methods. We have not used generative AI tools for proposing or refining hypotheses, interpreting results, or constructing mathematical proofs, and qualitative and thematic data analysis is not applicable to this work. Additionally, we used generative AI tools for creating or editing software code, creating or modifying scientific figures, editing the paper to improve readability, and checking mathematical proofs. We have reviewed all AI-assisted work. In particular, we checked the correctness of the AI-generated code. We take responsibility for the final content of this work, including text, claims or artifacts produced with the aid of generative AI.



\subsection*{Reproducibility statement}

Code and detailed configurations will be released on GitHub upon publication. Algorithm~\ref{alg:c-usim} specifies the C-USIM procedure, and Appendix~\ref{proof} provides proofs of the theoretical results under the stated assumptions. Appendix~\ref{app:experiment-details} documents how predictive densities are constructed from TFM outputs, the synthetic data-generating processes, and the evaluation measures. The experimental settings in Appendices~\ref{app:settings} and~\ref{app:additional-realdata} include sample sizes, random seeds, training--calibration allocations, and preprocessing and group construction for the real-world datasets. Appendix~\ref{app:split-ratio} gives further details of the allocation experiments.


\subsubsection*{Acknowledgments}
This work was supported by the National Research Foundation of Korea (NRF) grant funded by the Ministry of Science and ICT (MSIT) (RS-2025-00523567), the Global-LAMP Program of the National Research Foundation of Korea (NRF) grant funded by the Ministry of Education (RS-2023-00301976), and New Faculty Startup Fund from Seoul National University (326-20240027).

\bibliography{iclr2027_conference}
\bibliographystyle{iclr2027_conference}

\newpage

\appendix
\section{Proofs}
\label{proof}
\subsection{Proof of Theorem \ref{thm:pitl1}}
\label{proof:pitl1}
\begin{proof}
The quantity $\Delta_{\tilde s}(x)$ is bounded by the Kolmogorov--Smirnov (KS) distance between the conditional and marginal score distributions \citep{laplante2026postprocessingconformalpredictionapproach}:
\begin{align*}
    \Delta(x)_{\tilde s}\le d_{KS}(F_{\tilde s|X = x}, F_{\tilde s} )
\end{align*}

Let $(t_1^{(x)},\ldots,t_{r+1}^{(x)})$ be a permutation of $(1,\ldots,r+1)$ such that
\begin{align*}
S_{t_i^{(x)}}^{(x)} \leq S_{t_j^{(x)}}^{(x)}
\qquad\text{for } i<j.
\end{align*}

The probability measure of the PIT-corrected score $\tilde{s}$ is:
\begin{align*}
    \mu_{\tilde{s} | X = x} = \sum_{i=1}^{r+1} \tilde{p}_{t_i}^{(x)} \delta _{\sum_{j=1}^i \hat{p}_{t_j}^{(x)}}
\end{align*}
\begin{align*}
    d_{KS}(F_{\tilde{s} | X = x} , \operatorname{Unif}[0, 1]) = \max_{i\in[r+1]}\max \left({\left|\sum_{j=1}^i (\tilde p_{t_j}^{(x)} - \hat p_{t_j}^{(x)})\right|, \left|\sum_{j=1}^{i-1} \tilde p_{t_j}^{(x)} - \sum_{j=1}^{i}\hat p_{t_j}^{(x)}\right|}\right) \\ \le \frac{1}{2} \lVert\tilde p^{(x)} - \hat p^{(x)}\rVert_1 + \max_{i\in [r + 1]} \hat p_i^{(x)} = B(x)
\end{align*}

By the convexity of the KS distance, and since $\mu_{\tilde{s}}$ is the expectation of the conditional distribution $\mu_{\tilde{s} | X = x}$, the inequality below holds.

\begin{align*}
    d_{\mathrm{KS}}(F_{\tilde s},\operatorname{Unif}(0,1))
    &\leq
    \mathbb{E}_X\!\left[
    d_{\mathrm{KS}}(
    F_{\tilde s\mid X},
    \operatorname{Unif}(0,1)
    )
    \right] \\
    &\le \mathbb{E}_X  [B(X)]
\end{align*}
Therefore, by the triangle inequality of the KS distance,
\begin{align*}
    d_{KS}(F_{\tilde s|X = x}, F_{\tilde s}) &\le d_{KS}(F_{\tilde s} , \operatorname{Unif}[0, 1]) + d_{KS}(F_{\tilde s|X=x} , \operatorname{Unif}[0, 1]) \\&\le \mathbb{E}_{X}\!\left[B(X)\right] + B(x).
\end{align*}
\end{proof}

\subsection{Proof of Theorem  \ref{thm:idealcc}}
\label{proof:idealcc}
\begin{proof}
The theorem's assumption on the strictly increasing function $f$ implies that $f^{-1} = F_{s_{\mathcal{M}}(X, Y) | X = x}$ almost surely regardless of $x$. Thus,
    \begin{align*}
        P_t (x) &= \mathbb P\left(
    s_{\mathcal M}(X,Y)\le t \mid X=x
    \right) \\&= \mathbb E[\mathbb P\left(
    s_{\mathcal M}(X,Y)\le t \mid X=x
    \right)] \\&= P\left(
    s_{\mathcal M}(X,Y)\le t\right) \\&= F_{s_{\mathcal M}(X,Y)}(t)
    \end{align*}

Because of the assumption, $F_{s_{\mathcal M}(X,Y)}(s_{\mathcal M}(X,Y))\sim\operatorname{Unif}(0, 1)$, and since a cumulative distribution function preserves order, $C = P_{\hat{\tau}}(X) = F_{s_{\mathcal M}(X,Y)}(\hat\tau)$ has the distribution of the $k$-th order statistic of $n_{\mathrm {cal}}$ independent $\operatorname{Unif}(0, 1)$ random variables, where $k=\left\lceil (n_{\mathrm {cal}}+1)(1-\alpha)\right\rceil \le n_{\mathrm {cal}}$. This order statistic follows $\operatorname{Beta}(k, n_{\mathrm {cal}}+1-k)$.
\end{proof}

\newpage

\section{Experimental details}
\label{app:experiment-details}

\subsection{Density Reconstruction from TFM Outputs}
\label{app:density-estimation}

The following reconstruction rules convert TFM outputs into piecewise-constant densities without fitting a separate density estimator.

\paragraph{Binning-based Outputs.} For binning-based models, we represent the estimated conditional distribution at $x$ by contiguous predicted intervals $\{[\hat{y}_{i-1}^{(x)}, \hat{y}_i^{(x)})\}_{i=1}^{r+1}$, where $\hat{y}_{0}^{(x)} <\hat{y}_{1}^{(x)} <...    <\hat{y}_{r+1}^{(x)}$ and their corresponding probability weights $\{\hat{p}_i^{(x)}\}_{i=1}^{r+1}$. We construct a piecewise-constant density estimate by dividing each predicted probability weight by the width of its corresponding interval. Specifically,
\begin{align*}
    f_{\mathcal{M}(x)}(y)
    =
    \sum_{i=1}^{r+1}
    \frac{\hat{p}_i^{(x)}}
    {\hat{y}_{i}^{(x)}-\hat{y}_{i-1}^{(x)}}
    \mathbf{1}\!\left\{
    y\in[\hat{y}_{i-1}^{(x)},\hat{y}_{i}^{(x)})
    \right\}.
\end{align*}

For TabPFN, we consider the default setting ($r = 4999$, $\hat{y}_i^{(x)}$ is fixed with respect to $x$)  \citep{Grinsztajn2026-ei}.

\paragraph{Quantile-based Outputs.} For quantile-based models, we represent the estimated conditional distribution at $x$ by the set of quantile level-value pairs $\{(q_i^{(x)}, y_i^{(x)})\}_{i=0}^{r + 1}$ such that $0=\hat q_0^{(x)} < \hat q_1^{(x)}<\hat q_2^{(x)}<...<\hat q_{r+1}^{(x)}=1$, $\hat y_0^{(x)} < \hat y_1^{(x)}<\hat y_2^{(x)}<...<\hat y_{r+1}^{(x)}$. We construct a piecewise-constant density estimate by dividing each percentile gap by the gap between quantiles. Specifically,
\begin{align*}
    f_{\mathcal{M}(x)}(y) = \sum_{i=1}^{r+1}
    \frac{\hat{q}_{i}^{(x)}-\hat{q}_{i-1}^{(x)}}
    {\hat{y}_{i}^{(x)}-\hat{y}_{i-1}^{(x)}}
    \mathbf{1}\!\left\{
    y\in[\hat{y}_{i-1}^{(x)},\hat{y}_{i}^{(x)})
    \right\}.
\end{align*}

For TabICL, we use the default quantile grid $\left(r = 999, \hat q_i ^{(x)} = \dfrac{i}{r+1}\right)$. The model does not predict the boundary quantiles $\hat y_0^{(x)}$, $\hat y_{r+1}^{(x)}$ \citep{Qu2026-rv}. Before constructing the density, we retain strictly increasing raw quantile sequences and repair only those with crossings or duplicate values, using float64 arithmetic. For a raw sequence $v_0,\ldots,v_{998}$, let $R$ be its range, replaced by $\max(1,|v_0|)$ if the sequence is constant, and set $g=\max\{10^{-6}R/998,\,8\operatorname{ulp}(\max_j|v_j|)\}$, where $\operatorname{ulp}$ denotes the spacing to the next larger float64 value. We apply unweighted least-squares isotonic regression to $v_j-jg$ using the pool-adjacent-violators algorithm and then restore the ramp $jg$. The resulting strictly increasing sequence supplies the interior quantiles in the density above. After this correction, we set the boundary quantile values to $\hat y_1^{(x)} -P\cdot (\hat y_2^{(x)} - \hat y_1^{(x)})$ and $\hat y_r^{(x)} +P\cdot (\hat y_r^{(x)} - \hat y_{r-1}^{(x)})$, respectively, with $P$ fixed at 3 throughout our experiments. This gives 1,000 finite intervals, each with probability mass 0.001.

\subsection{Synthetic data-generating processes}
\label{app:synthetic-dgps}

Within each example, context, calibration, and test observations (and hence their covariates) are sampled i.i.d. from the same joint distribution.

\subsubsection{One-dimensional Examples}
The three one-dimensional examples use $X\sim\operatorname{Unif}(0,2\pi)$. Throughout this appendix, $Z\sim\mathcal{N}(0,1)$ is independent of the covariates and any mixture-component draw. All Gaussian scale parameters below are standard deviations.

\paragraph{1D-1: Sine with shifted exponential noise.}
\dgplabel{dgp:1d-1}{1D-1}
For $E\sim\operatorname{Exp}(1)$ independent of $X$, the response is
\[
Y=\sin(10X)+0.18(E-1).
\]
The exponential distribution has rate one, so the noise is centered and has variance $0.18^2$.

\paragraph{1D-2: Periodic Gaussian scale mixture.}
\dgplabel{dgp:1d-2}{1D-2}
Given $X=x$, draw $B\sim\operatorname{Bernoulli}(p(x))$, independently of $Z$, where
\[
p(x)=0.08+0.30(0.5+0.5\sin(5x)).
\]
Then
\[
Y=0.8\sin(12X)+0.2\cos(3X)+\sigma_B Z,
\qquad
\sigma_B=
\begin{cases}
0.48, & B=1,\\
0.07, & B=0.
\end{cases}
\]
Thus, $p(x)$ is the probability of the wider Gaussian component.

\paragraph{1D-3: Step function.}
\dgplabel{dgp:1d-3}{1D-3}
The response is
\[
Y=m(X)+0.13Z,
\qquad
m(x)=
\begin{cases}
0.85, & \sin(10x)\geq 0,\\
-0.85, & \sin(10x)<0.
\end{cases}
\]

\subsubsection{Multi-dimensional examples}

The three multi-dimensional examples use i.i.d.\ coordinates $X_j\sim\operatorname{Unif}(0,1)$, $j=1,\ldots,d$. These examples share the mean
\[
m(x)=\sin(2\pi x_1)+0.65(2x_2-1)(2x_3-1).
\]
Given $X=x$, draw $B\sim\operatorname{Bernoulli}(p(x))$ independently of $Z\sim\mathcal{N}(0,1)$ and set $Y=m(X)+\sigma_B(X)Z$. The parameters below specify the probability and standard deviation of each mixture component. These mechanisms were selected during earlier exploratory screening; the displayed evaluation includes these three mechanisms for both models.

\paragraph{MD-1: Input-dependent severity, $d=5$.}
\dgplabel{dgp:md-1}{MD-1}
The mixture probability is $p(x)=0.20$, with $\sigma_0(x)=0.05$ and $\sigma_1(x)=0.30+0.90x_4$. The fourth coordinate changes the scale of the wider component.

\paragraph{MD-2: Radial mixture weight, $d=10$.}
\dgplabel{dgp:md-2}{MD-2}
Define
\[
r(x)=\left\{\frac{1}{7}\sum_{j=4}^{10}(2x_j-1)^2\right\}^{1/2},
\qquad
p(x)=0.03+0.65\operatorname{clip}\!\left(\frac{r(x)-0.25}{0.60},0,1\right),
\]
where $\operatorname{clip}(a,0,1)=\min(1,\max(0,a))$. The component scales are $\sigma_0=0.05$ and $\sigma_1=1.00$.

\paragraph{MD-3: Input-dependent scale mixture, $d=20$.}
\dgplabel{dgp:md-3}{MD-3}
Here $p(x)=0.05+0.50x_4$, $\sigma_0=0.07$, and $\sigma_1=0.90$. The remaining 16 coordinates do not affect the conditional response law.

\subsubsection{Bimodal Examples}

The three bimodal examples uses i.i.d coordinates $X_j\sim \operatorname{Unif}(0, 1)$ for $j=1, \dots ,d$. All these examples share the center function \(m\), which is defined as
\[
m(x)=\sin(2\pi x_1).
\]
For each sample, given \(X=x\), we draw the latent variable \(S\sim\operatorname{Bernoulli}(w(x))\) independently of \(Z\sim N(0,1)\). The multimodal weight function \(w(x)\) determines the relative weights of the two modes. The response and its conditional expectation given \(X=x\) are
\begin{align*}
    &Y = m(X)+((c_0 +\sigma_0 Z) \mathbf{1}(S=0) + (c_1 + \sigma_1 Z)\mathbf{1}(S=1)) \\
    &\mathbb{E}[Y\mid X=x]=m(x)+c_0(1-w(x))+c_1w(x) \qquad \forall x\in [0, 1]^d
\end{align*}

\paragraph {BM-1: Symmetric, well-separated mixture, $d = 1$.}
\dgplabel{dgp:bm-1}{BM-1}
$w(X) = 0.5$ with $c_0 = -1.2, ~c_1=1.2, ~\sigma_0 = \sigma_1=0.15$.

\paragraph {BM-2: Unequal mixture weights and variances, $d = 1$.}
\dgplabel{dgp:bm-2}{BM-2}
$w(X) = 0.3$ with $c_0 = -0.36, \sigma_0 = 0.08, c_1=0.84, ~\sigma_1=0.20$.

\paragraph {BM-3: Covariate-dependent mixture weights, $d = 2$.}
\dgplabel{dgp:bm-3}{BM-3}
$w(X) = 0.1+0.8X_2$ with $c_0 = -0.6,  c_1=0.6, ~\sigma_0 = \sigma_1=0.15$.

\subsection{Construction of percentile rank--score plots and conditional coverage distributions}
\label{app:percentile-score}

First, for each fixed $X_i$, we independently draw two sets of Monte Carlo samples:
\begin{align*}
\widetilde{Y}_i^{(1)},\ldots,\widetilde{Y}_i^{(B)}
\overset{\mathrm{i.i.d.}}{\sim} Y\mid X=X_i,
\qquad
Y_i^{(1)},\ldots,Y_i^{(L)}
\overset{\mathrm{i.i.d.}}{\sim} Y\mid X=X_i.
\end{align*}
The first set serves as a reference sample for approximating the conditional percentile rank:
\begin{align*}
\widehat{P}(X_i,y) := \frac{1}{B} \sum_{b=1}^{B}
\mathbf{1}\!\left\{
T(X_i,\widetilde{Y}_i^{(b)}) \leq T(X_i,y)
\right\}.
\end{align*}

For each \(X_i\), we plot an empirical percentile rank--score curve by sorting and connecting the $L$ points
\begin{align*}
\left\{
\widehat{P}(X_i,Y_i^{(\ell)}), 
T(X_i,Y_i^{(\ell)})
\right\}_{\ell=1}^{L}
\end{align*}
We overlay these curves for $i=1,\ldots,m$.  The conditional-coverage distribution is approximated from the horizontal coordinates of the intersections between these interpolated curves and the horizontal line at each score threshold. We approximate $q_{1-\alpha}$ by the empirical $(1-\alpha)$-quantile of the pooled Monte Carlo scores.

Under Assumptions~\ref{ass:preserve-indep} and \ref{ass:conti-dist}, the population threshold satisfies $\mathbb{E}[P_{q_{1-\alpha}} (X)]=\mathbb{P}[s_{\mathcal{M}}(X, Y)\le q_{1-\alpha}] = 1-\alpha$. For a discrete marginal score law, coverage at this threshold is at least $1-\alpha$ and can exceed it because of a jump in the score CDF.

Generally, the predictive distributions of TabPFN and TabICL in our setting are piecewise uniform, so Assumption~\ref{ass:conti-dist} does not hold exactly.

In particular, the identity in Section~\ref{sec:percentile-score},
$$P(x,y) =\mathbb{P}\!\left(s_{\mathcal M}(x,Z) < s_{\mathcal M}(x,y) \mid X=x\right)$$
need not hold rigorously, since the inequality

$$
\mathbb{P}\!\left(
Z \in \widehat{C}_{T(x,y)}(x) \mid X=x
\right)
\geq
\mathbb{P}\!\left(
s_{\mathcal M}(x,Z) < s_{\mathcal M}(x,y) \mid X=x
\right)
$$

is strict whenever that level has positive probability under the true conditional response law; this probability is the difference between the two sides. Discrete scores do not satisfy exact PIT uniformity but both models yield sufficiently fine discretizations that this discrepancy can reasonably be regarded as negligible and as having no material effect on the overall patterns in the plots.

Unless otherwise noted, the top two rows of each figure show TabPFN, and the bottom two rows show TabICL. Rows with $n_{\mathrm{train}}=4096$ alternate with rows with $n_{\mathrm{train}}=1024$, starting from the top. We set $B = L = 10000$. Plot (a) shows percentile rank--score plots, whereas plot (b) shows the interpolated conditional-coverage distributions at ideal C-USIM and plug-in HPD.


\clearpage
\begin{figure}[p]
\centering
\begin{subfigure}[b]{.86\textwidth}
    \centering
    \includegraphics[width=\textwidth]{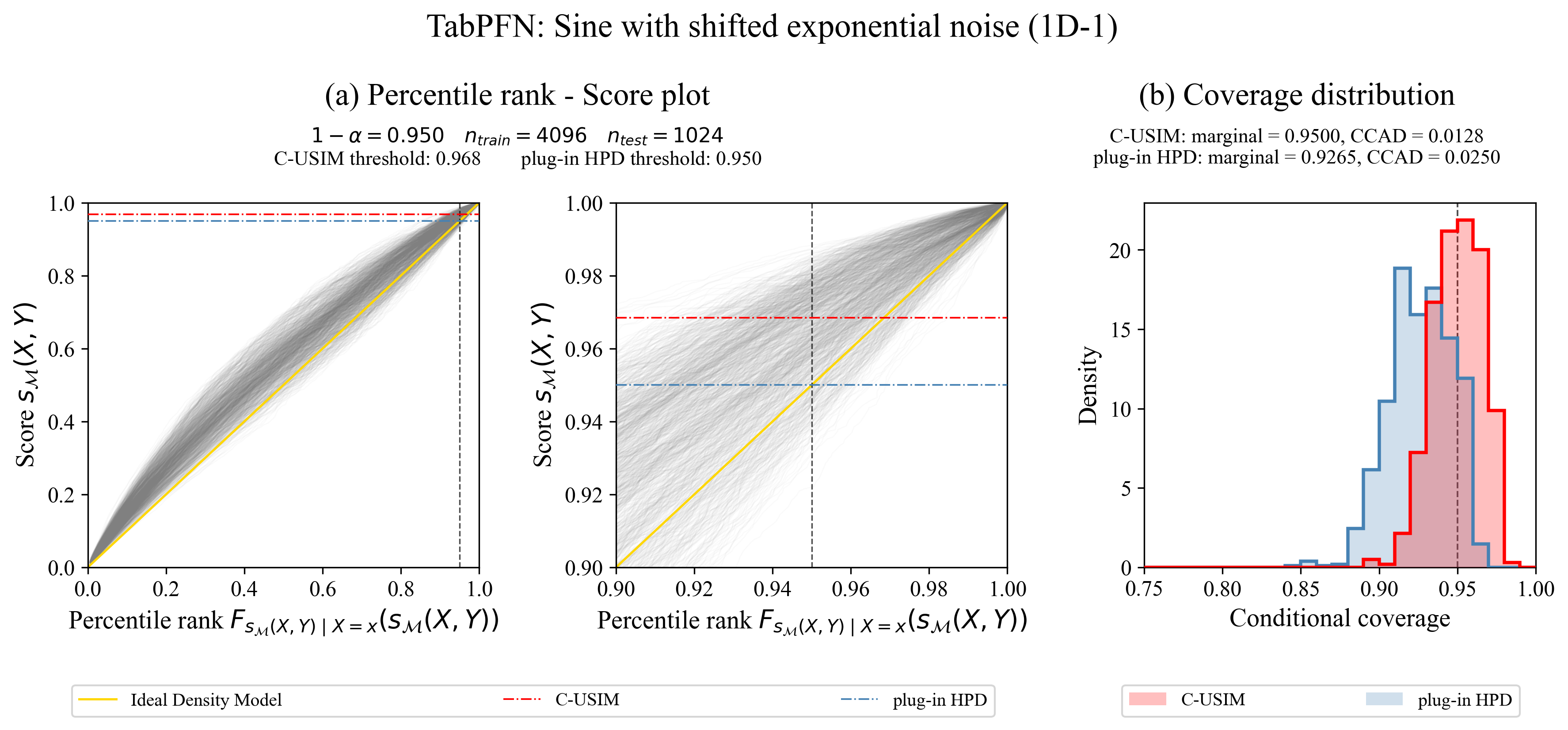}
\end{subfigure}
\hfill
\begin{subfigure}[b]{.86\textwidth}
    \centering
    \includegraphics[width=\textwidth]{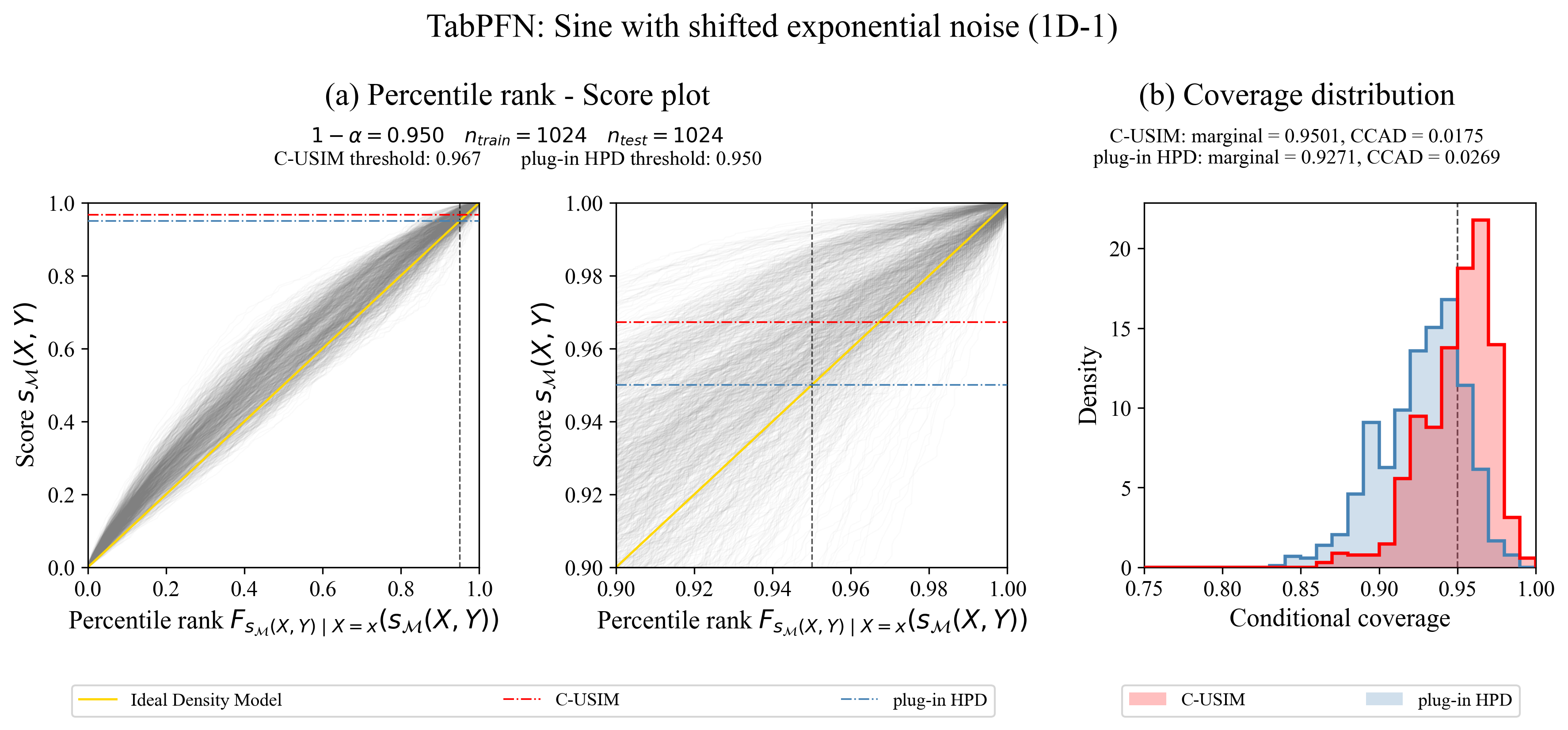}
\end{subfigure}
\hfill
\begin{subfigure}[b]{.86\textwidth}
    \centering
    \includegraphics[width=\textwidth]{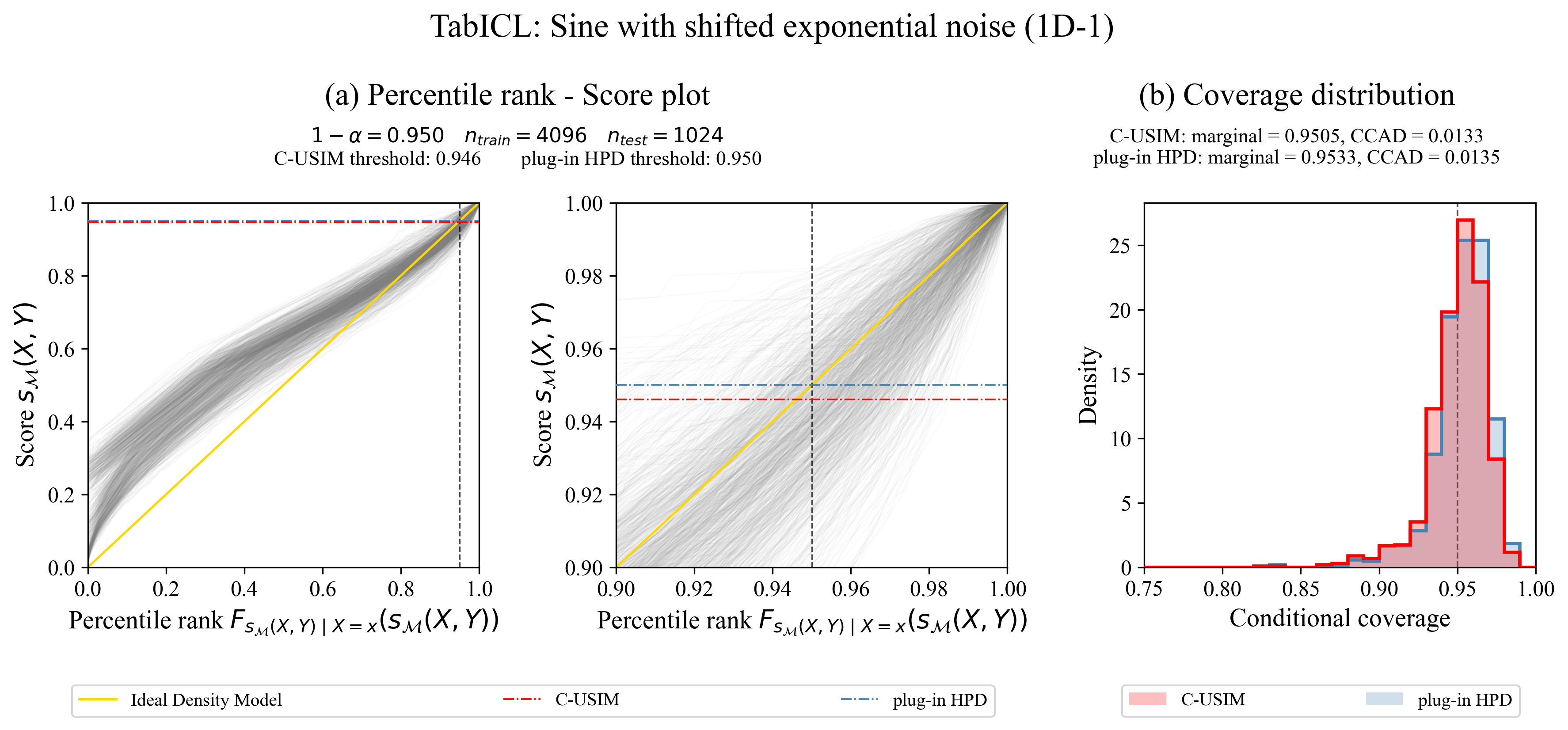}
\end{subfigure}
\hfill
\begin{subfigure}[b]{.86\textwidth}
    \centering
    \includegraphics[width=\textwidth]{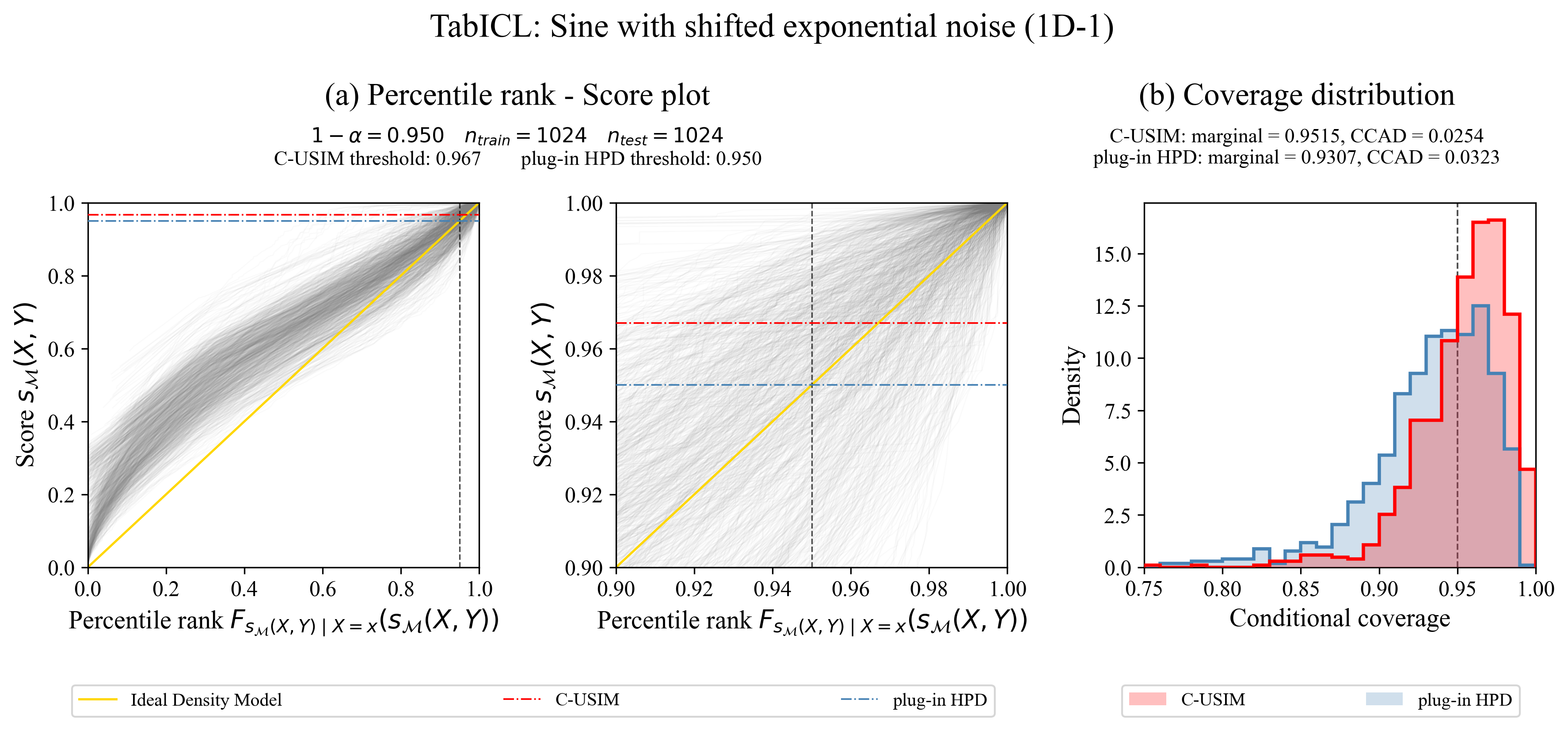}
\end{subfigure}

\caption{Percentile rank--score diagnostics for \ref{dgp:1d-1}, seed 2026: C-USIM ($n_{\mathrm{cal}}\to\infty$) and plug-in HPD.}
\end{figure}

\begin{figure}[p]
\centering
\begin{subfigure}[b]{.86\textwidth}
    \centering
    \includegraphics[width=\textwidth]{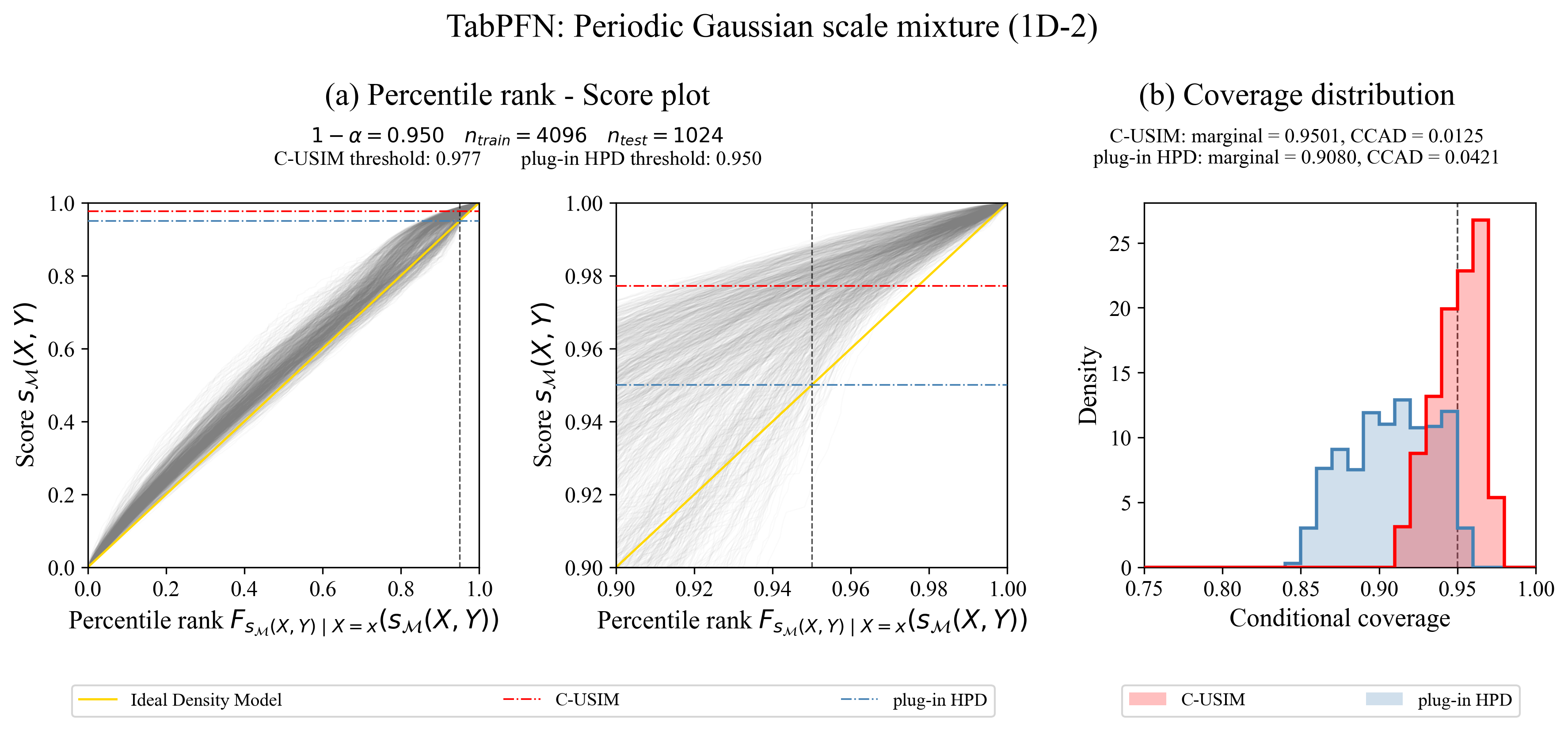}
\end{subfigure}
\hfill
\begin{subfigure}[b]{.86\textwidth}
    \centering
    \includegraphics[width=\textwidth]{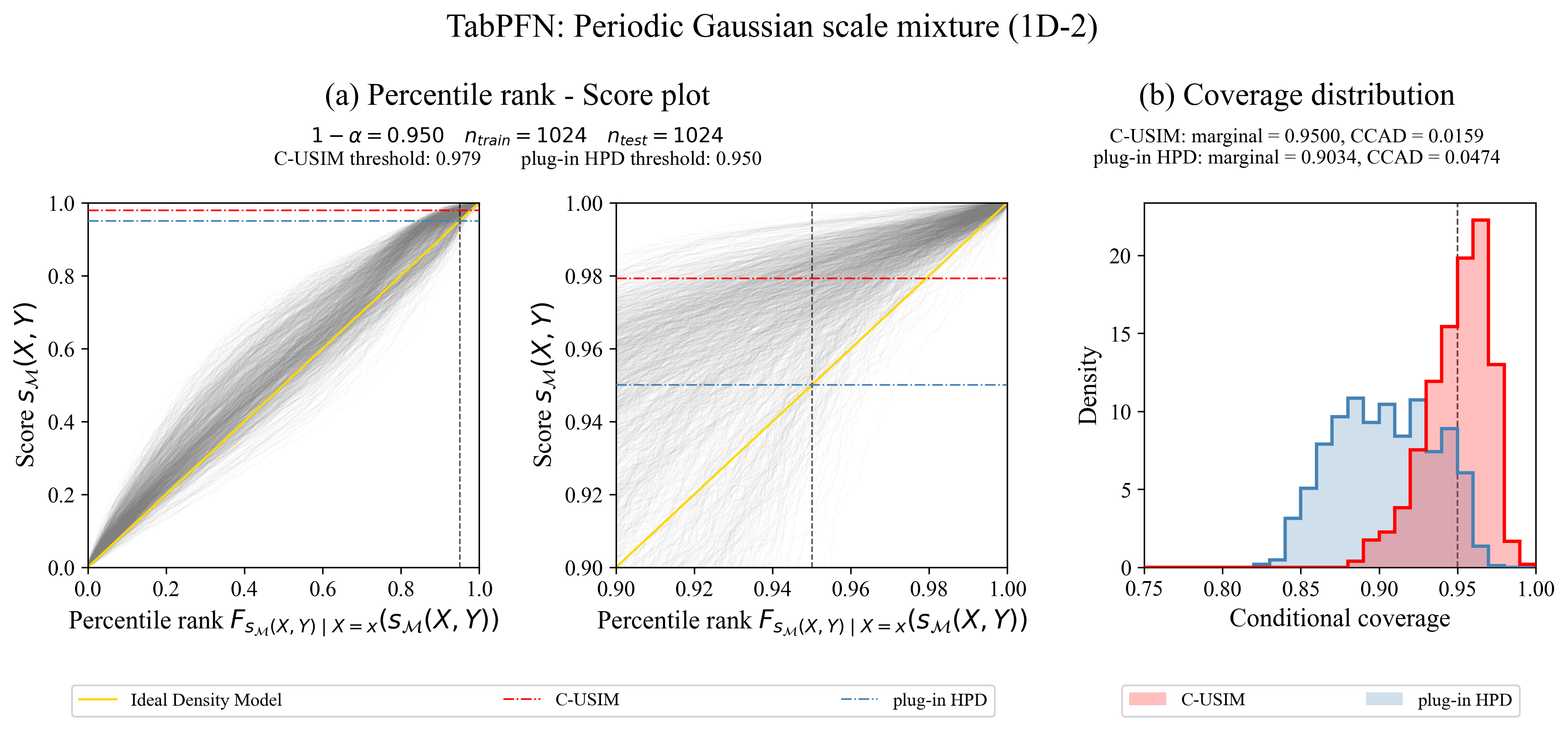}
\end{subfigure}
\hfill
\begin{subfigure}[b]{.86\textwidth}
    \centering
    \includegraphics[width=\textwidth]{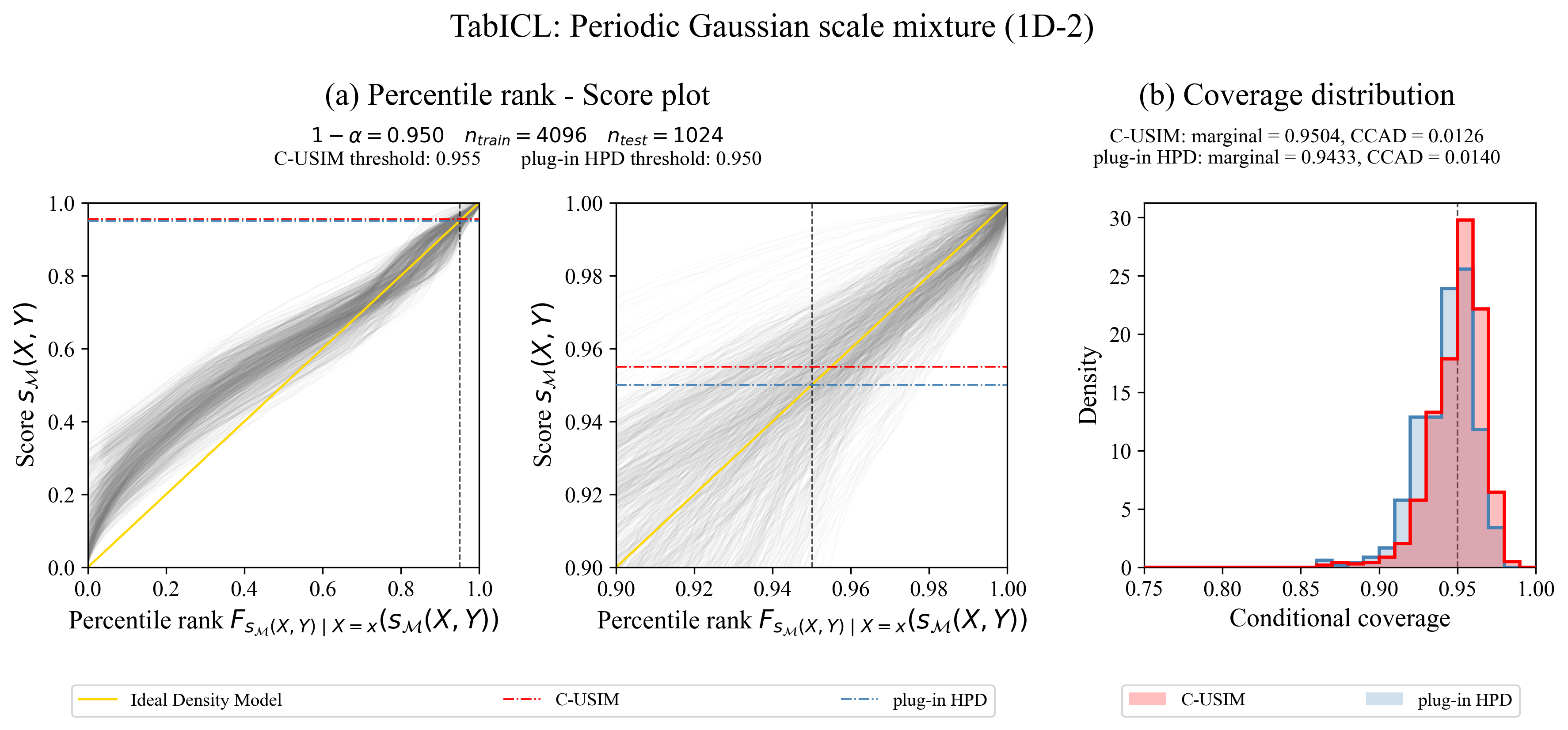}
\end{subfigure}
\hfill
\begin{subfigure}[b]{.86\textwidth}
    \centering
    \includegraphics[width=\textwidth]{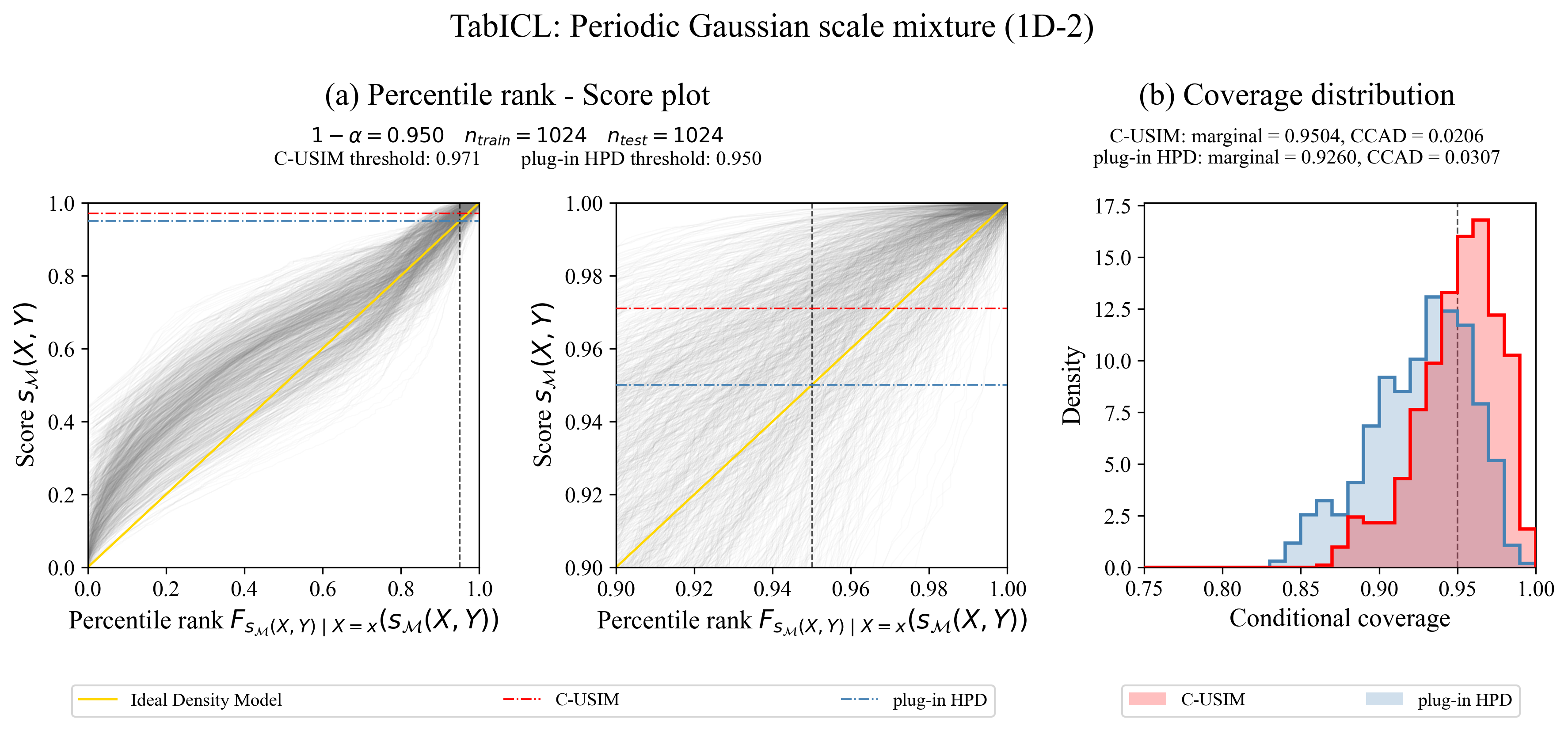}
\end{subfigure}

\caption{Percentile rank--score diagnostics for \ref{dgp:1d-2}, seed 2026: C-USIM ($n_{\mathrm{cal}}\to\infty$) and plug-in HPD.}
\end{figure}

\begin{figure}[p]
\centering
\begin{subfigure}[b]{.86\textwidth}
    \centering
    \includegraphics[width=\textwidth]{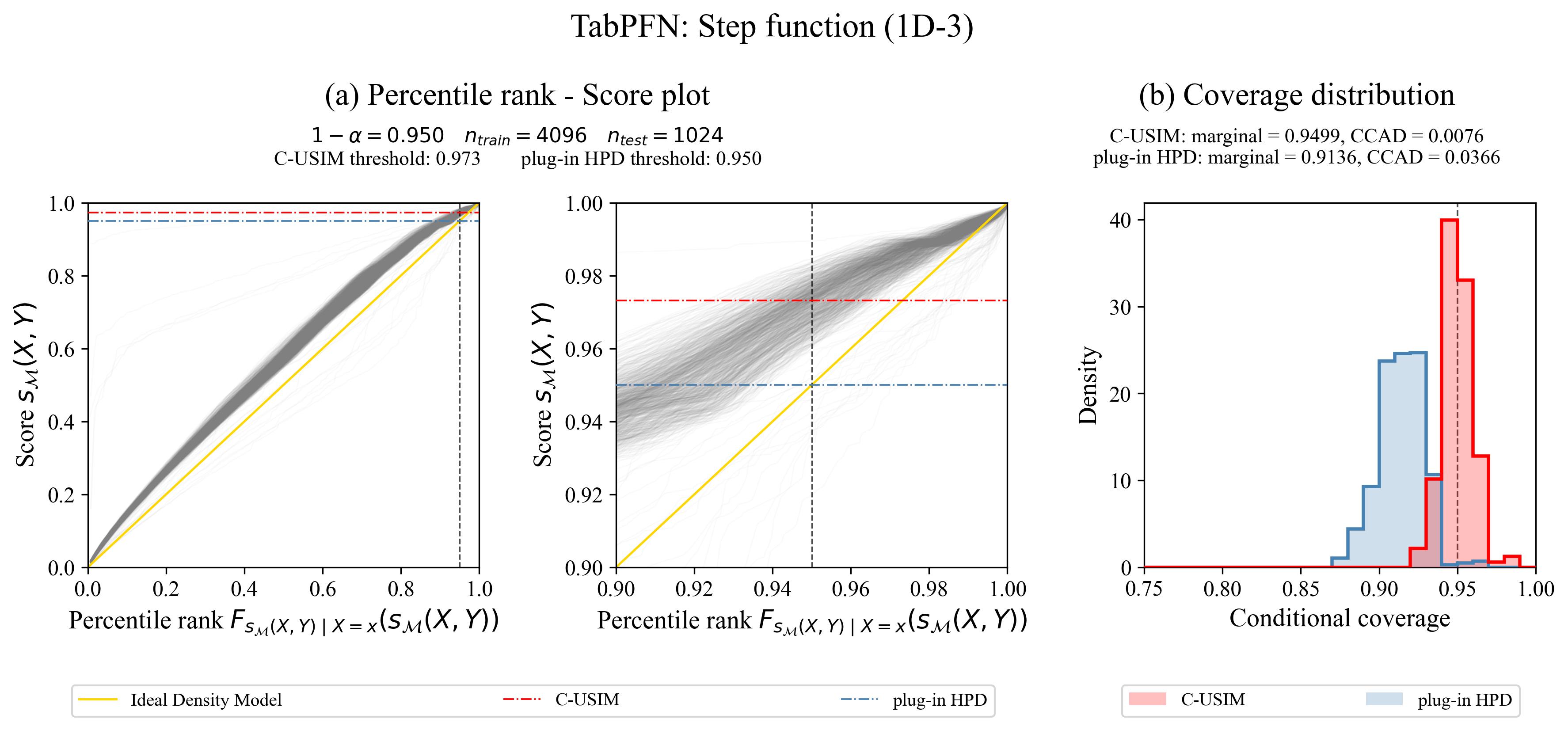}
\end{subfigure}
\hfill
\begin{subfigure}[b]{.86\textwidth}
    \centering
    \includegraphics[width=\textwidth]{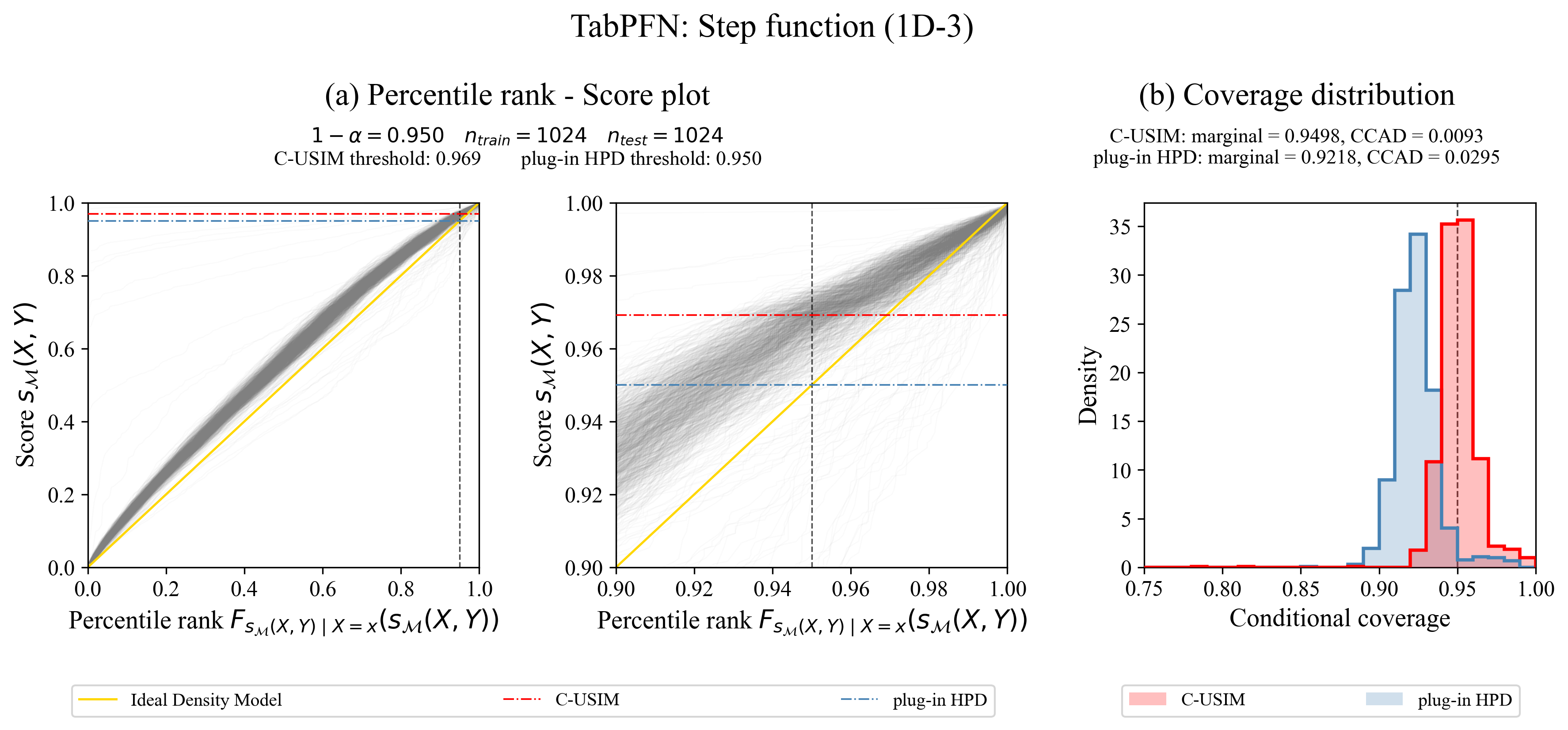}
\end{subfigure}
\hfill
\begin{subfigure}[b]{.86\textwidth}
    \centering
    \includegraphics[width=\textwidth]{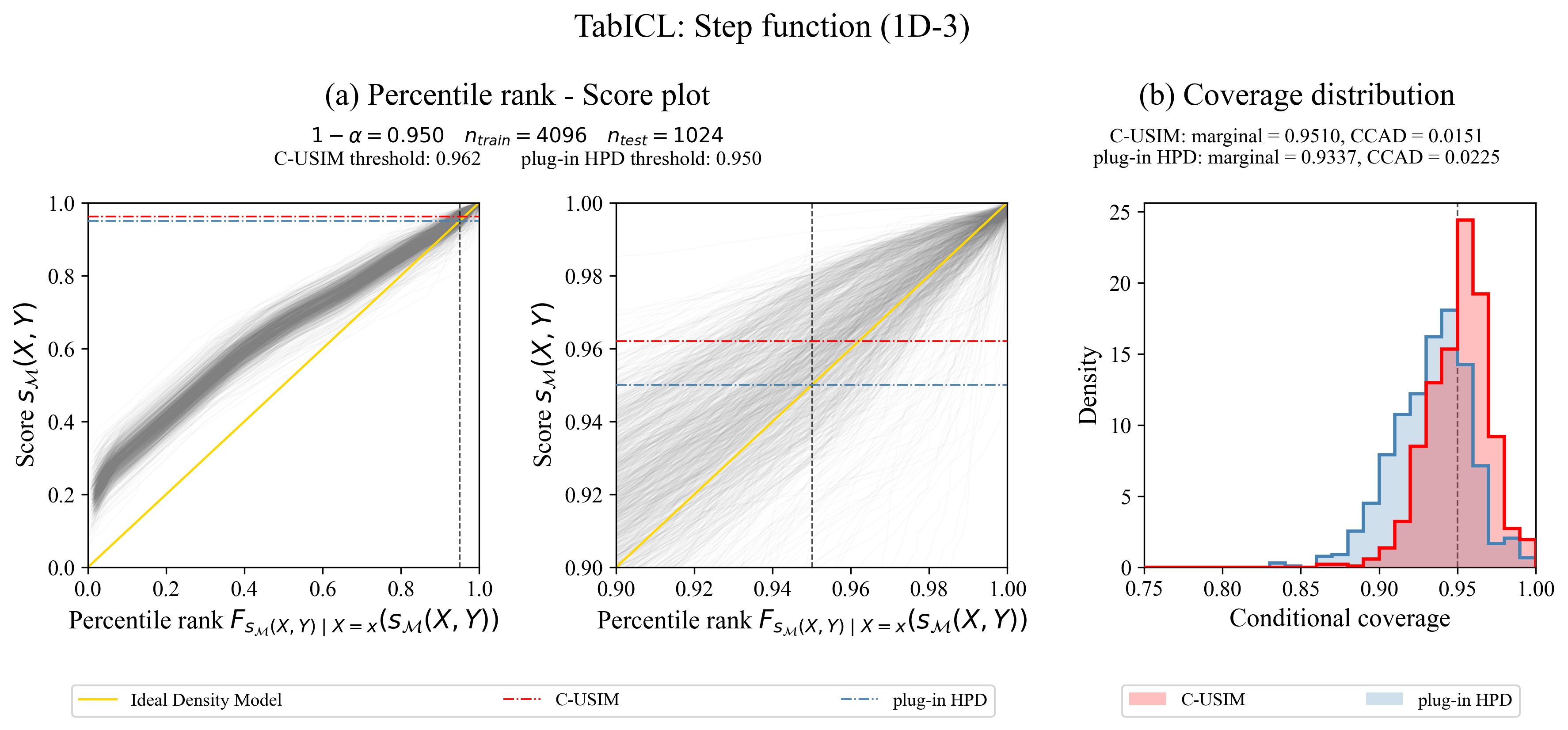}
\end{subfigure}
\hfill
\begin{subfigure}[b]{.86\textwidth}
    \centering
    \includegraphics[width=\textwidth]{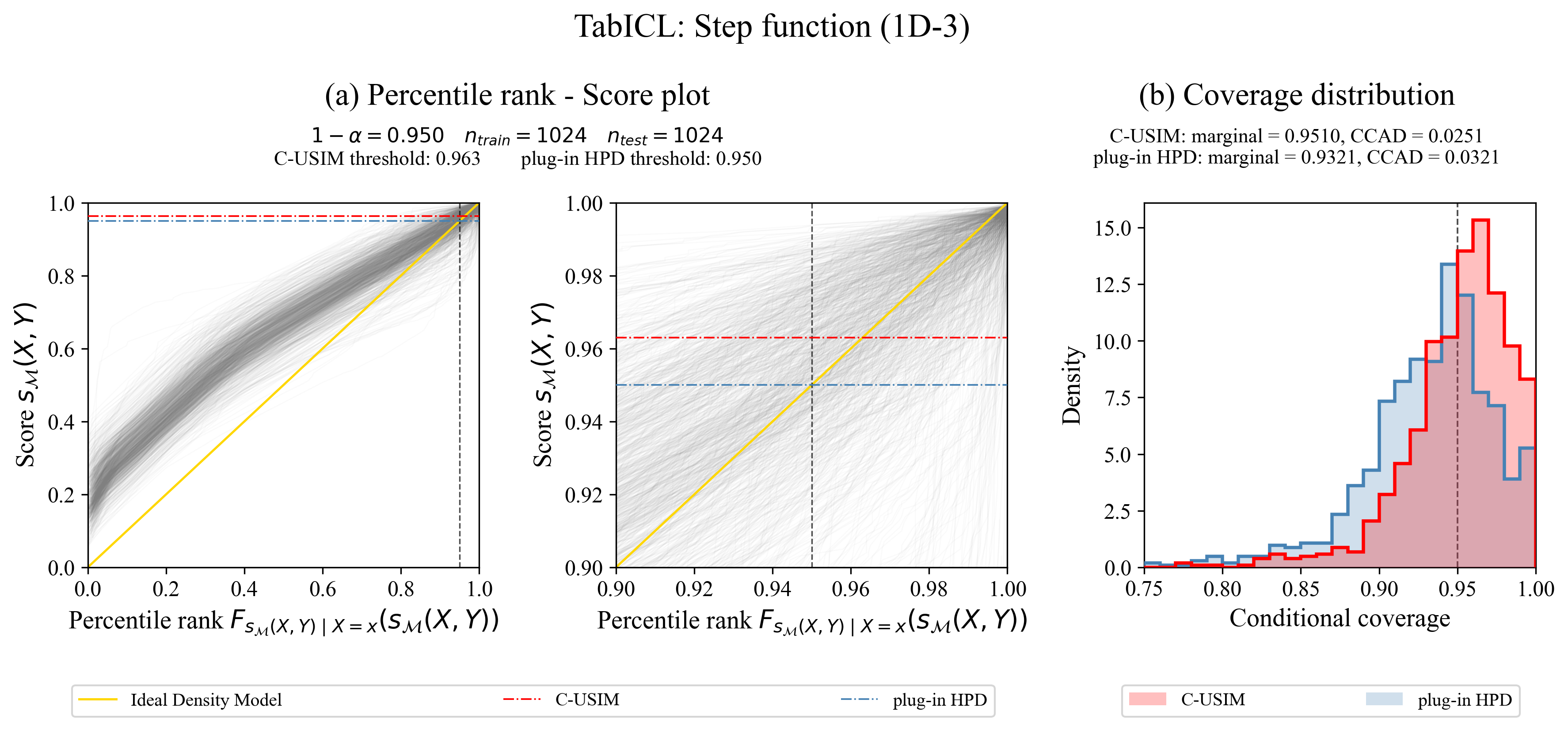}
\end{subfigure}

\caption{Percentile rank--score diagnostics for \ref{dgp:1d-3}, seed 2026: C-USIM ($n_{\mathrm{cal}}\to\infty$) and plug-in HPD.}
\end{figure}

\begin{figure}[p]
\centering
\begin{subfigure}[b]{.86\textwidth}
    \centering
    \includegraphics[width=\textwidth]{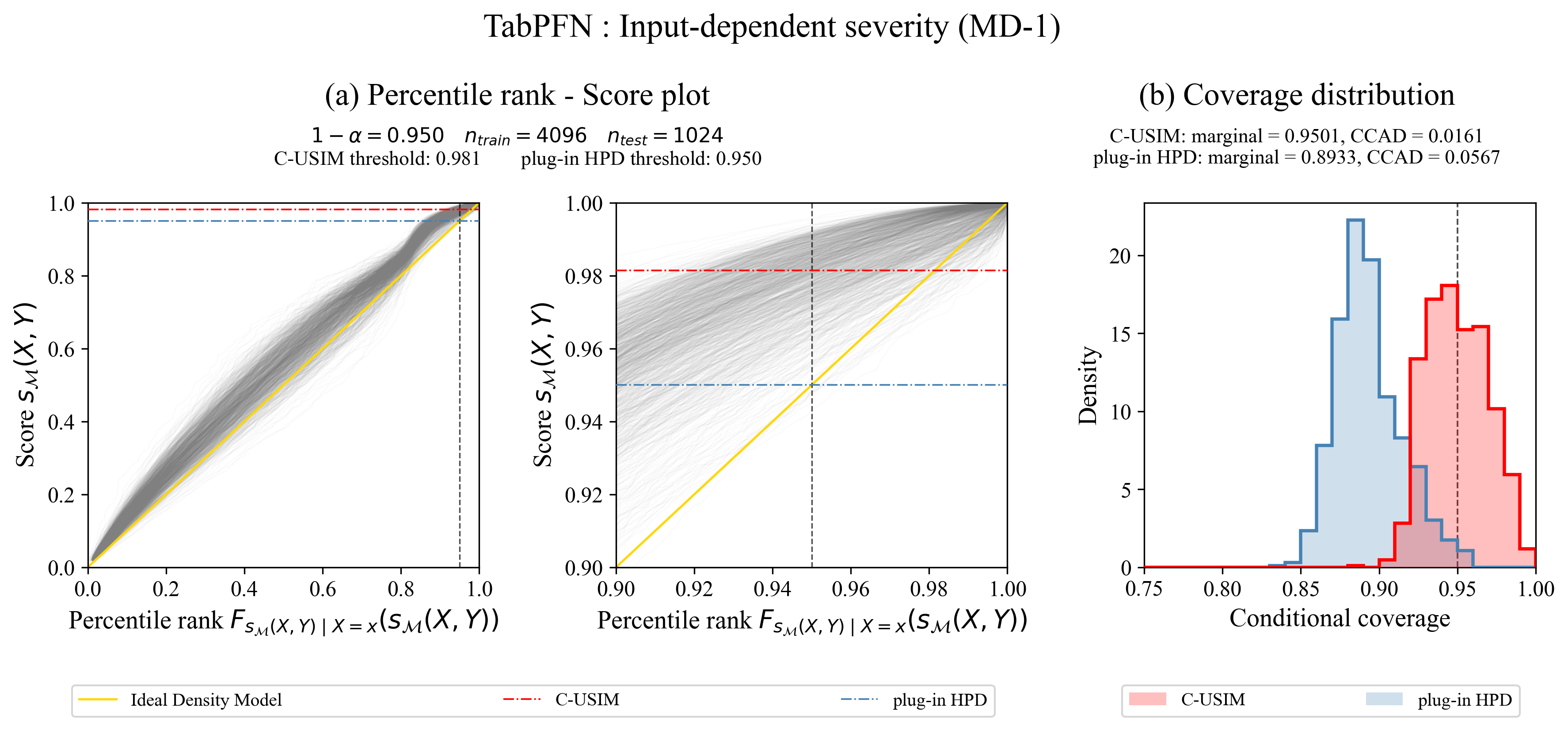}
\end{subfigure}
\hfill
\begin{subfigure}[b]{.86\textwidth}
    \centering
    \includegraphics[width=\textwidth]{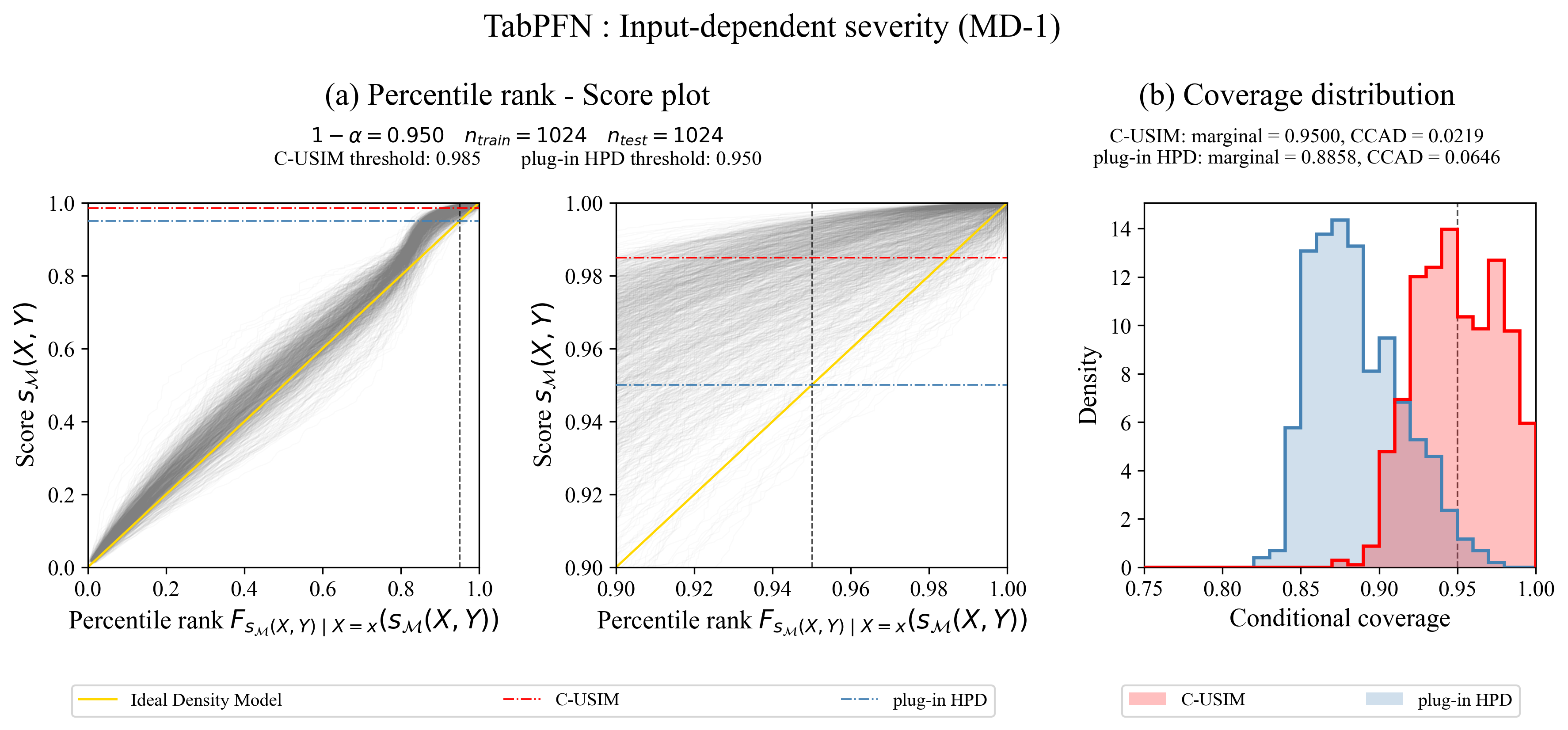}
\end{subfigure}
\hfill
\begin{subfigure}[b]{.86\textwidth}
    \centering
    \includegraphics[width=\textwidth]{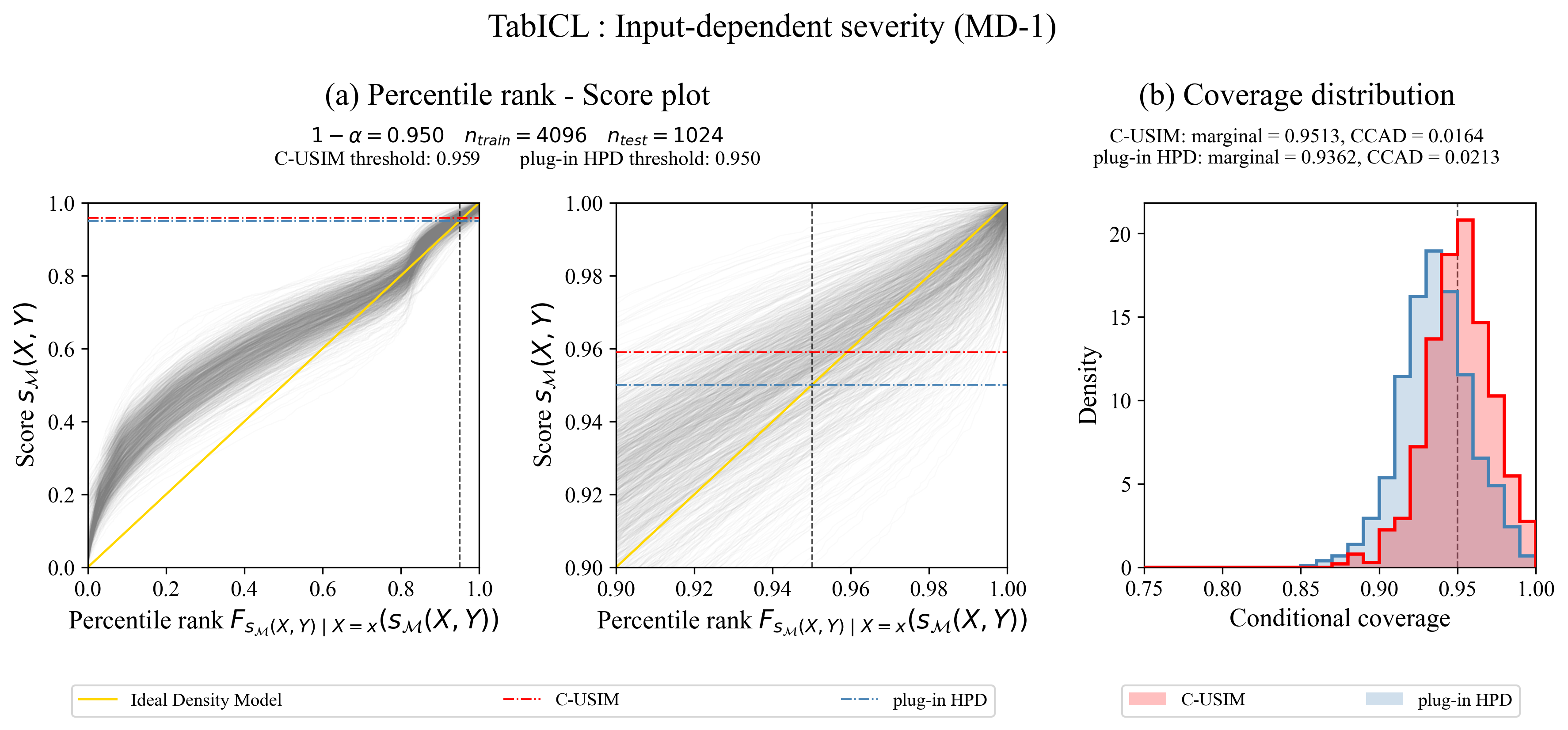}
\end{subfigure}
\hfill
\begin{subfigure}[b]{.86\textwidth}
    \centering
    \includegraphics[width=\textwidth]{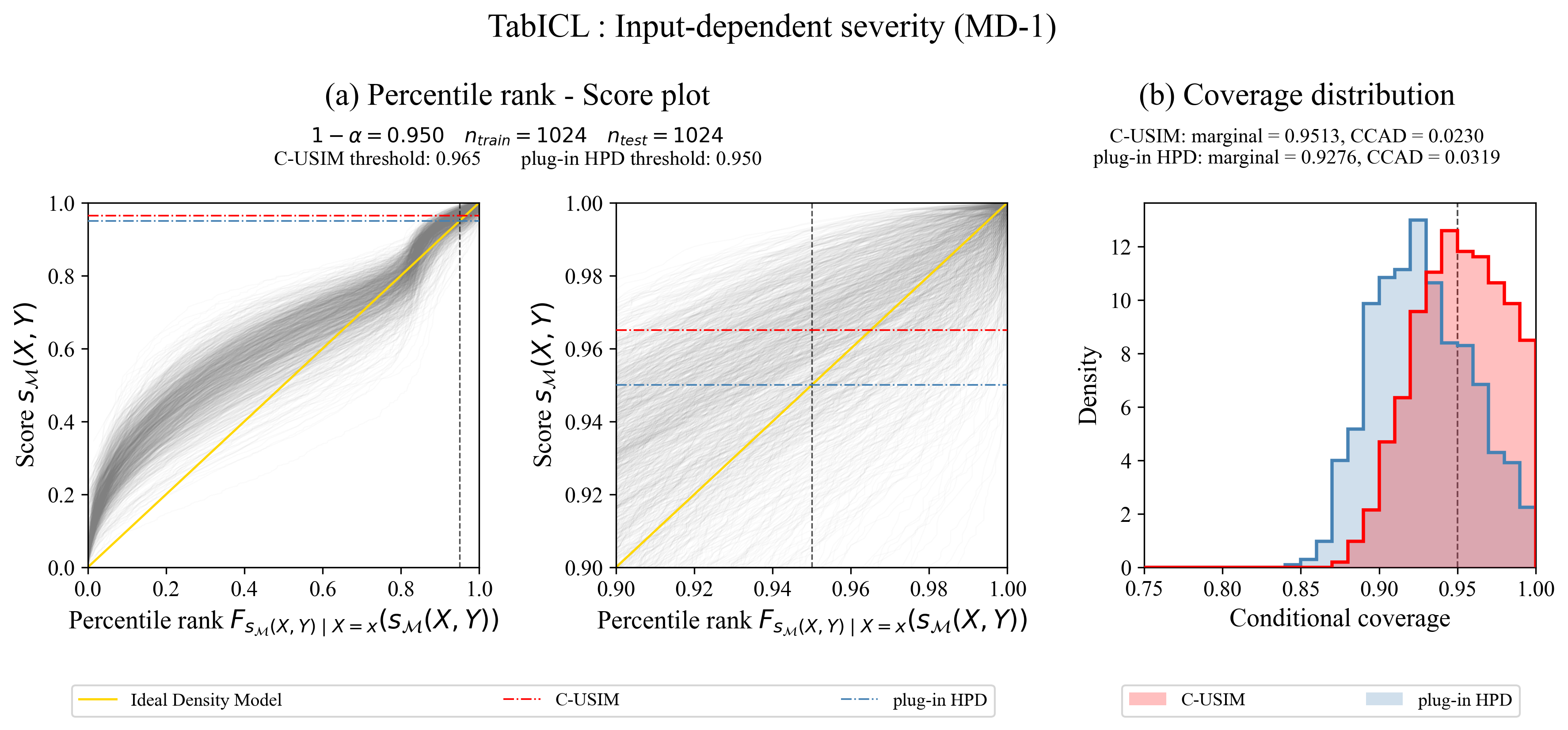}
\end{subfigure}

\caption{Percentile rank--score diagnostics for \ref{dgp:md-1}, seed 2026: C-USIM ($n_{\mathrm{cal}}\to\infty$) and plug-in HPD.}
\end{figure}

\begin{figure}[p]
\centering
\begin{subfigure}[b]{.86\textwidth}
    \centering
    \includegraphics[width=\textwidth]{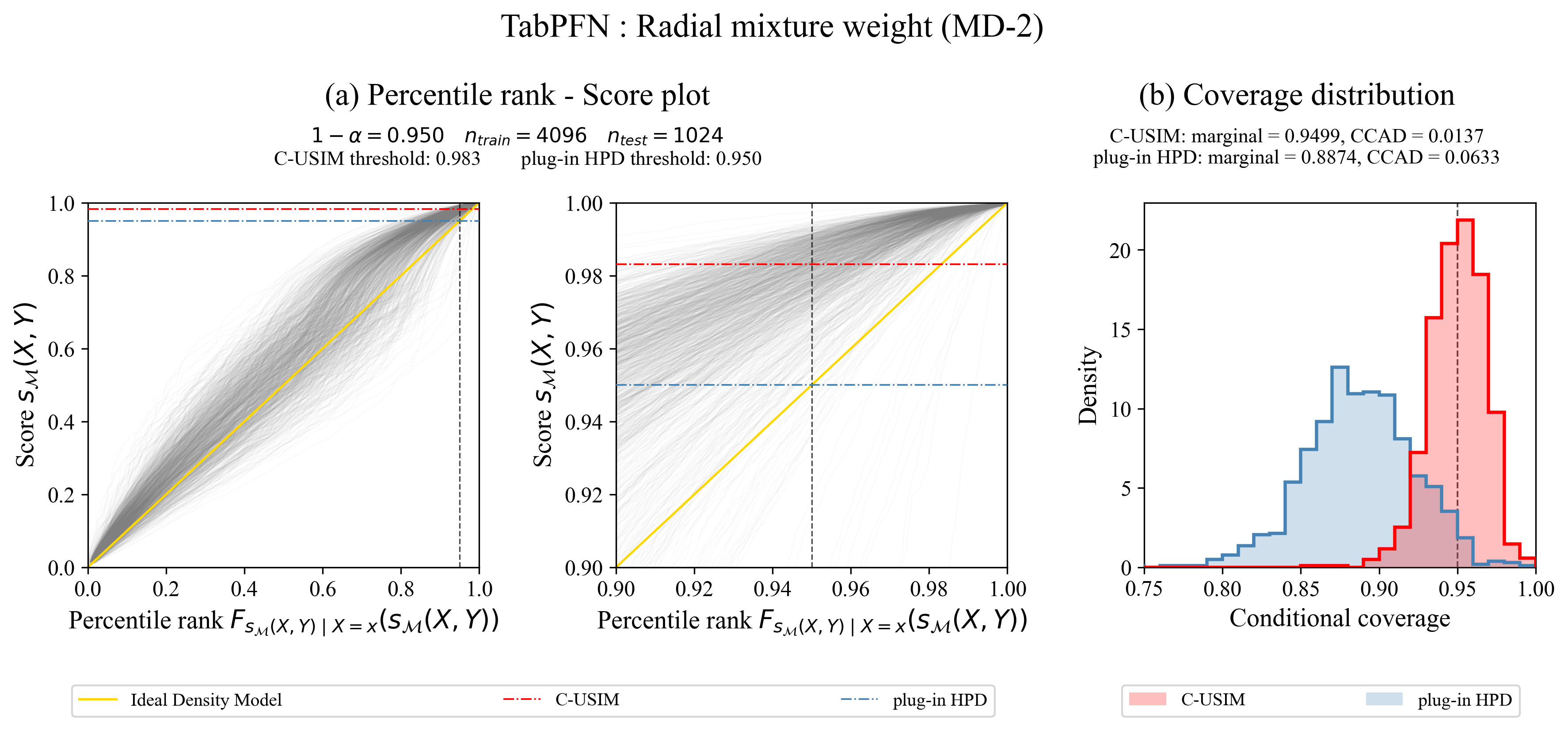}
\end{subfigure}
\hfill
\begin{subfigure}[b]{.86\textwidth}
    \centering
    \includegraphics[width=\textwidth]{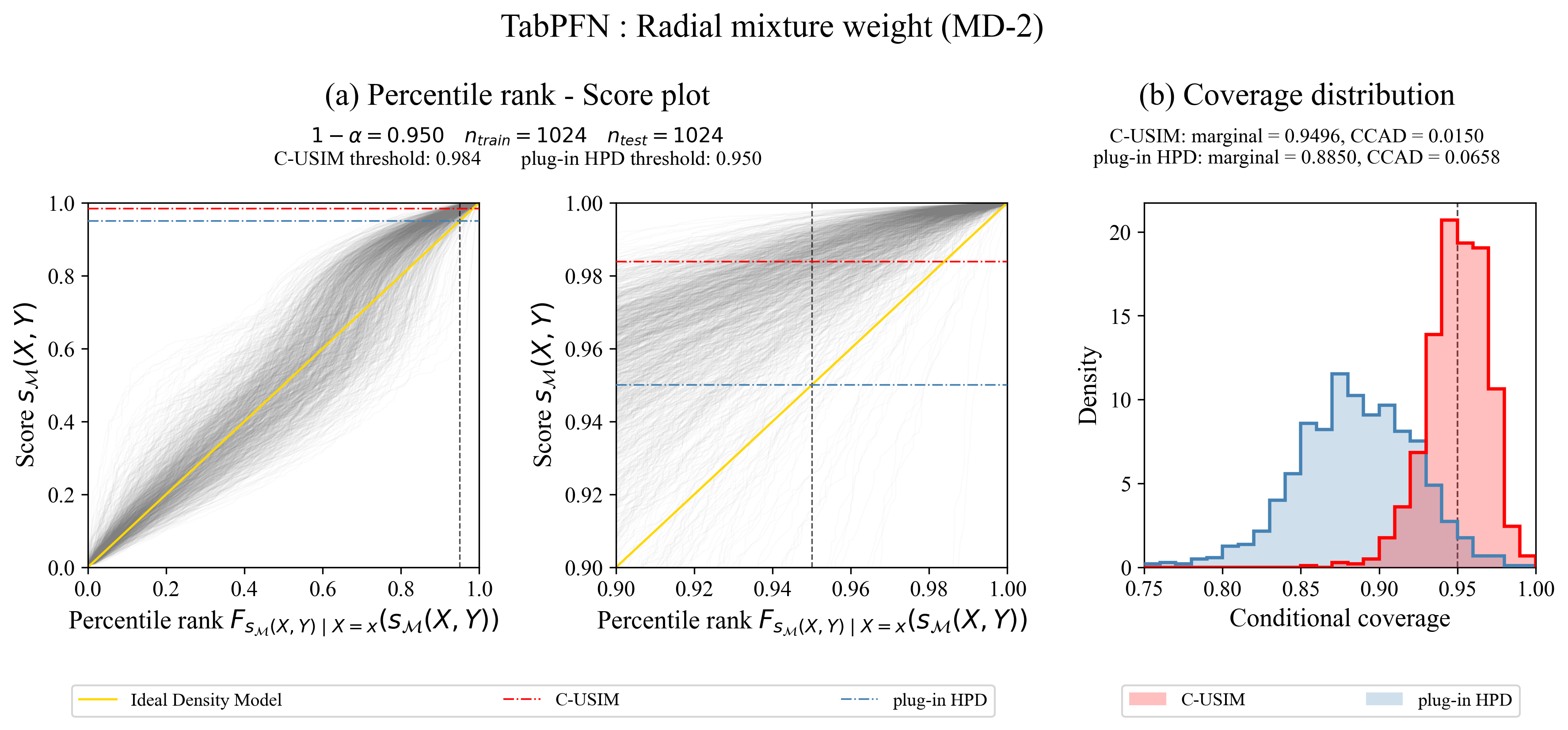}
\end{subfigure}
\hfill
\begin{subfigure}[b]{.86\textwidth}
    \centering
    \includegraphics[width=\textwidth]{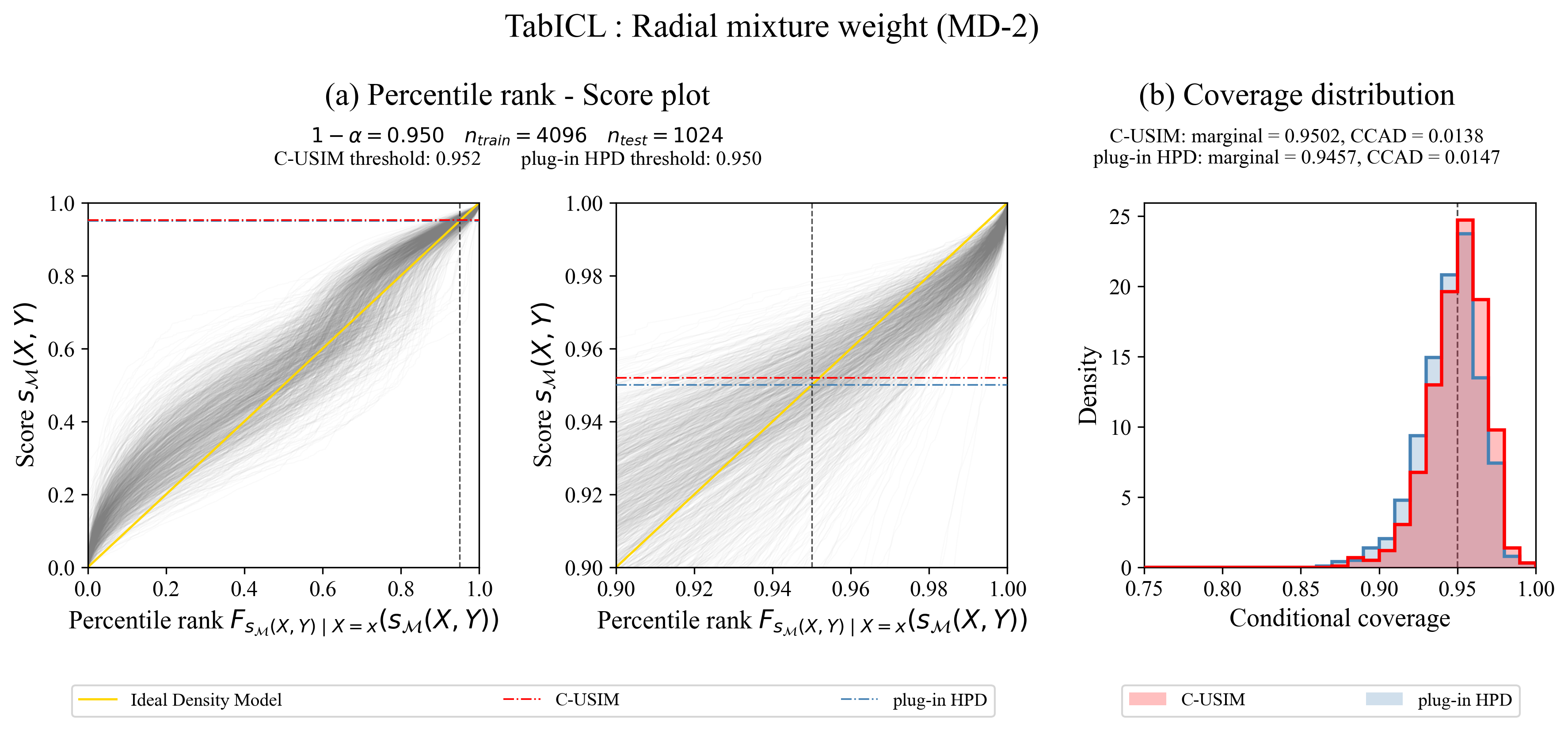}
\end{subfigure}
\hfill
\begin{subfigure}[b]{.86\textwidth}
    \centering
    \includegraphics[width=\textwidth]{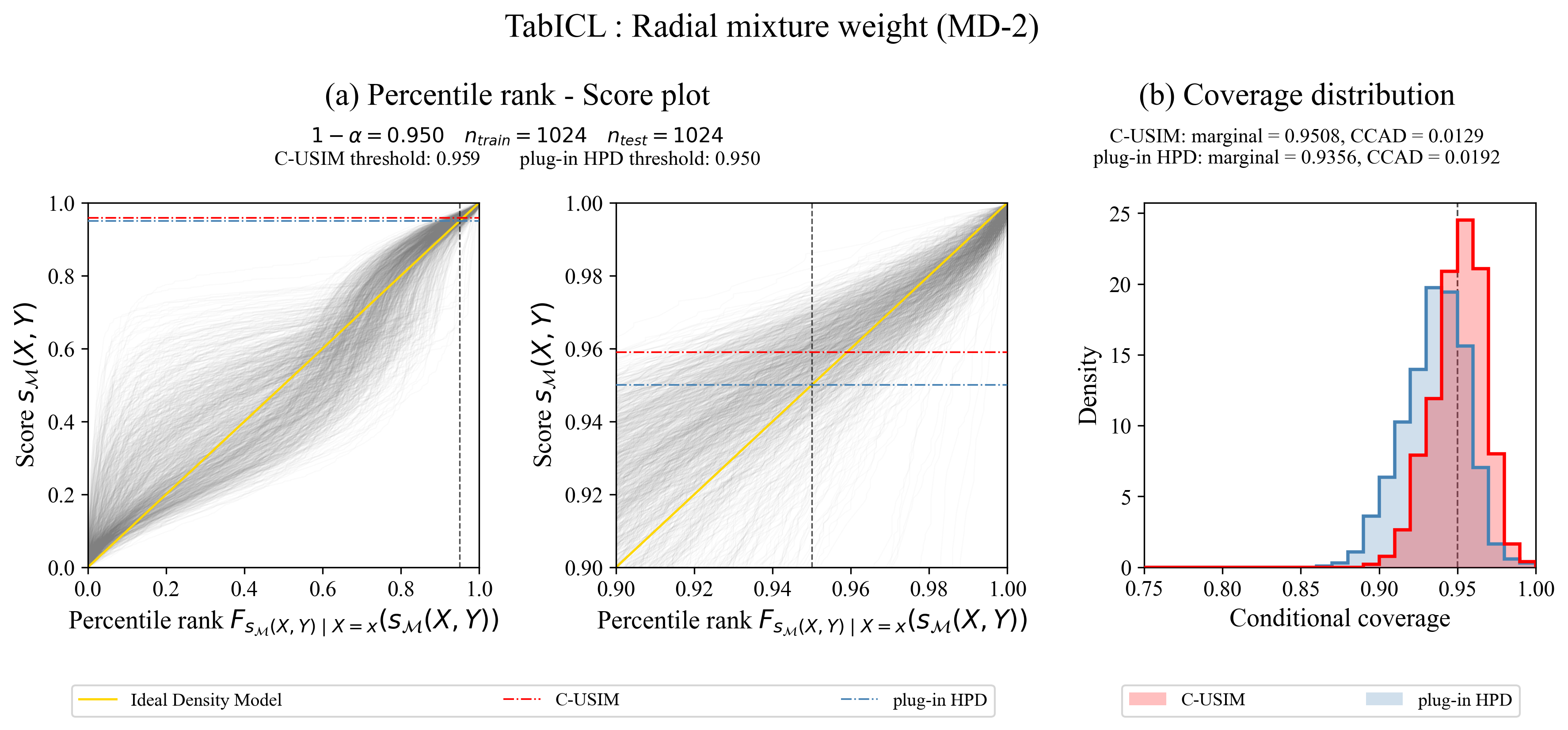}
\end{subfigure}

\caption{Percentile rank--score diagnostics for \ref{dgp:md-2}, seed 2026: C-USIM ($n_{\mathrm{cal}}\to\infty$) and plug-in HPD.}
\end{figure}

\begin{figure}[p]
\centering
\begin{subfigure}[b]{.86\textwidth}
    \centering
    \includegraphics[width=\textwidth]{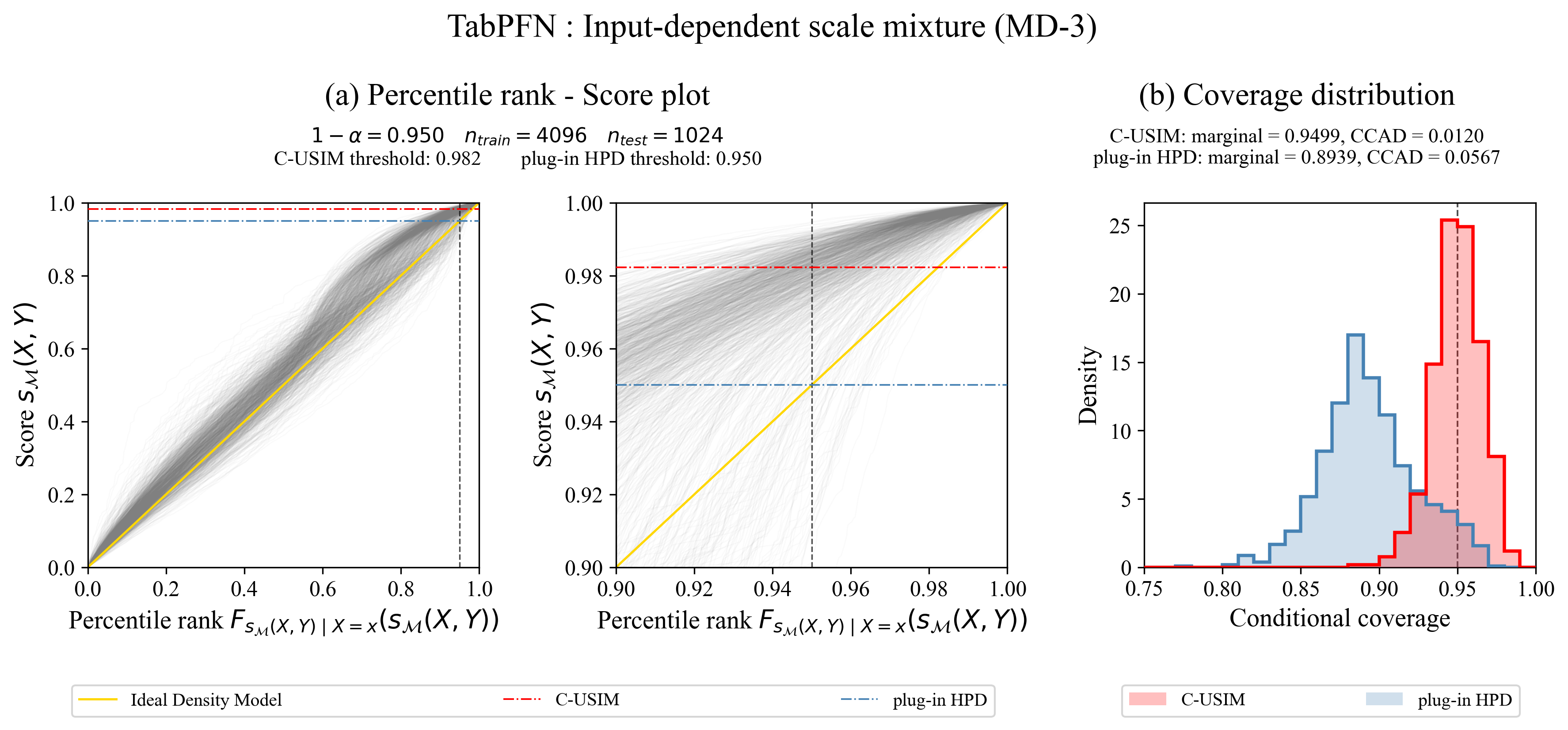}
\end{subfigure}
\hfill
\begin{subfigure}[b]{.86\textwidth}
    \centering
    \includegraphics[width=\textwidth]{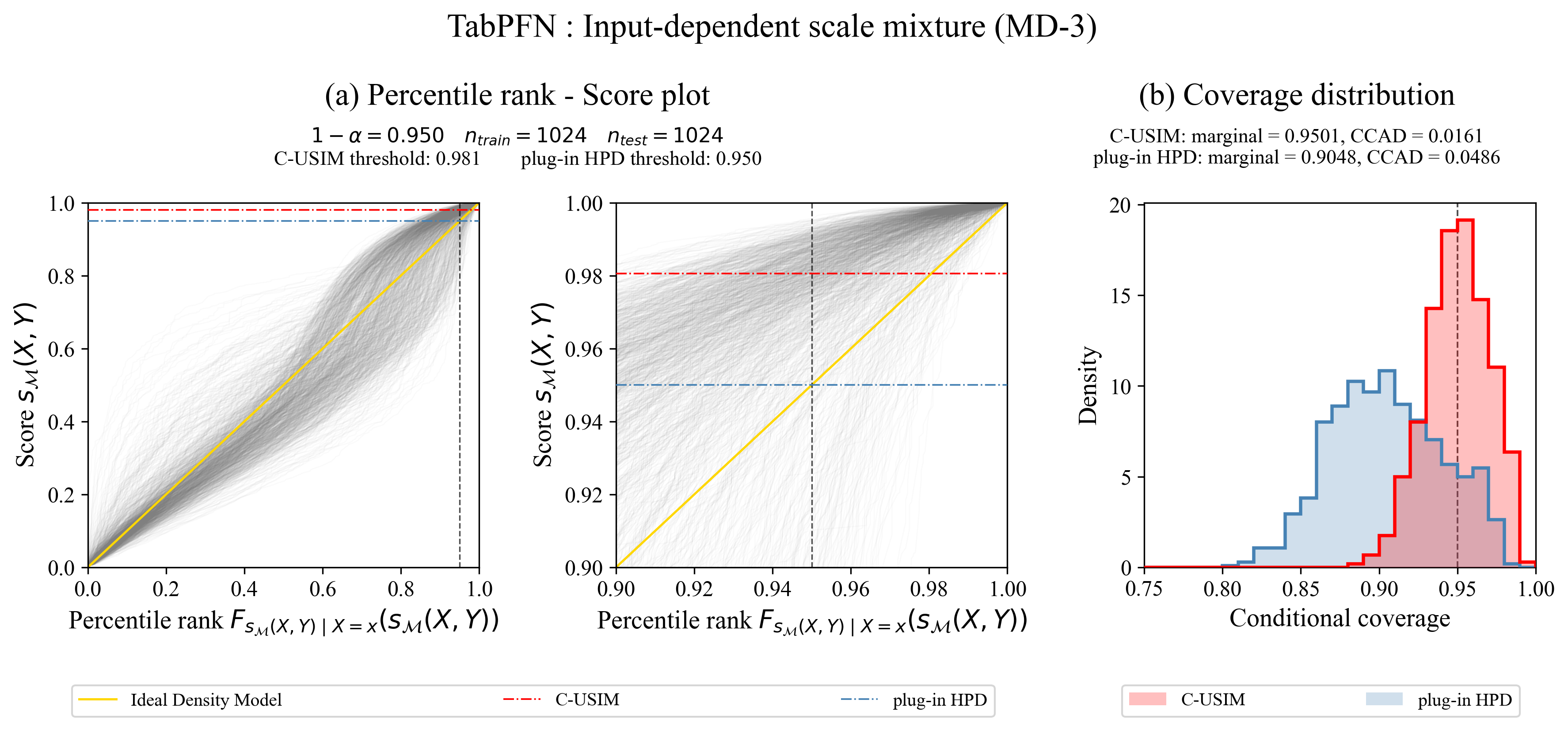}
\end{subfigure}
\hfill
\begin{subfigure}[b]{.86\textwidth}
    \centering
    \includegraphics[width=\textwidth]{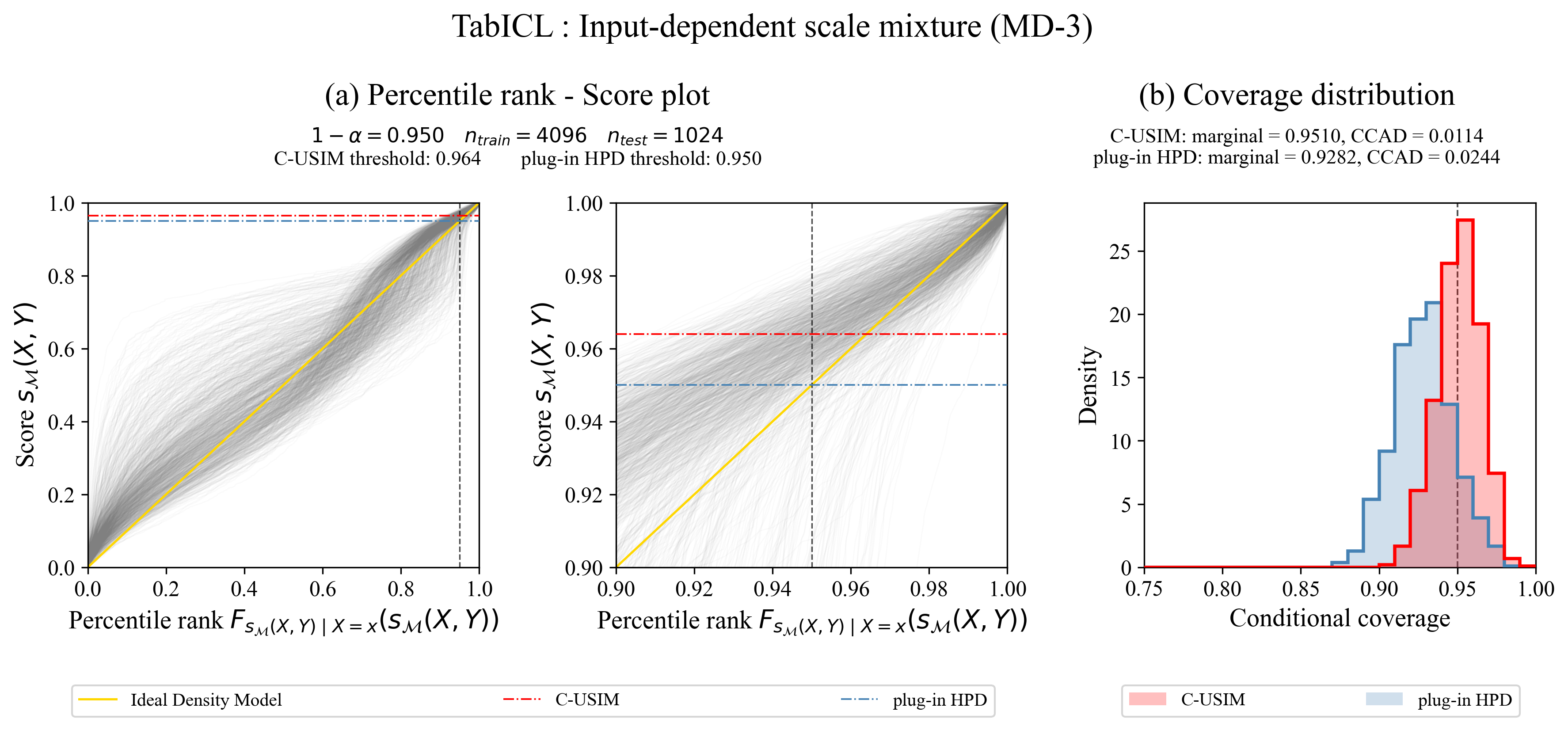}
\end{subfigure}
\hfill
\begin{subfigure}[b]{.86\textwidth}
    \centering
    \includegraphics[width=\textwidth]{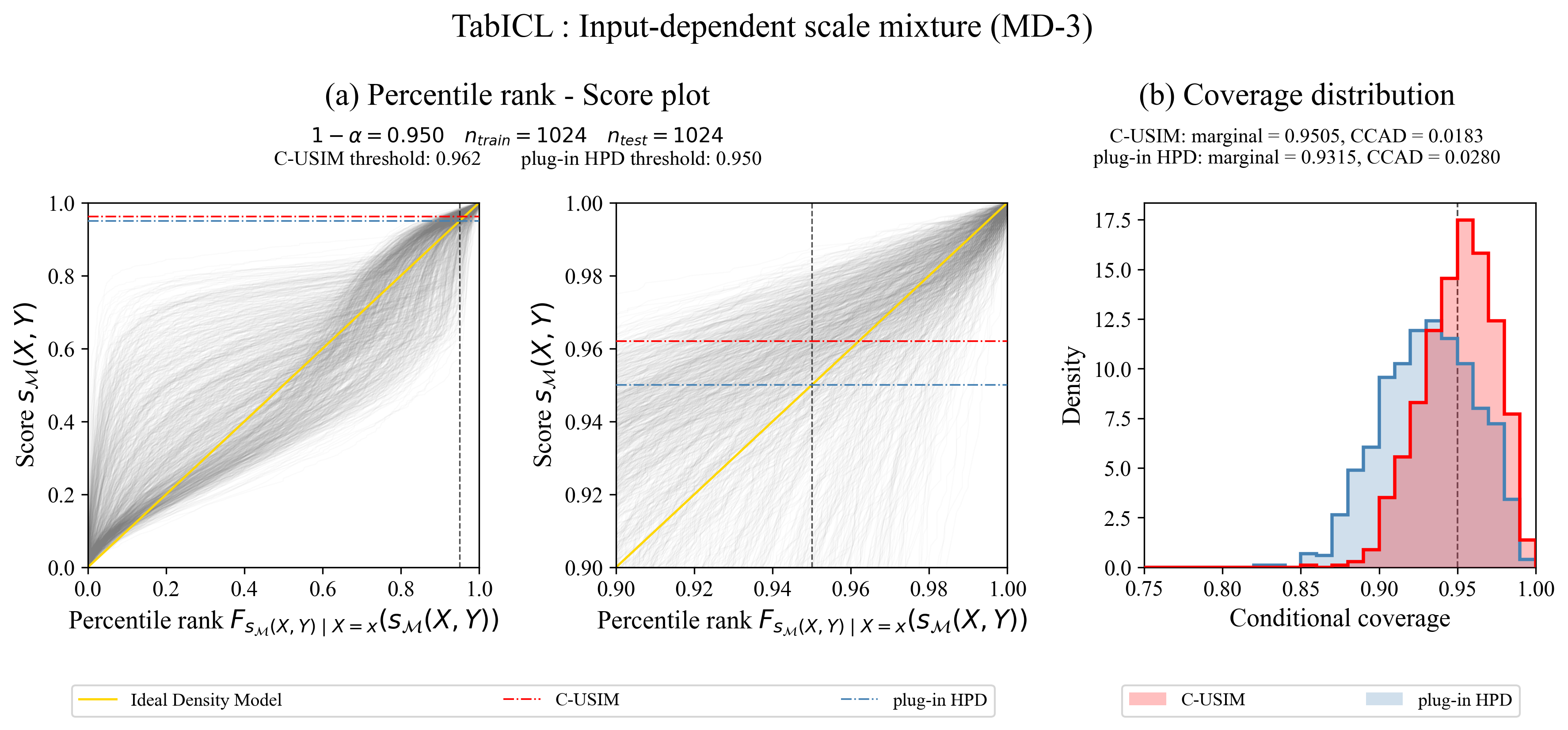}
\end{subfigure}

\caption{Percentile rank--score diagnostics for \ref{dgp:md-3}, seed 2026: C-USIM ($n_{\mathrm{cal}}\to\infty$) and plug-in HPD.}
\end{figure}

\begin{figure}[p]
\centering
\begin{subfigure}[b]{.86\textwidth}
    \centering
    \includegraphics[width=\textwidth]{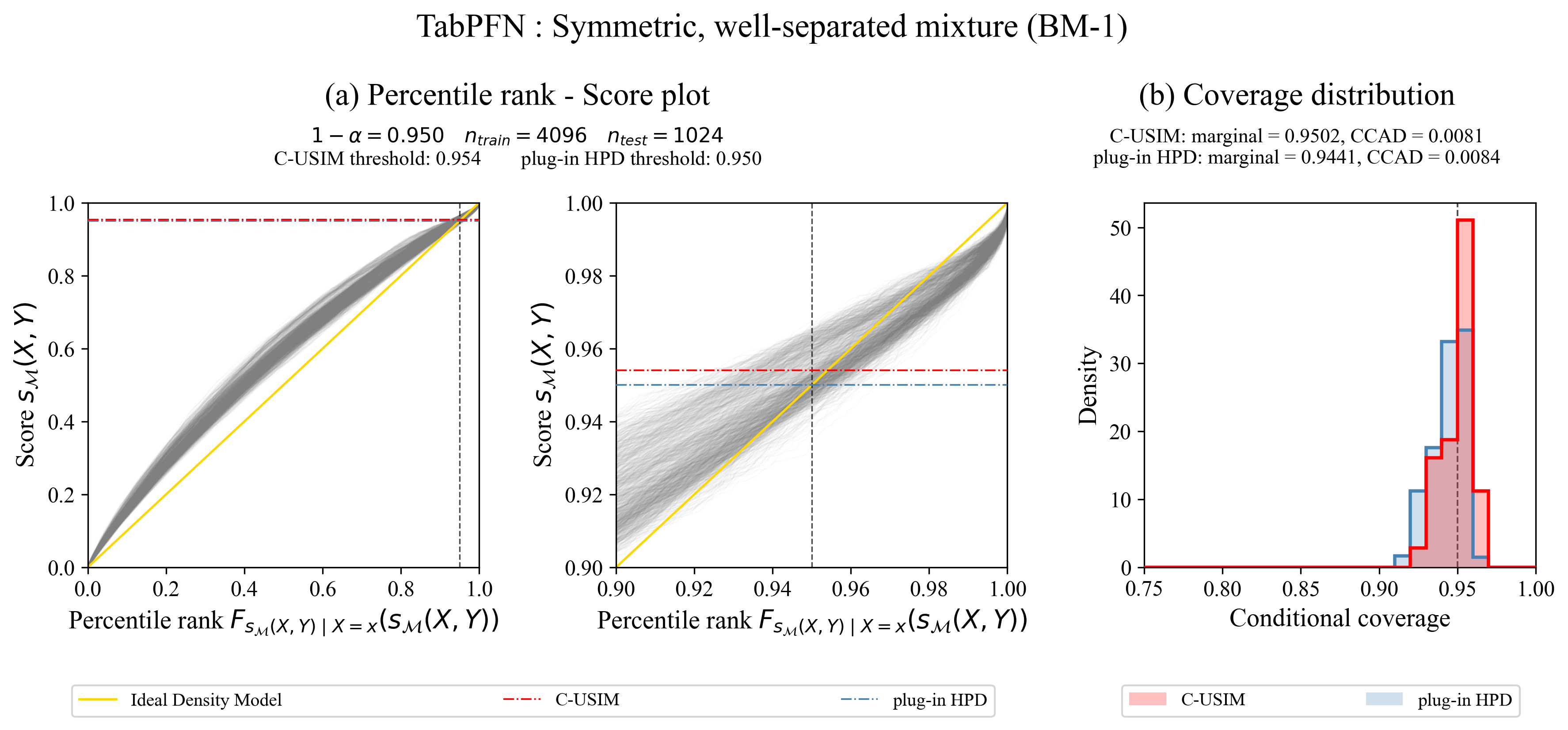}
\end{subfigure}
\hfill
\begin{subfigure}[b]{.86\textwidth}
    \centering
    \includegraphics[width=\textwidth]{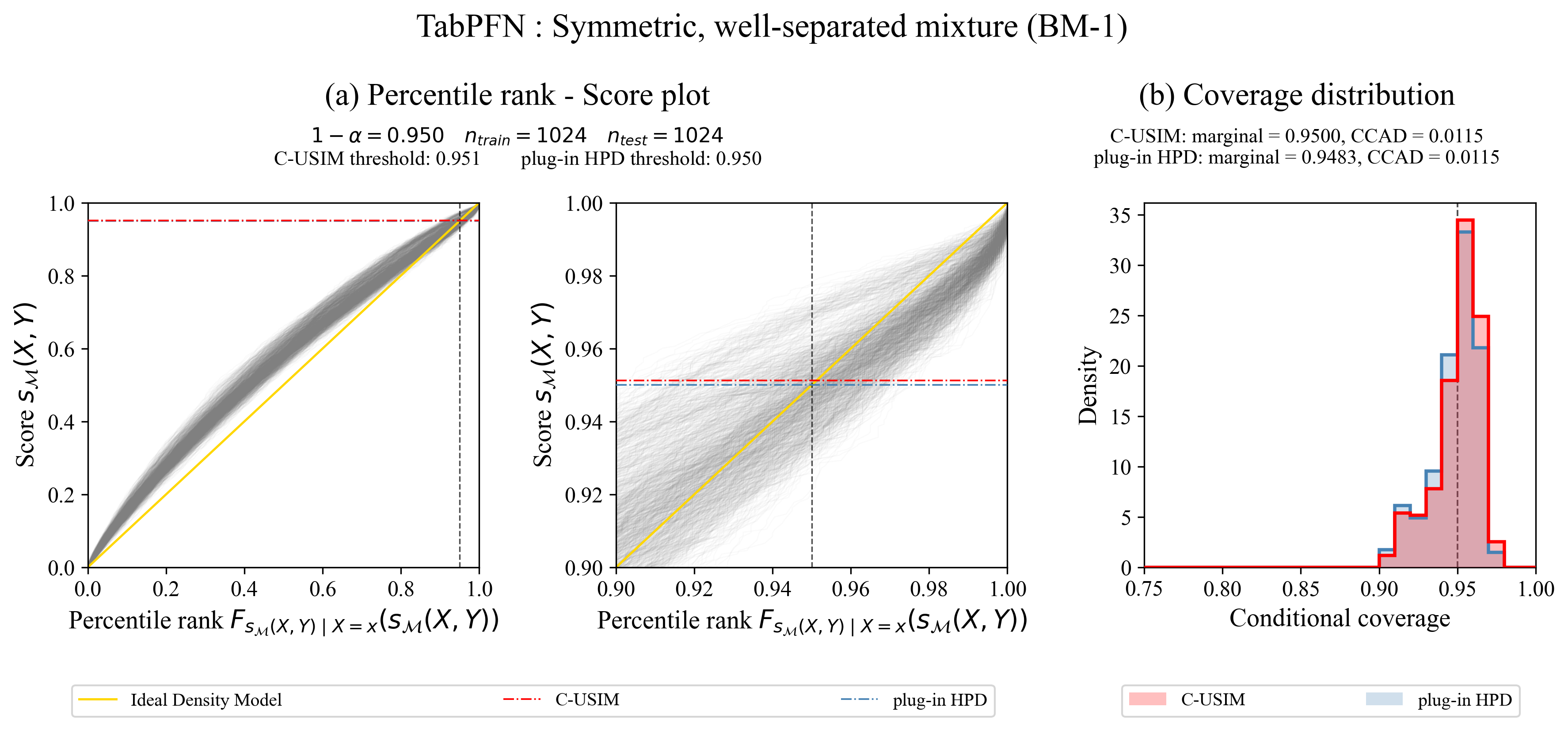}
\end{subfigure}
\hfill
\begin{subfigure}[b]{.86\textwidth}
    \centering
    \includegraphics[width=\textwidth]{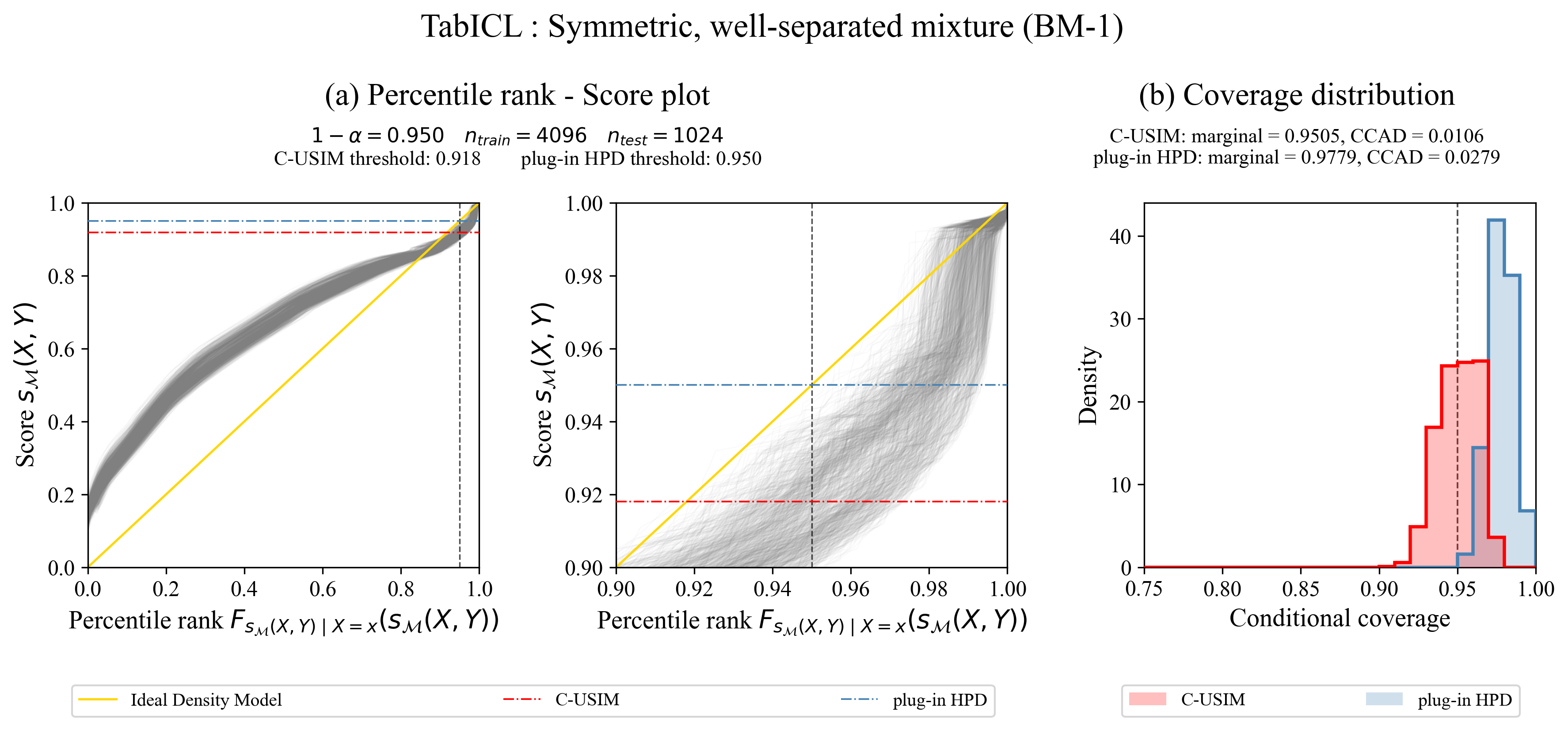}
\end{subfigure}
\hfill
\begin{subfigure}[b]{.86\textwidth}
    \centering
    \includegraphics[width=\textwidth]{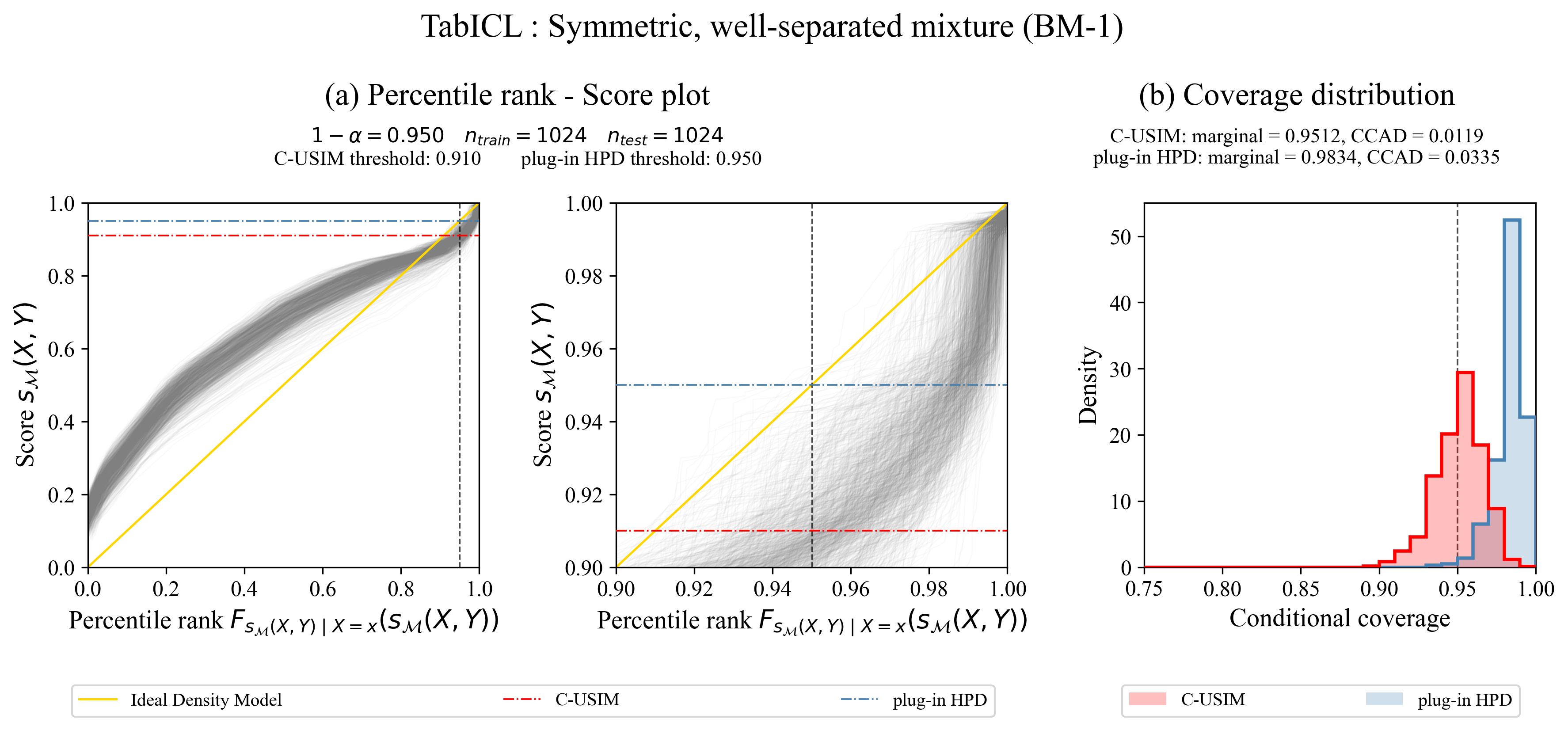}
\end{subfigure}

\caption{Percentile rank--score diagnostics for \ref{dgp:bm-1}, seed 2026: C-USIM ($n_{\mathrm{cal}}\to\infty$) and plug-in HPD.}
\end{figure}

\begin{figure}[p]
\centering
\begin{subfigure}[b]{.86\textwidth}
    \centering
    \includegraphics[width=\textwidth]{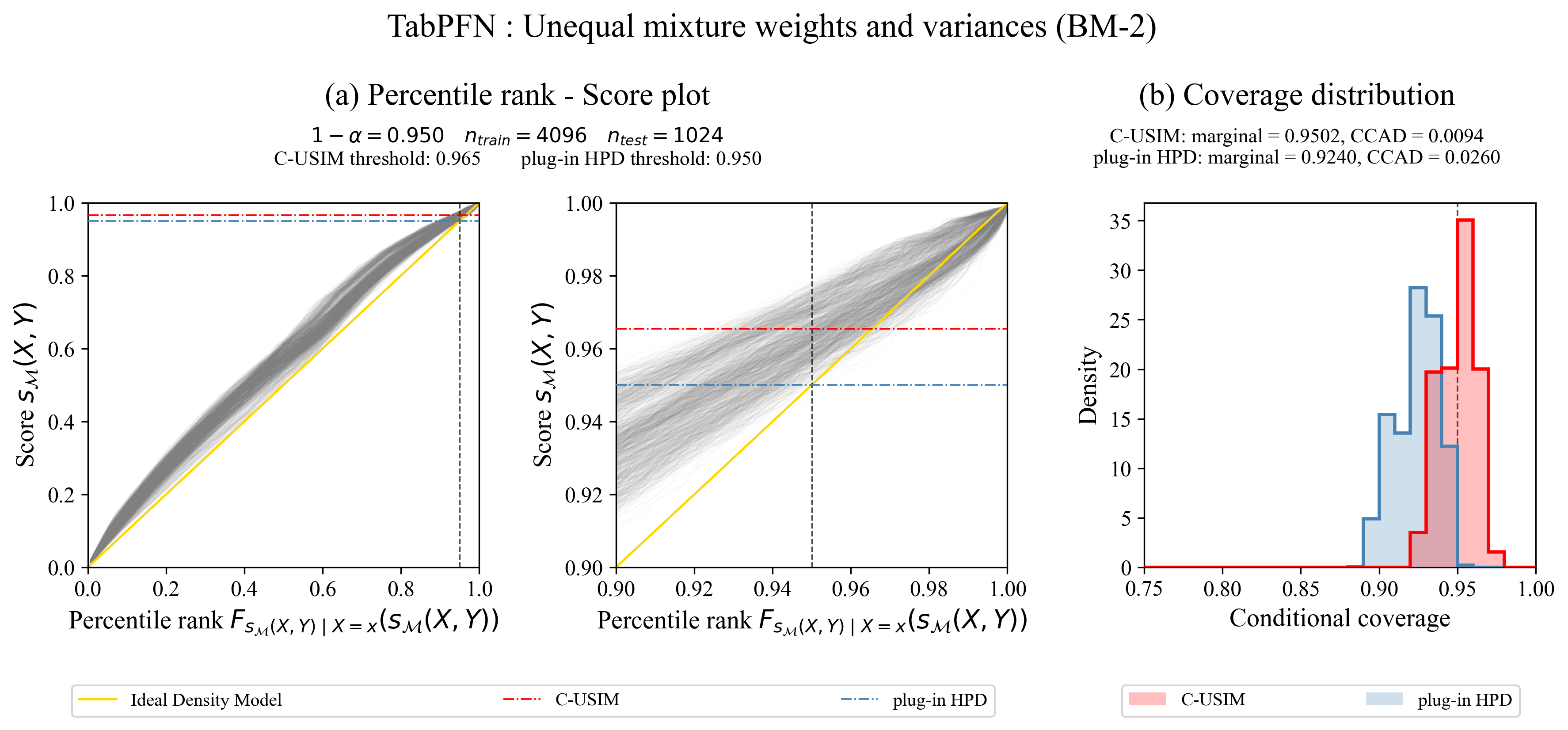}
\end{subfigure}
\hfill
\begin{subfigure}[b]{.86\textwidth}
    \centering
    \includegraphics[width=\textwidth]{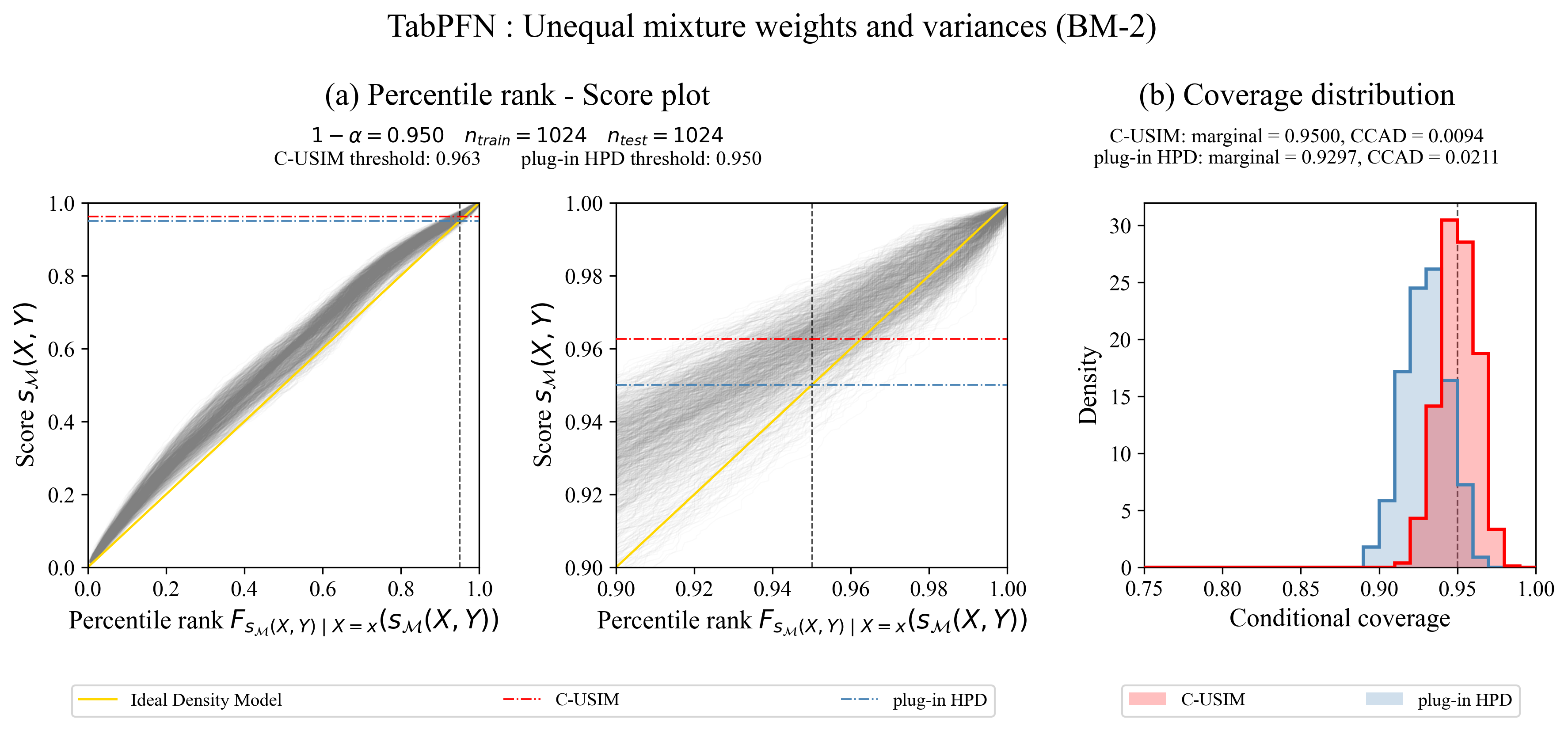}
\end{subfigure}
\hfill
\begin{subfigure}[b]{.86\textwidth}
    \centering
    \includegraphics[width=\textwidth]{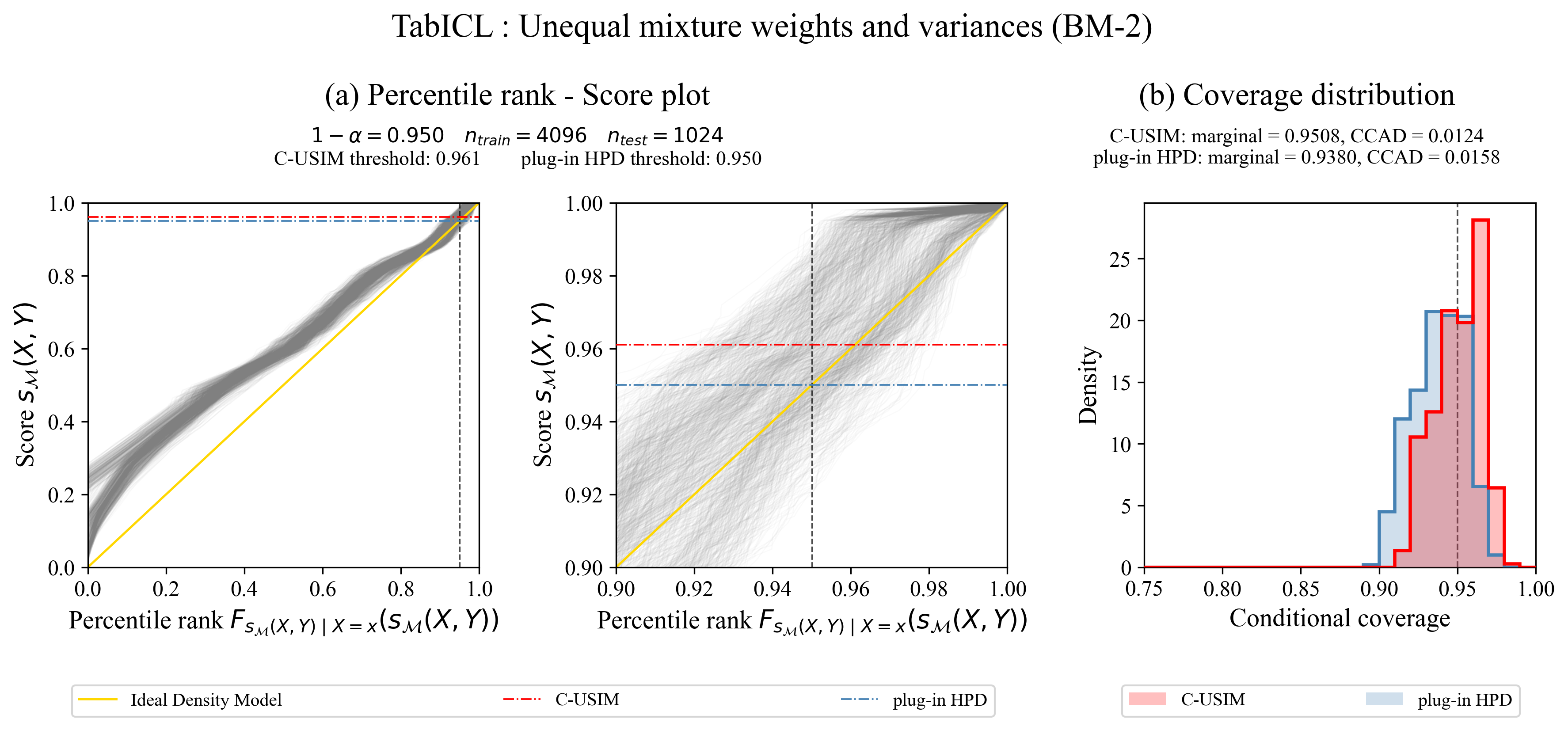}
\end{subfigure}
\hfill
\begin{subfigure}[b]{.86\textwidth}
    \centering
    \includegraphics[width=\textwidth]{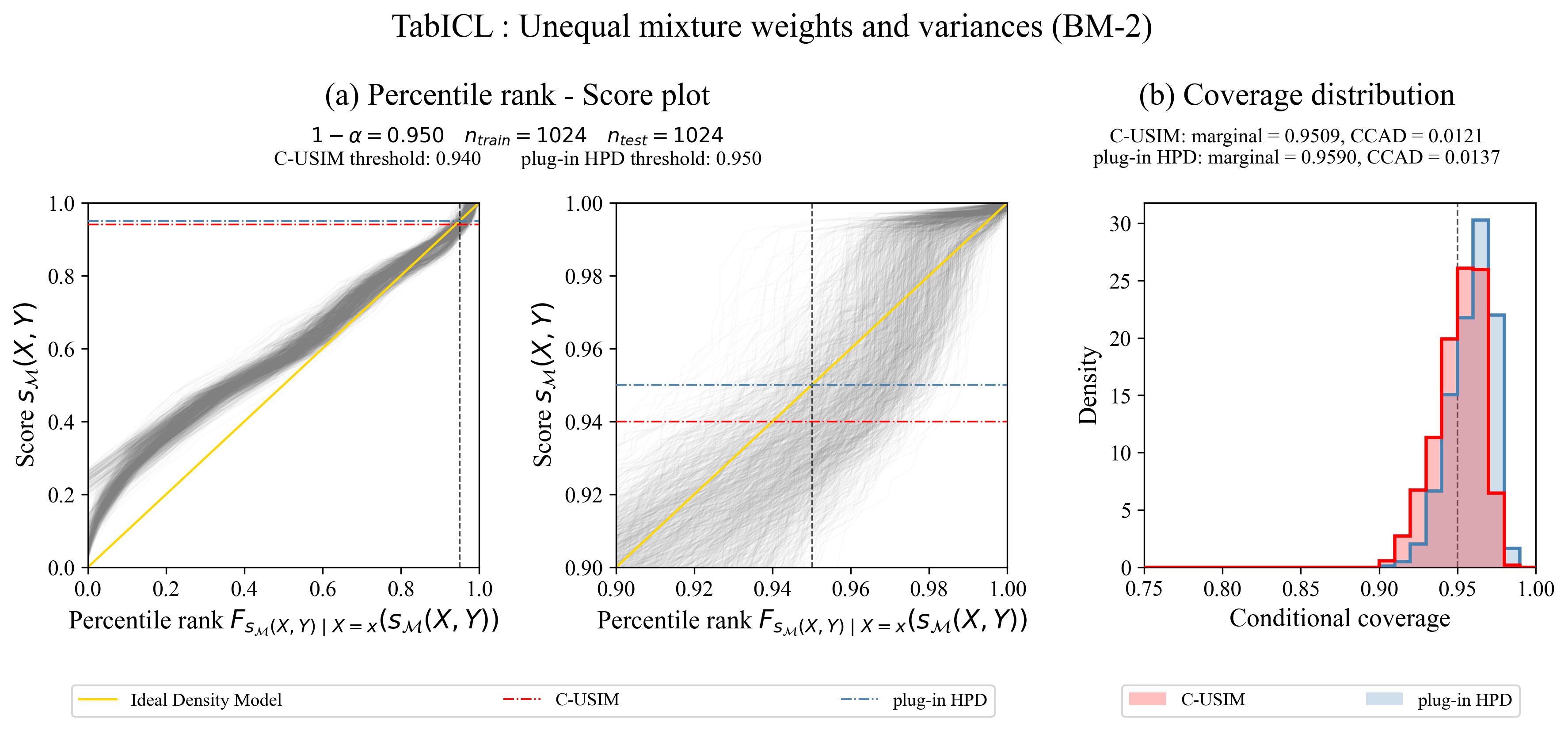}
\end{subfigure}

\caption{Percentile rank--score diagnostics for \ref{dgp:bm-2}, seed 2026: C-USIM ($n_{\mathrm{cal}}\to\infty$) and plug-in HPD.}
\end{figure}

\begin{figure}[p]
\centering
\begin{subfigure}[b]{.86\textwidth}
    \centering
    \includegraphics[width=\textwidth]{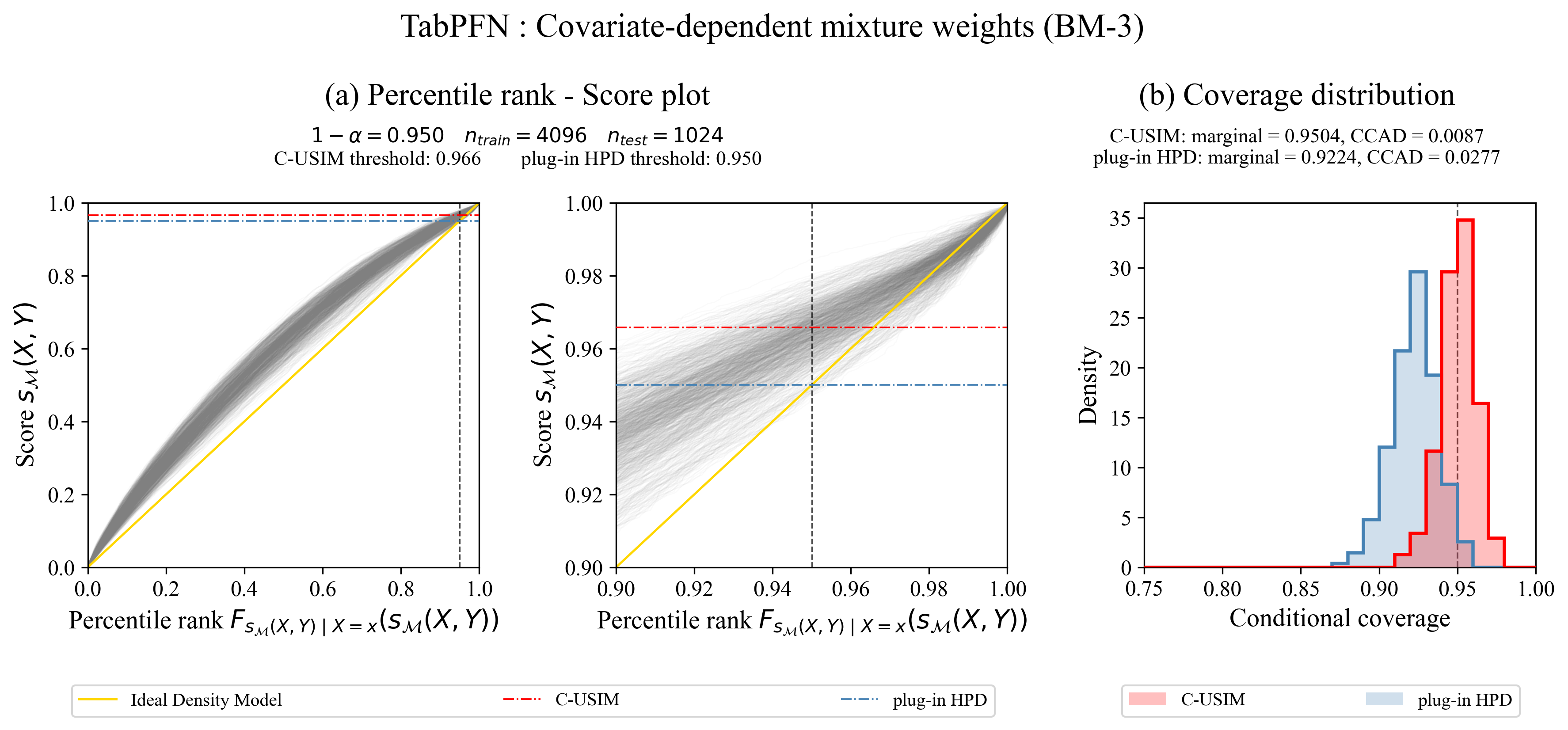}
\end{subfigure}
\hfill
\begin{subfigure}[b]{.86\textwidth}
    \centering
    \includegraphics[width=\textwidth]{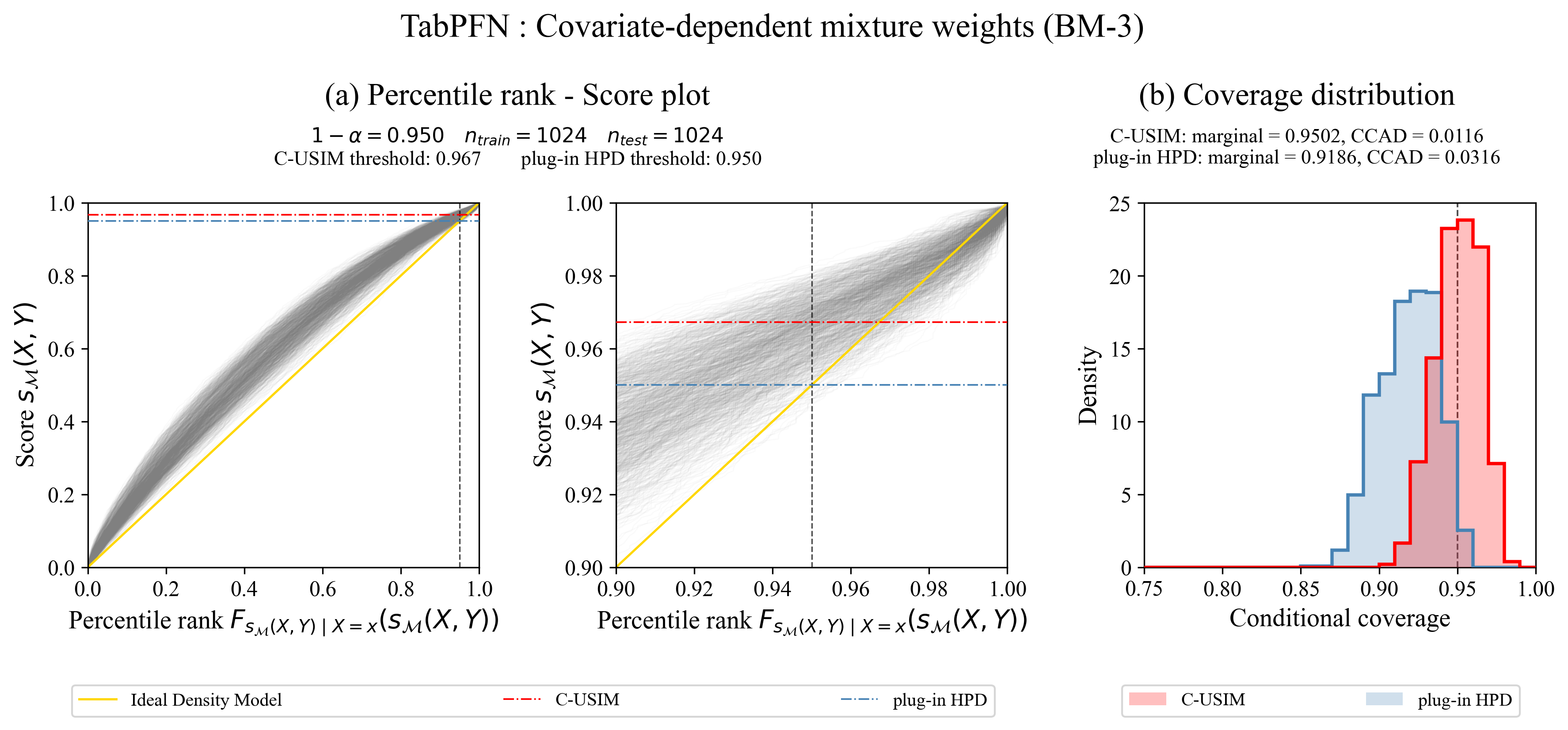}
\end{subfigure}
\hfill
\begin{subfigure}[b]{.86\textwidth}
    \centering
    \includegraphics[width=\textwidth]{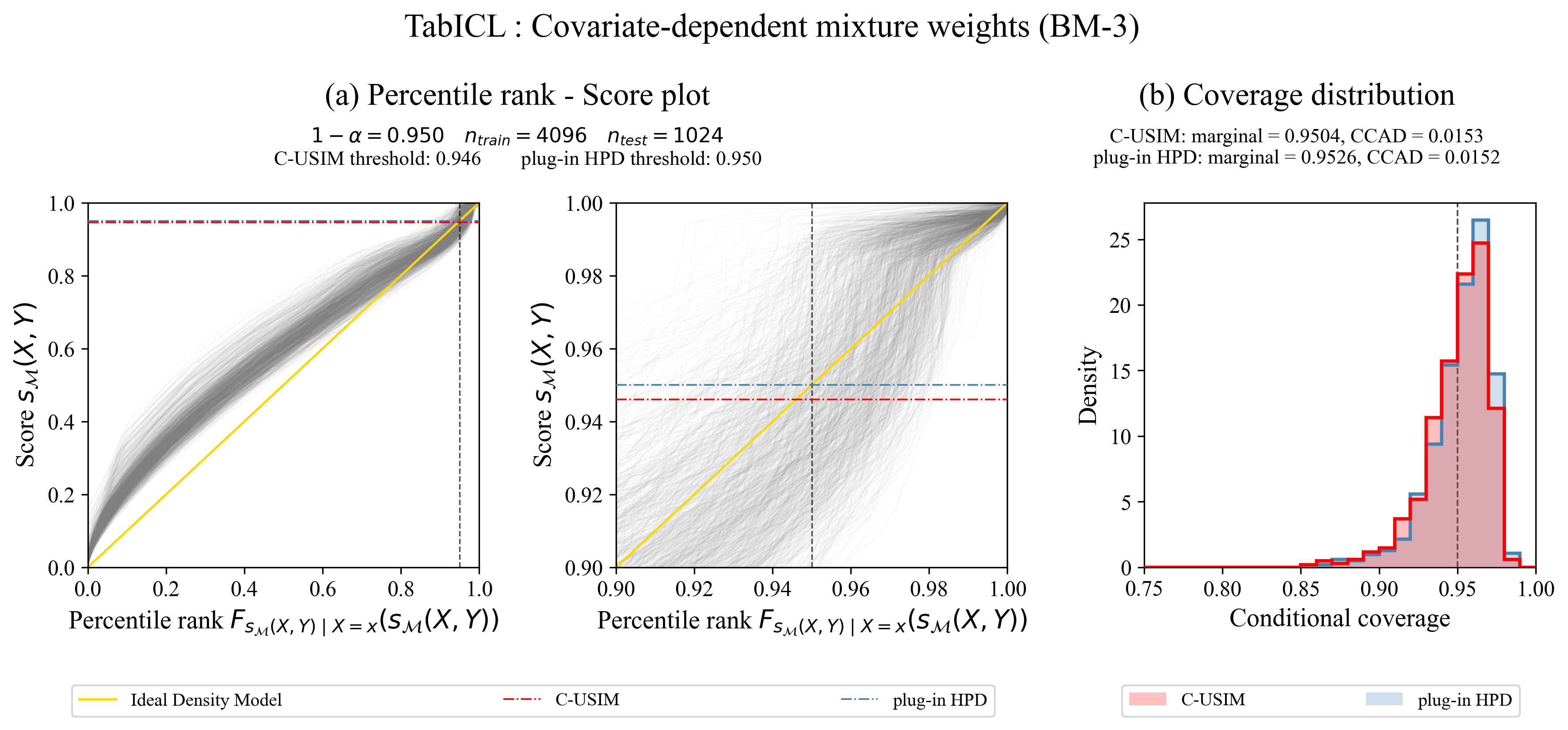}
\end{subfigure}
\hfill
\begin{subfigure}[b]{.86\textwidth}
    \centering
    \includegraphics[width=\textwidth]{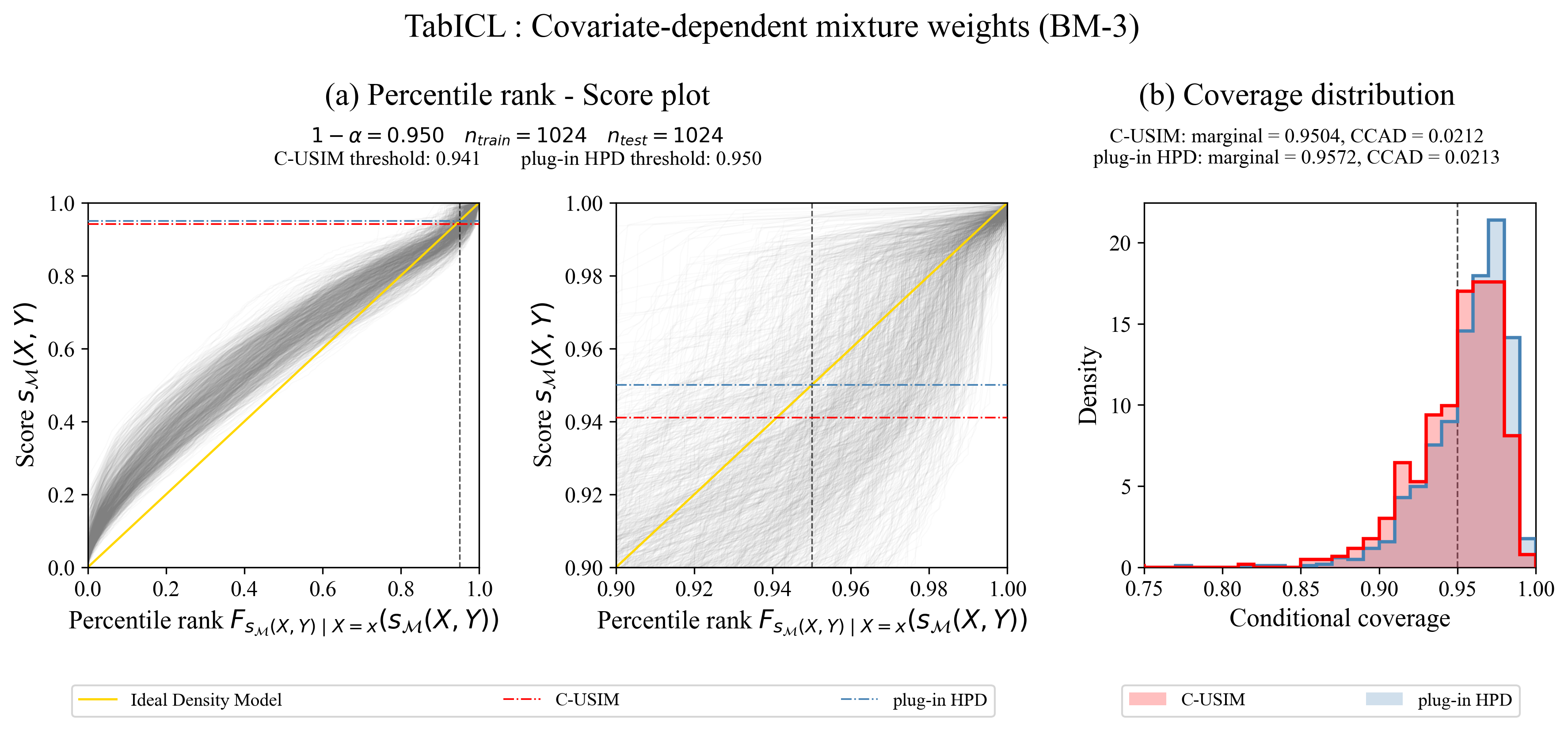}
\end{subfigure}

\caption{Percentile rank--score diagnostics for \ref{dgp:bm-3}, seed 2026: C-USIM ($n_{\mathrm{cal}}\to\infty$) and plug-in HPD.}
\end{figure}
\clearpage

\subsection{Evaluation measure definitions}
\label{app:evaluation-measures}

For each run, we condition on the fitted predictor, calibration data, and algorithmic randomness, treating $\widehat{C}_{1-\alpha}$ as a fixed prediction-set function. Its conditional coverage, marginal coverage, and conditional coverage absolute deviation (CCAD) are
\begin{align*}
c(x) &:= \mathbb{P}\!\left(Y\in\widehat{C}_{1-\alpha}(x)\mid X=x\right),
\label{eq:experiment-conditional-coverage}\\
\mathrm{Cov} &:= \mathbb{E}_{X}[c(X)],
&\mathrm{CCAD} &:= \mathbb{E}_{X}\!\left[\left|c(X)-(1-\alpha)\right|\right].
\end{align*}
The expectation is over the test covariate distribution.

Unlike absolute marginal coverage error, CCAD takes absolute deviations before averaging over inputs, so undercoverage and overcoverage cannot cancel. A reduction in CCAD therefore represents improved average conditional coverage accuracy, including when achieved through a shared threshold adjustment. Correcting marginal undercoverage can contribute to this improvement.

Below, we describe how these quantities are evaluated on synthetic and real data, followed by prediction-set length and aggregation across runs.

For synthetic data, let $X_i$ be a test input and let $Y_{ij}\sim P(\cdot\mid X_i)$, $j=1,\ldots,M$, be independent conditional response draws, with $M=1000$. Both methods use the same inputs and response draws. With $\mathbf{1}\{\cdot\}$ denoting the indicator function, we compute
\begin{align*}
\widehat{c}_i &= \frac{1}{M}\sum_{j=1}^{M}
\mathbf{1}\!\left\{Y_{ij}\in\widehat{C}_{1-\alpha}(X_i)\right\},
\\
\widehat{\mathrm{Cov}}_{\mathrm{MC}}
&=\frac{1}{n_{\mathrm{test}}}\sum_{i=1}^{n_{\mathrm{test}}}\widehat{c}_i,
&\widehat{\mathrm{CCAD}}
&=\frac{1}{n_{\mathrm{test}}}\sum_{i=1}^{n_{\mathrm{test}}}
\left|\widehat{c}_i-(1-\alpha)\right|.
\end{align*}
The empirical CCAD retains Monte Carlo error from both the test inputs and conditional response draws.

For real data, let $A_1,\ldots,A_K$ be a disjoint covariate partition learned from validation inputs alone. For the test observations $(X_i,Y_i)$, define $I_g=\{i:X_i\in A_g\}$, $n_g=|I_g|$, and
\begin{equation*}
h_i=\mathbf{1}\!\left\{Y_i\in\widehat{C}_{1-\alpha}(X_i)\right\},
\qquad
\widehat{c}_g=\frac{1}{n_g}\sum_{i\in I_g}h_i\quad(n_g>0).
\label{eq:experiment-group-coverage}
\end{equation*}
Observed marginal coverage is $\widehat{\mathrm{Cov}}_{\mathrm{obs}}=n_{\mathrm{test}}^{-1}\sum_{i=1}^{n_{\mathrm{test}}}h_i$. We summarize absolute group coverage errors in two ways. Mean group error gives equal weight to each nonempty test group. CEC-X instead weights each group's absolute error by its share of the test observations:
\begin{equation*}
\widehat{\mathrm{CEC}}_{X,1}
=\sum_{g:\,n_g>0}\frac{n_g}{n_{\mathrm{test}}}
\left|\widehat{c}_g-(1-\alpha)\right|,
\label{eq:experiment-cecx}
\end{equation*}
where empty test groups contribute zero. This is an absolute, not squared, coverage error over a finite partition and is not a direct estimate of pointwise CCAD.

We report prediction-set length as the sum of the lengths of all component intervals, excluding gaps between disjoint intervals. Infinite lengths remain in the averages. For each real-world dataset, coverage uses the entire test set, whereas mean length uses the first 256 inputs in the fixed test order, shared across seeds, models, and methods. The test sets contain 9,731 observations for Journal SJR and Allstate, and 5,372 for JP Anime.

In seed-repeated analyses, absolute gaps, CCAD, mean group error, and CEC-X are computed within each seed before averaging across seeds. Taking the absolute gap of the mean coverage can give a different result.


\subsection{Experimental settings}
\label{app:settings}

Table~\ref{tab:experiment-settings} summarizes the main experimental settings. Settings specific to JP Anime and Allstate Claims Severity are given in Appendix~\ref{app:additional-realdata}. In the synthetic and real-world method comparisons, the two methods share evaluation samples and the model random seed within each run, with a total budget of 1,536 labels and distinct contexts of 1,536 and 512 labels. At target coverage $1-\alpha=0.95$, the calibrated cutoff in these comparisons is the 974th of 1,024 calibration scores. The split-ratio sweep instead varies the context and calibration sizes for C-USIM while keeping $n_{\mathrm{train}}+n_{\mathrm{cal}}=1536$.

\begin{table}[htbp]
\centering
\caption{Experimental settings. Synthetic and split-ratio sample counts are per mechanism and run; conditional response draws are additional Monte Carlo samples. Split-ratio coverage and CCAD use the known conditional CDF rather than the stored response draws. A dash denotes an unused setting.}
\label{tab:experiment-settings}
\small
\begin{tabular}{@{}p{0.32\linewidth}p{0.16\linewidth}p{0.23\linewidth}p{0.19\linewidth}@{}}
\hline
Setting & Synthetic & Journal SJR & Split ratio \\
\hline
Dataset size & 1,792 per run & 27,803 cleaned rows & 1,792 per run \\
Total label budget & 1,536 & 1,536 & 1,536 \\
Input dimensions & 1, 2, 5, 10, 20 & 7 (categorical) & 1, 2, 5, 10, 20 \\
Run seeds & 100--109 & 12100--12299 & 25600--25649 \\
Number of seeds & 10 & 200 & 50 \\
Plug-in context observations & 1,536 & 1,536 & -- \\
C-USIM context observations & 512 & 512 & 256--1,280 \\
Calibration observations & 1,024 & 1,024 & 256--1,280 \\
Validation observations & -- & 512 & -- \\
Test observations & 256 & 9,731 & 256 \\
Conditional draws per test input & 1,000 & -- & 1,000 (stored) \\
Target coverage & 95\% & 95\% & 95\% \\
Group counts ($K$) & -- & 5, 10, 15, 20, 30, 40 & -- \\
Clustering seeds & -- & 1717, 2717, 3717, 4717, 5717 & -- \\
Representative grouping & -- & $K=10$, seed 1717 & -- \\
\hline
\multicolumn{4}{@{}p{\linewidth}@{}}{Split-ratio training--calibration allocations:
$256{:}1280$, $512{:}1024$, $768{:}768$, $1024{:}512$, $1229{:}307$, and
$1280{:}256$. The fifth allocation is the closest integer split to $8{:}2$.
All nine synthetic mechanisms are evaluated with both models.} \\
\hline
\end{tabular}
\end{table}

\paragraph{Synthetic sampling.} Conditional response draws are independent at each test input and shared across models and allocation arms. Table~\ref{tab:experiment-settings} reports the number of test inputs and response draws; generating equations are given in Appendix~\ref{app:synthetic-dgps}.

\paragraph{Journal SJR preprocessing and groups.} The categorical inputs are publication type, country, region, publisher, coverage years, subject categories, and subject areas. The response and prediction-set lengths use the $\log_{10}(H\text{-index}+1)$ scale. For TFM inference, missing categories receive a dedicated token, and categorical inputs are retained. Each seed uses the same context/calibration row allocation for both models. For each fitted context, all query covariates are processed in one prediction call: 9,731 test inputs for plug-in HPD, and 1,024 calibration plus 9,731 test inputs for C-USIM. Calibration scores and the cutoff are computed from this joint output, using the finite-density reconstruction described above. For Journal SJR, the validation and test sets remain fixed across run seeds and are shared by both models. For group construction, we impute missing categorical values with their validation-set modes and apply one-hot encoding before fitting $K$-means with 10 initializations per clustering seed. The imputer, encoder, and cluster centers are fitted on validation covariates alone, without using responses. Test covariates undergo the same transformation and are assigned to their nearest cluster centers by Euclidean distance in the encoded feature space; categories absent from validation are encoded as all zeros in the corresponding feature block. We evaluate every combination of the six group counts and five clustering seeds in Table~\ref{tab:experiment-settings}, giving 30 groupings in total. Each grouping uses its specified number of groups, with group assignments shared by both models and methods and reused across run seeds. The representative grouping was chosen before the expanded evaluation. For each seed, we compute $|\widehat c_g-0.95|$ for every nonempty test group before averaging.

\paragraph{Numerical implementation and checks.} Both models are configured with eight estimators and FP32 parameters and outputs. For one-dimensional inputs, the default TabICL ensemble generator produces two distinct preprocessing configurations; for the multidimensional inputs it produces eight. TabICL computes scaled dot-product attention in FP64 before returning FP32 activations; this numerical policy is shared by the synthetic and real-world studies. For every synthetic fitted context, numerical checks compare repeated predictions and predictions under singleton, short, extreme-mixed, and permuted queries. For every real-world seed and fitted context, we permute the entire query table, restore the original row order, and compare predictions at every query input. We also verify that each ensemble forward receives the full query table; internal batching over ensemble members preserves all query rows. For the additional datasets, FP64 attention may also be batched over independent leading batch items to limit memory use; the query and key dimensions of each attention problem are kept intact. All checks use the fixed tolerance $10^{-4}+10^{-4}|\mathrm{reference}|$. These checks support the reported numerical implementation but do not establish query independence for every possible input.

\paragraph{Computational cost.} The lightweight aspect of C-USIM is that, once the joint calibration and test outputs are available, it requires only post-processing, with no additional training or model inference. No separate density or score-correction model is fitted. For a reconstructed density with $B$ intervals, the implementation sorts density levels once and computes cumulative probability masses, combining all intervals with equal density into the same score level. This requires $O(B\log B)$ time and $O(B)$ storage per query; extracting the prediction set from these quantities takes $O(B)$ time. Calibration adds an empirical order-statistic calculation. These costs exclude TFM inference.

\subsection{Detailed experimental results}
\label{app:detailed-results}
We report the synthetic results first, followed by Journal SJR and the two additional real-world datasets. The tables summarize performance across seeds, and the figures show the corresponding distributions and illustrative prediction sets.

\subsubsection{Synthetic}
\label{app:synthetic-results}
Table~\ref{tab:synthetic-calibration} and Figure~\ref{fig:experiment-calibration} summarize all nine mechanisms, over ten seeds per model. Figures~\ref{fig:synthetic-notebook-1d-1}--\ref{fig:synthetic-notebook-bm-3} show prediction sets and estimated conditional-coverage distributions for each mechanism, with TabPFN above TabICL. These illustrations use seed 100, fixed before evaluation. The comparison allocates the same 1,536 labels to either the plug-in context or the C-USIM context and calibration sets; the fitted predictors therefore differ between methods.

For TabICL, marginal coverage improves on \ref{dgp:1d-1} while CCAD increases. Undercoverage and overcoverage across inputs can offset each other in marginal coverage, whereas both contribute to the absolute errors measured by CCAD. With the training context and the resulting TFM predictive distributions held unchanged, raising only the shared calibration threshold in the rank--score plot can reduce undercoverage at some inputs while increasing overcoverage at others. If the added overcoverage outweighs the reduction in undercoverage on average, CCAD increases even as marginal coverage approaches the target. This balance depends on the model-induced score distributions across inputs. The fitted predictors also differ in the present comparison, so the observed CCAD difference does not isolate the effect of threshold adjustment.

\paragraph{Reading the prediction-set plots.} For each model and mechanism, plots (a) and (b) show Plug-in HPD(1536) and C-USIM, respectively. They display 80 of the 256 test inputs, selected at evenly spaced ranks of the first covariate using the same rows for both methods and models. Each prediction set is centered at its observed test response, so the black vertical line at zero marks that response; gray sets miss it. The two methods share an interval-axis range within each model and mechanism. All finite interval endpoints are retained, and arrows indicate unbounded components. In the multidimensional examples, ordering by the first covariate does not define a one-dimensional conditional slice.

\paragraph{Reading the coverage distributions.} Plot (c) uses all 256 test inputs with equal weight. Conditional coverage at each input is estimated from 1,000 shared draws from the known conditional response distribution. Both methods use the same 35 histogram bins, spanning their combined observed range with padding; no observations are removed. A dashed line marks the $0.95$ target.


\begin{figure}[p]
\centering
\includegraphics[width=\linewidth,height=0.21\textheight,keepaspectratio]{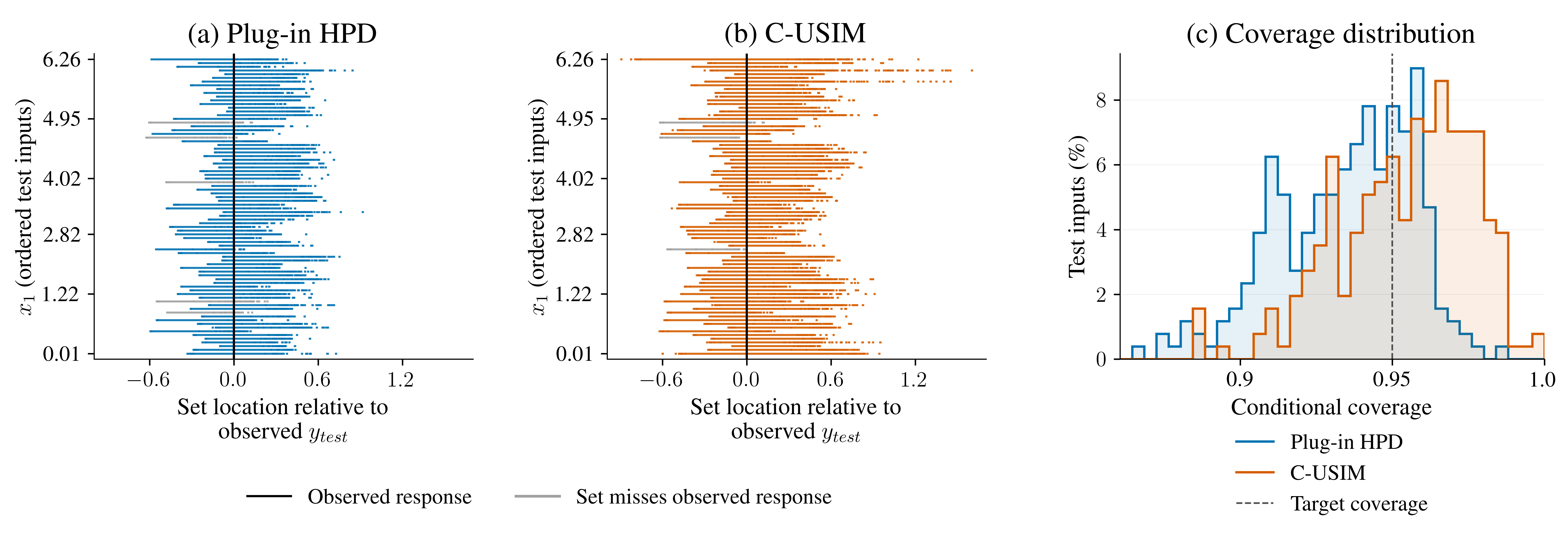}
\par
\includegraphics[width=\linewidth,height=0.21\textheight,keepaspectratio]{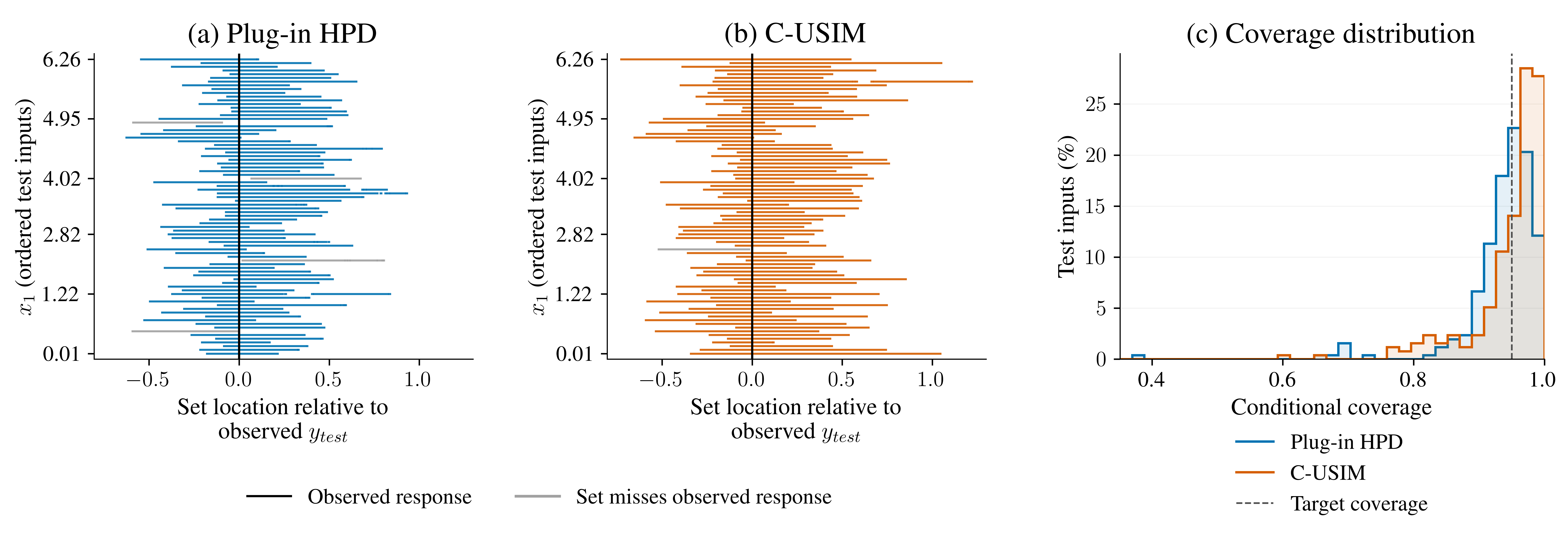}
\caption{\ref{dgp:1d-1}, seed 100. Plotting details: Appendix~\ref{app:synthetic-results}. }
\label{fig:synthetic-notebook-1d-1}
\end{figure}

\begin{figure}[p]
\centering
\includegraphics[width=\linewidth,height=0.21\textheight,keepaspectratio]{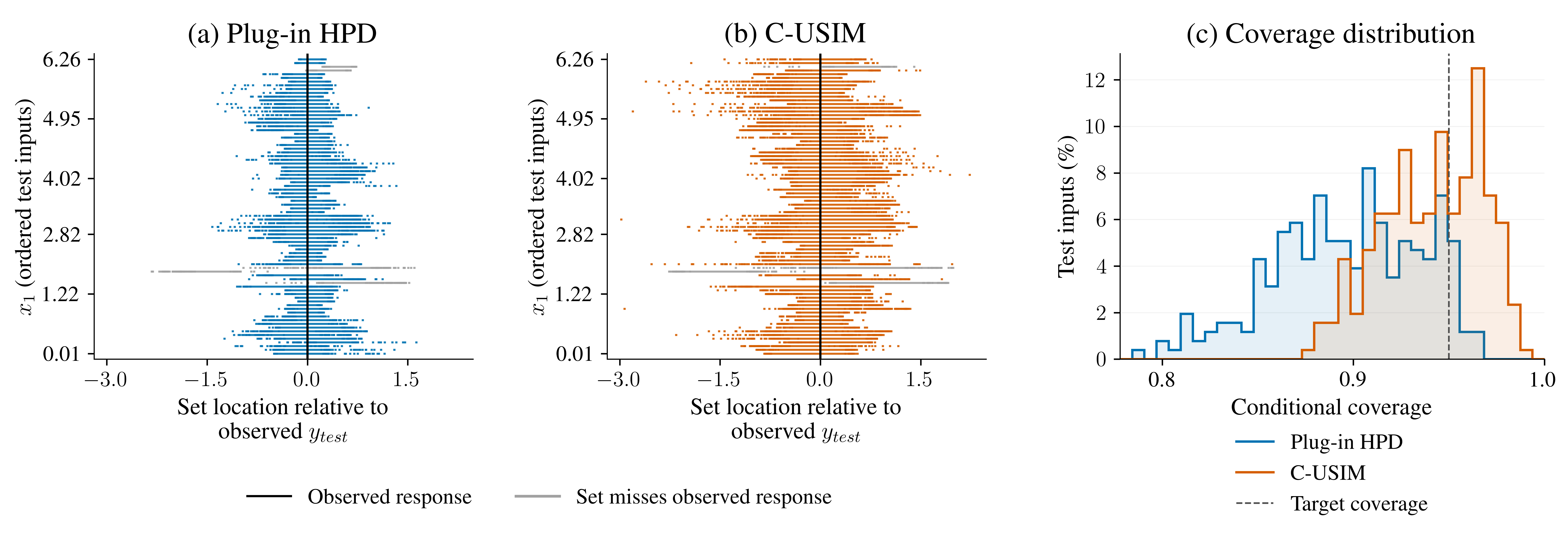}
\par
\includegraphics[width=\linewidth,height=0.21\textheight,keepaspectratio]{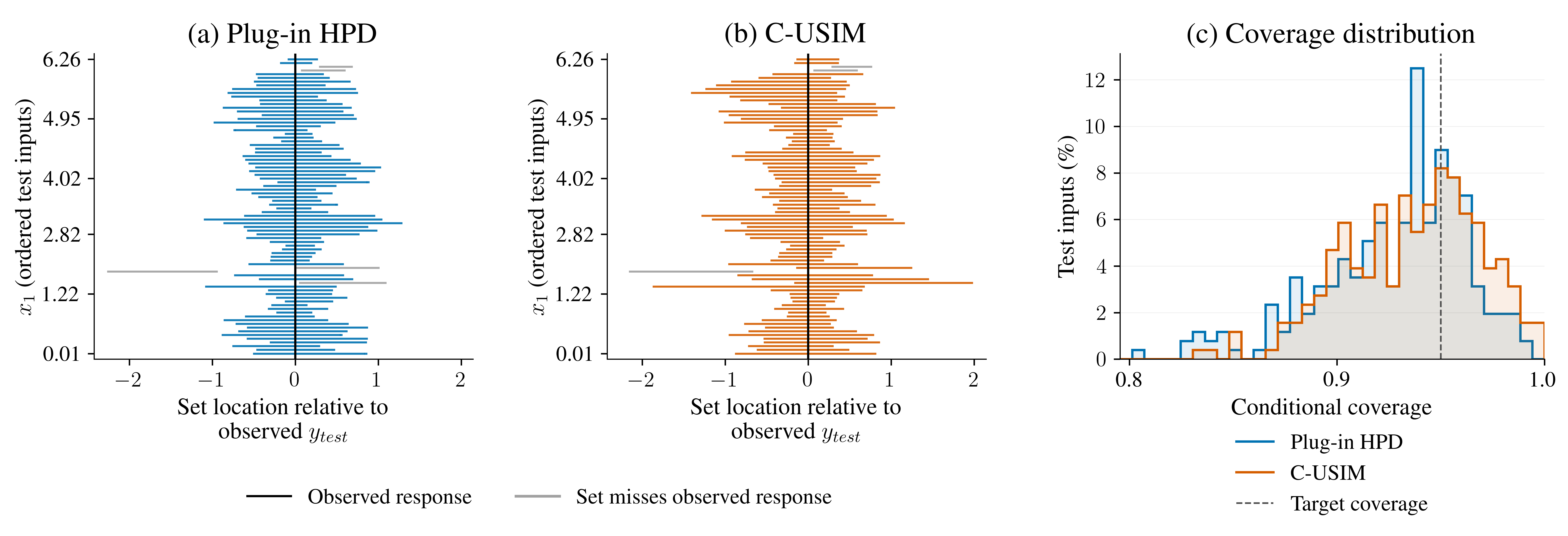}
\caption{\ref{dgp:1d-2}, seed 100. Plotting details: Appendix~\ref{app:synthetic-results}. }
\label{fig:synthetic-notebook-1d-2}
\end{figure}

\clearpage

\begin{figure}[p]
\centering
\includegraphics[width=\linewidth,height=0.21\textheight,keepaspectratio]{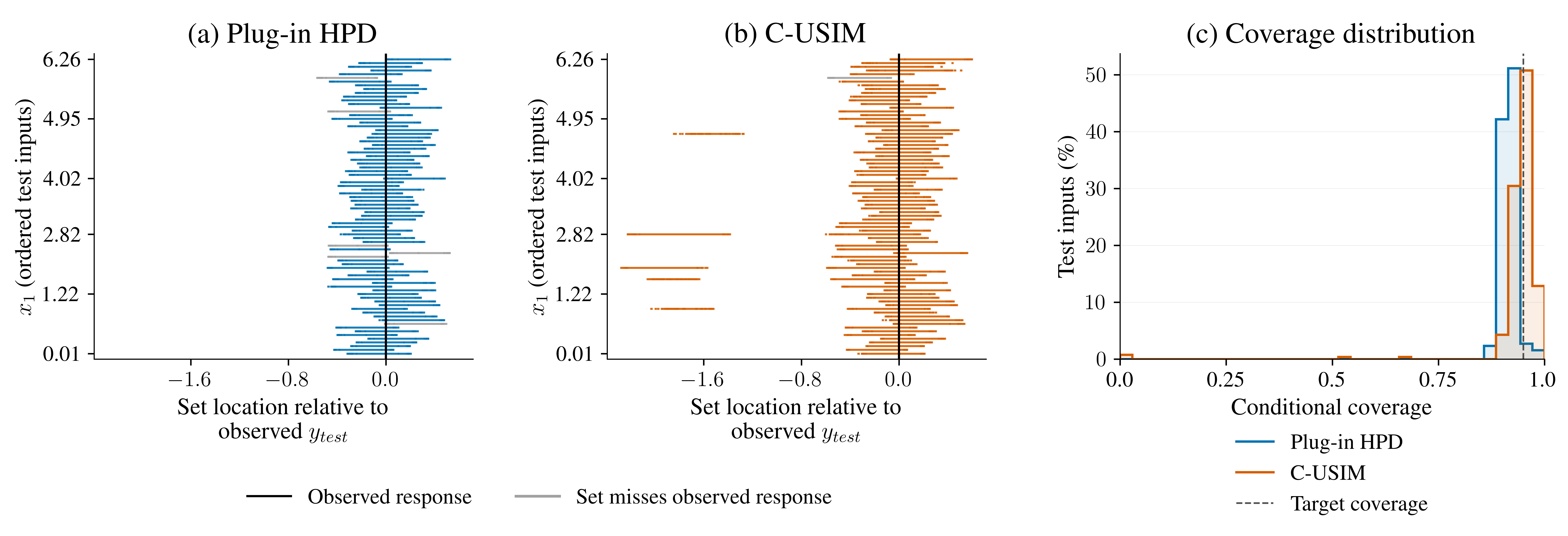}
\par
\includegraphics[width=\linewidth,height=0.21\textheight,keepaspectratio]{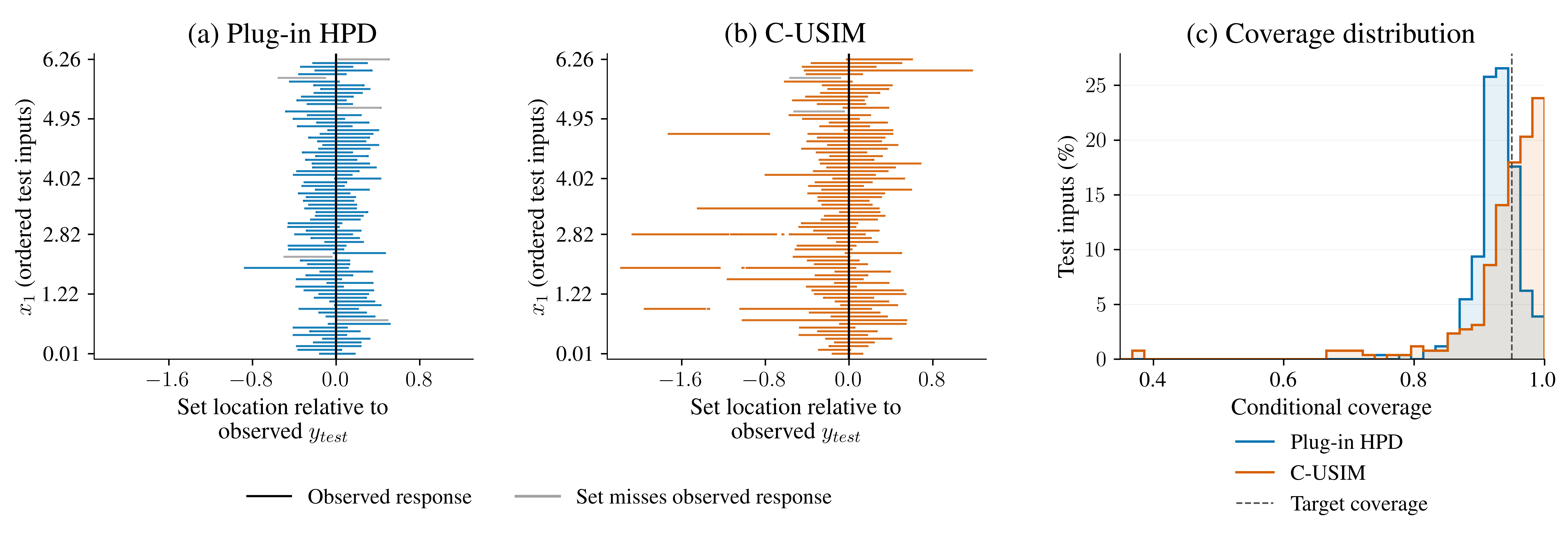}
\caption{\ref{dgp:1d-3}, seed 100. Plotting details: Appendix~\ref{app:synthetic-results}. }
\label{fig:synthetic-notebook-1d-3}
\end{figure}

\begin{figure}[p]
\centering
\includegraphics[width=\linewidth,height=0.21\textheight,keepaspectratio]{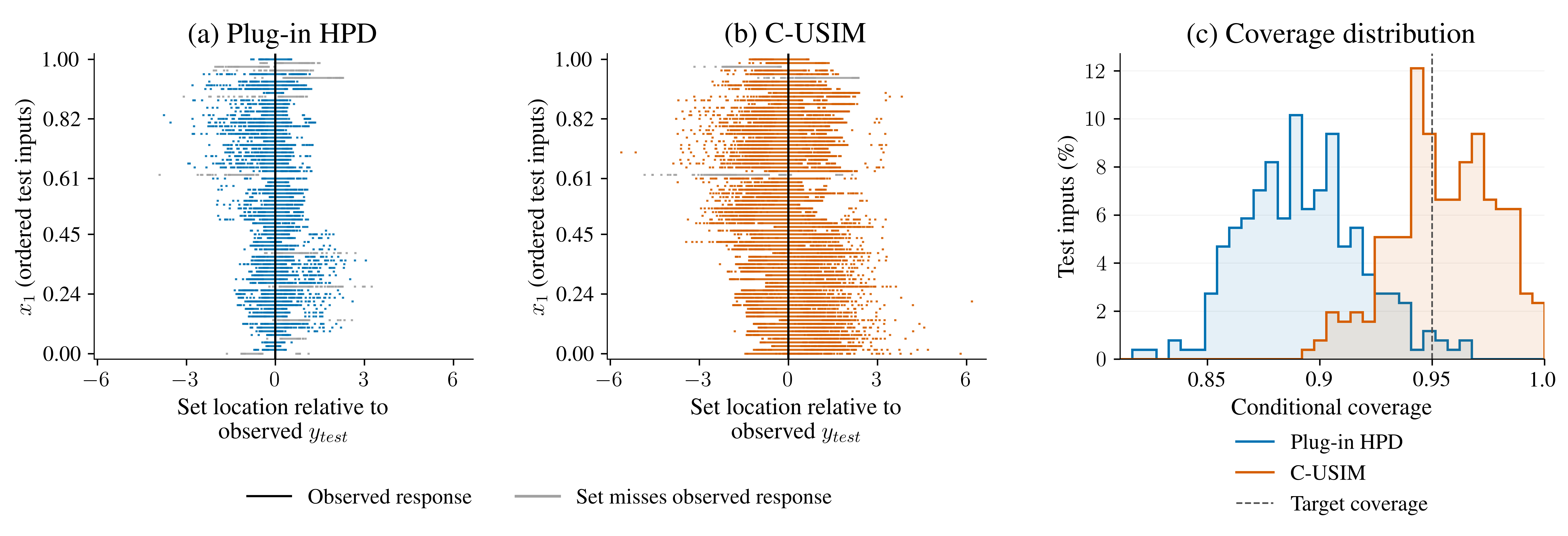}
\par
\includegraphics[width=\linewidth,height=0.21\textheight,keepaspectratio]{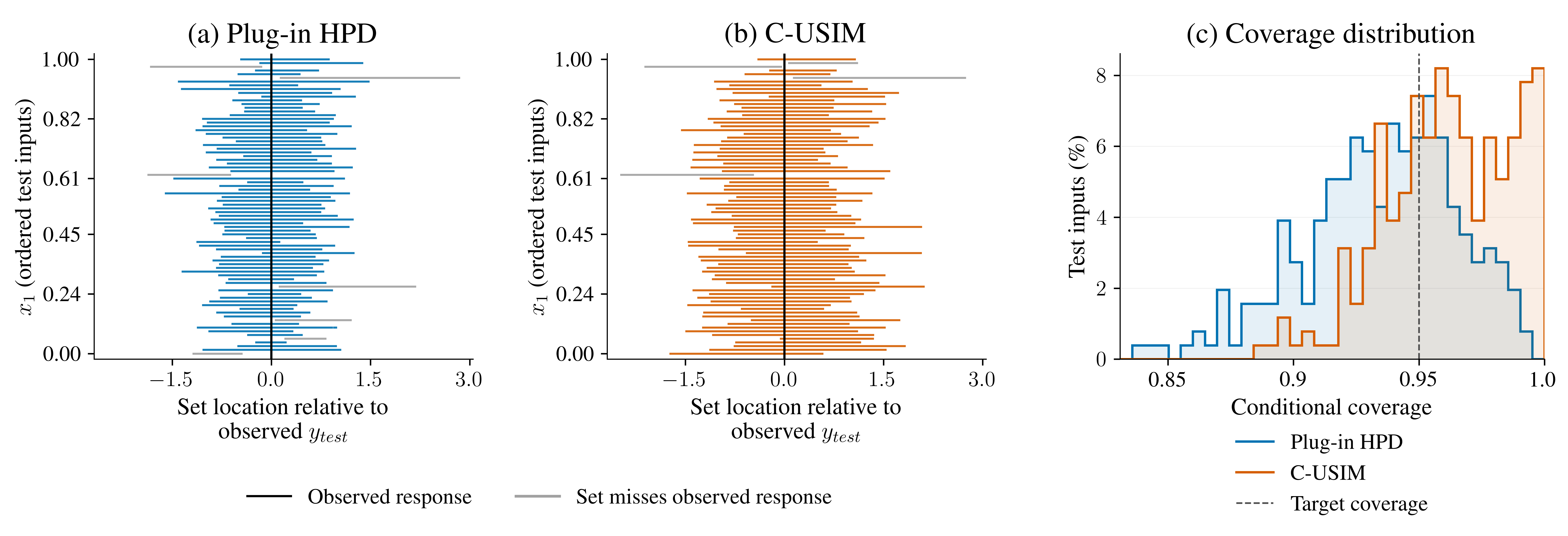}
\caption{\ref{dgp:md-1}, seed 100. Plotting details: Appendix~\ref{app:synthetic-results}. }
\label{fig:synthetic-notebook-md-1}
\end{figure}

\clearpage

\begin{figure}[p]
\centering
\includegraphics[width=\linewidth,height=0.21\textheight,keepaspectratio]{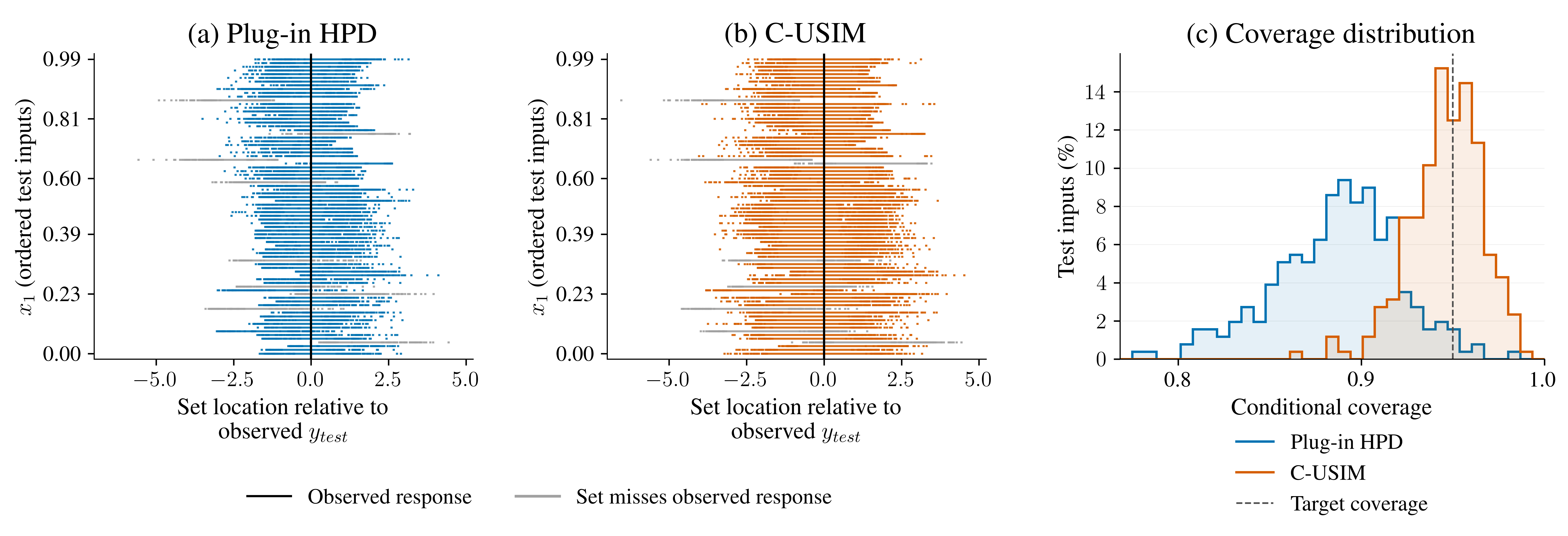}
\par
\includegraphics[width=\linewidth,height=0.21\textheight,keepaspectratio]{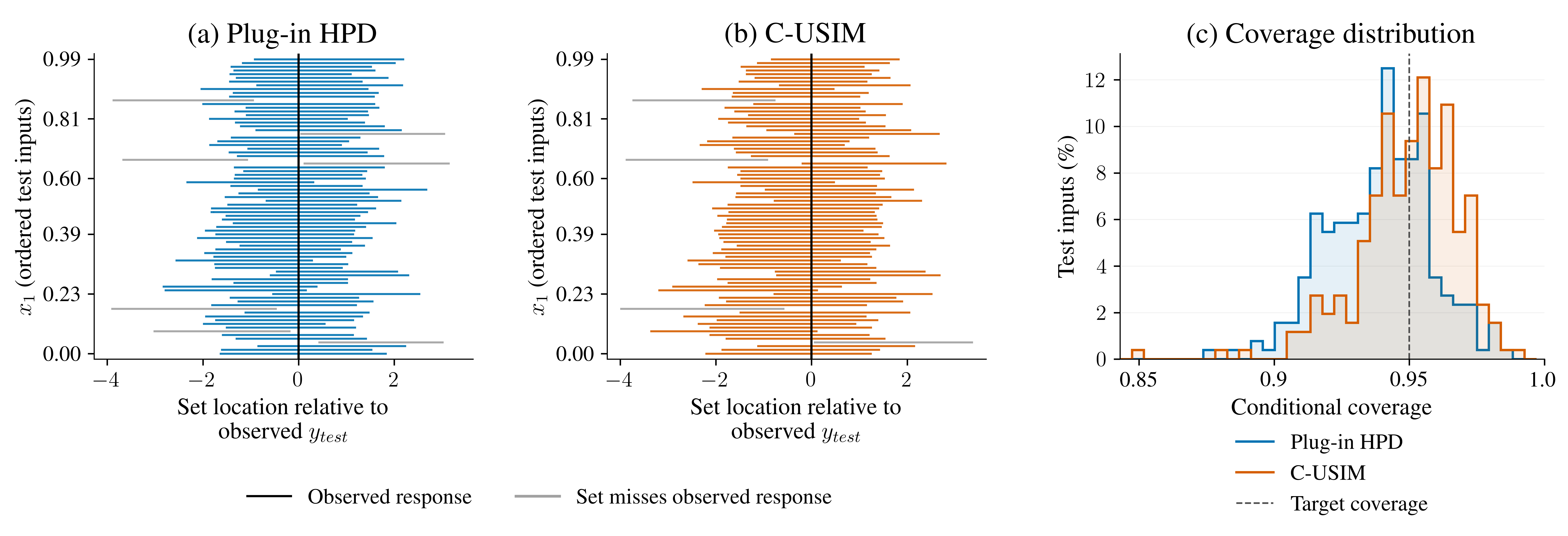}
\caption{\ref{dgp:md-2}, seed 100. Plotting details: Appendix~\ref{app:synthetic-results}. }
\label{fig:synthetic-notebook-md-2}
\end{figure}

\begin{figure}[p]
\centering
\includegraphics[width=\linewidth,height=0.21\textheight,keepaspectratio]{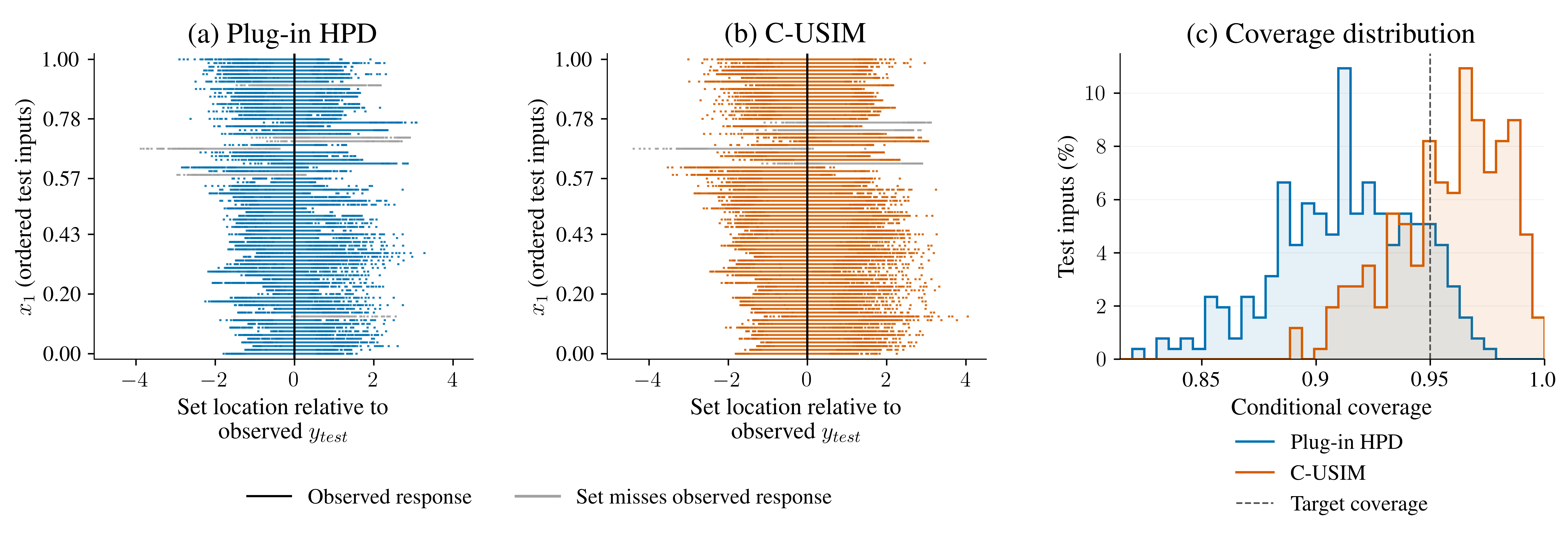}
\par
\includegraphics[width=\linewidth,height=0.21\textheight,keepaspectratio]{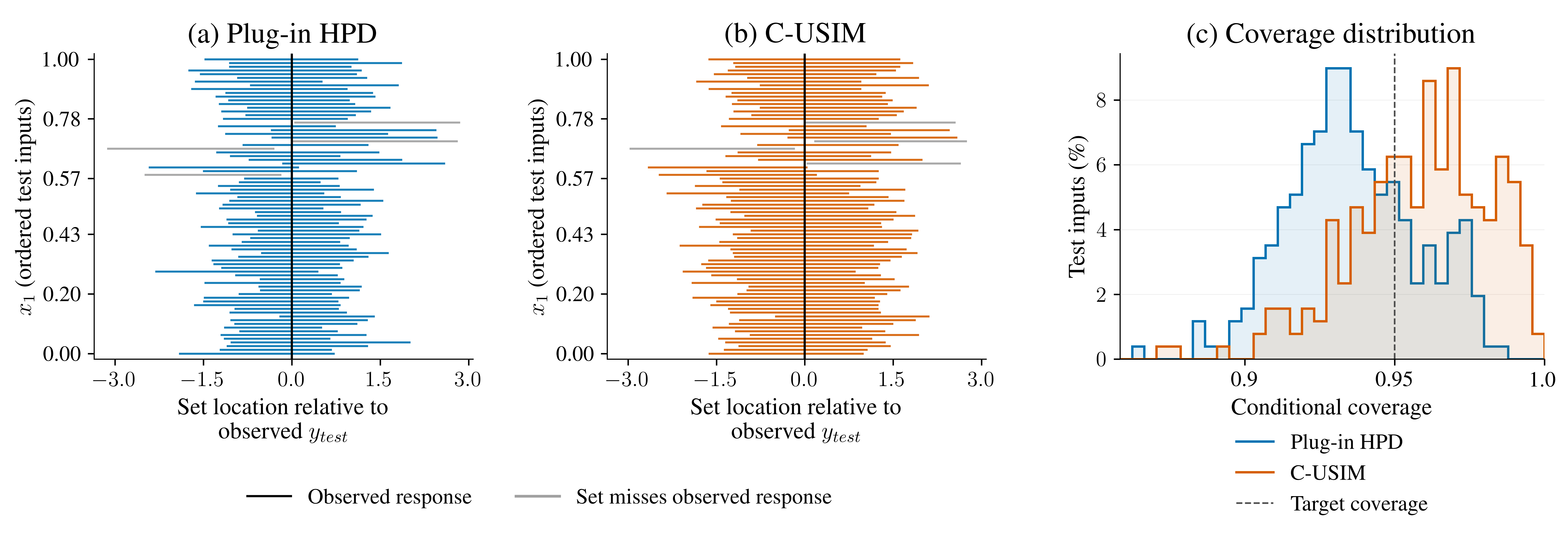}
\caption{\ref{dgp:md-3}, seed 100. Plotting details: Appendix~\ref{app:synthetic-results}. }
\label{fig:synthetic-notebook-md-3}
\end{figure}

\clearpage

\begin{figure}[p]
\centering
\includegraphics[width=\linewidth,height=0.21\textheight,keepaspectratio]{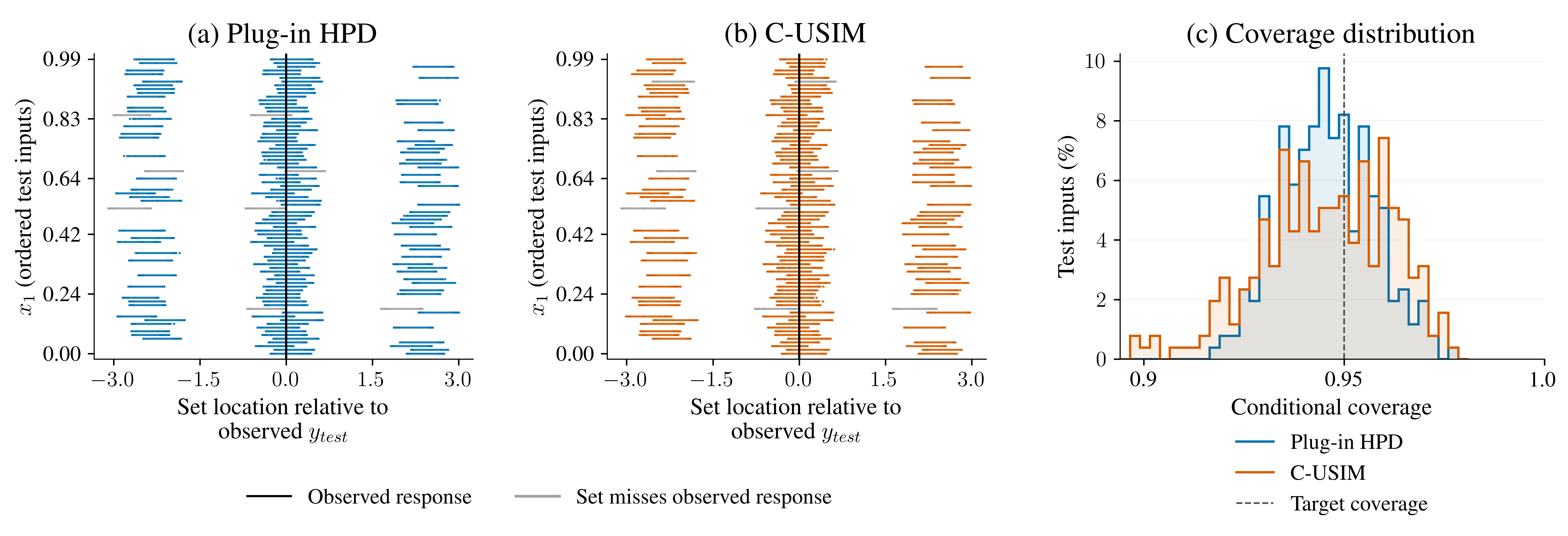}
\par
\includegraphics[width=\linewidth,height=0.21\textheight,keepaspectratio]{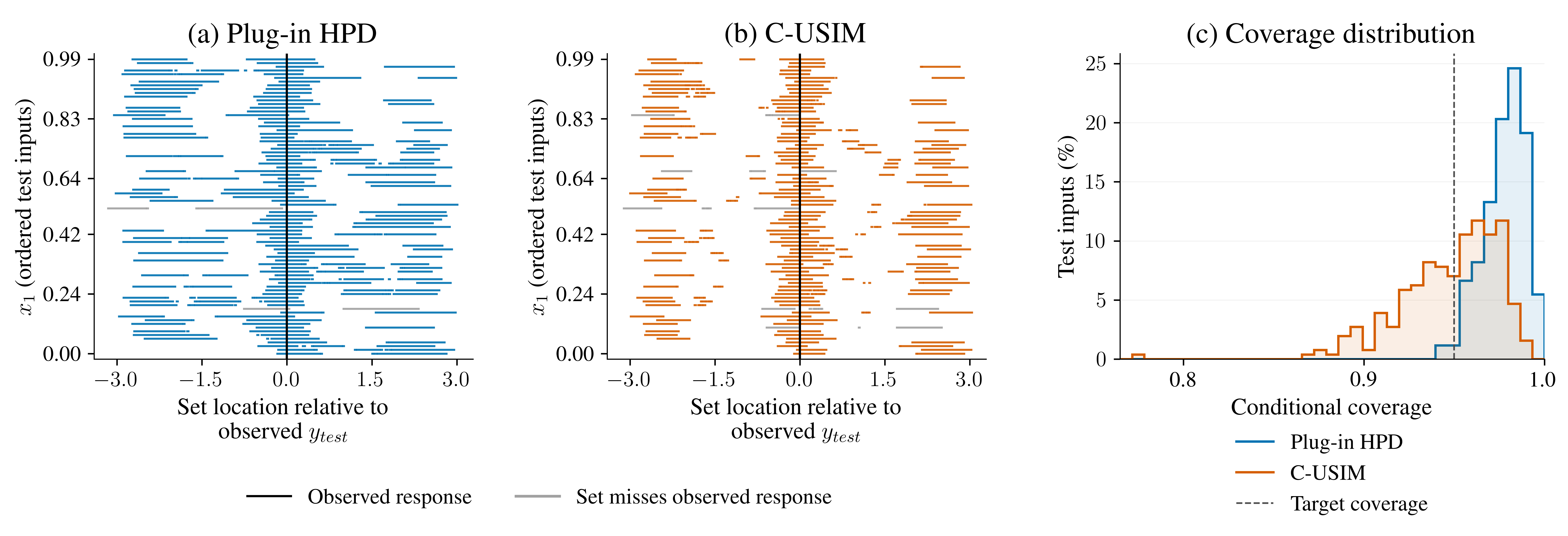}
\caption{\ref{dgp:bm-1}, seed 100. Plotting details: Appendix~\ref{app:synthetic-results}. }
\label{fig:synthetic-notebook-bm-1}
\end{figure}

\begin{figure}[p]
\centering
\includegraphics[width=\linewidth,height=0.21\textheight,keepaspectratio]{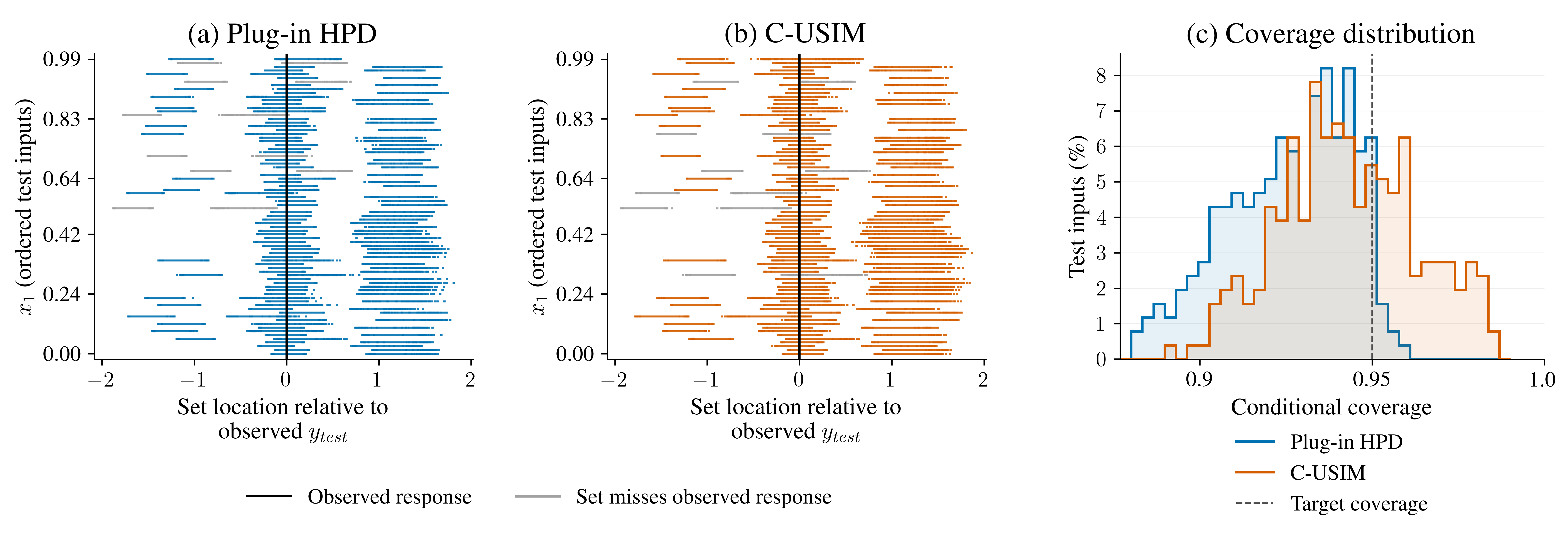}
\par
\includegraphics[width=\linewidth,height=0.21\textheight,keepaspectratio]{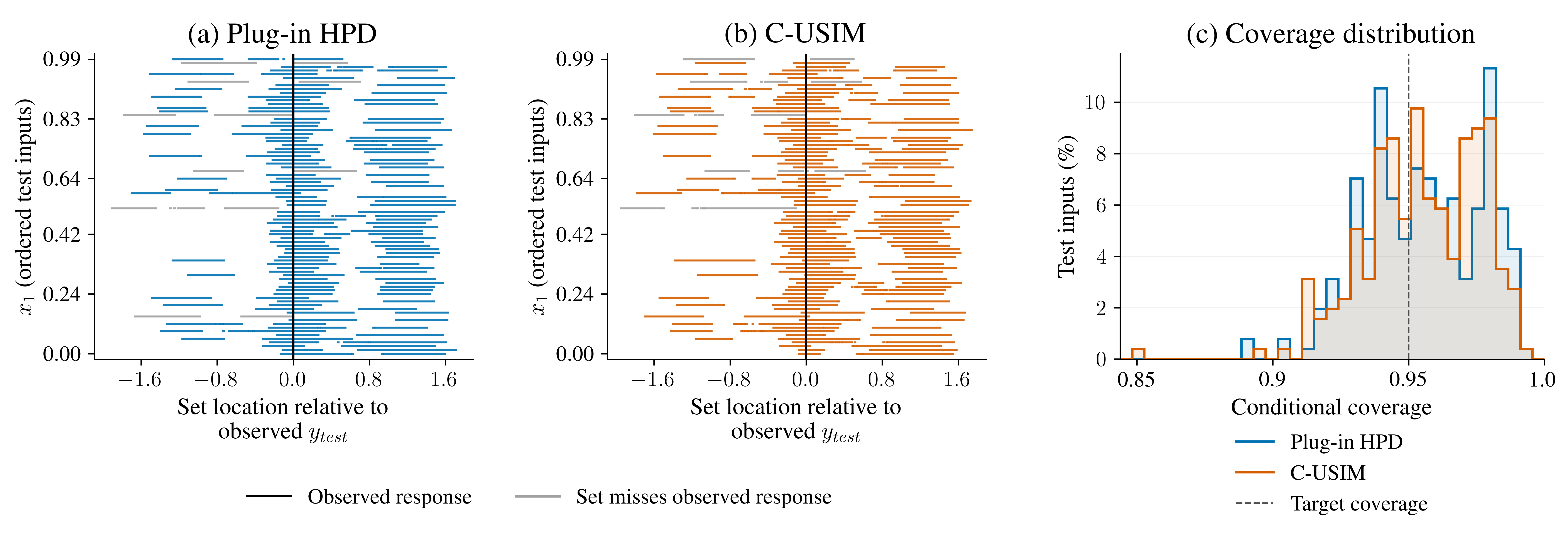}
\caption{\ref{dgp:bm-2}, seed 100. Plotting details: Appendix~\ref{app:synthetic-results}. }
\label{fig:synthetic-notebook-bm-2}
\end{figure}

\clearpage

\begin{figure}[!htbp]
\centering
\includegraphics[width=\linewidth,height=0.21\textheight,keepaspectratio]{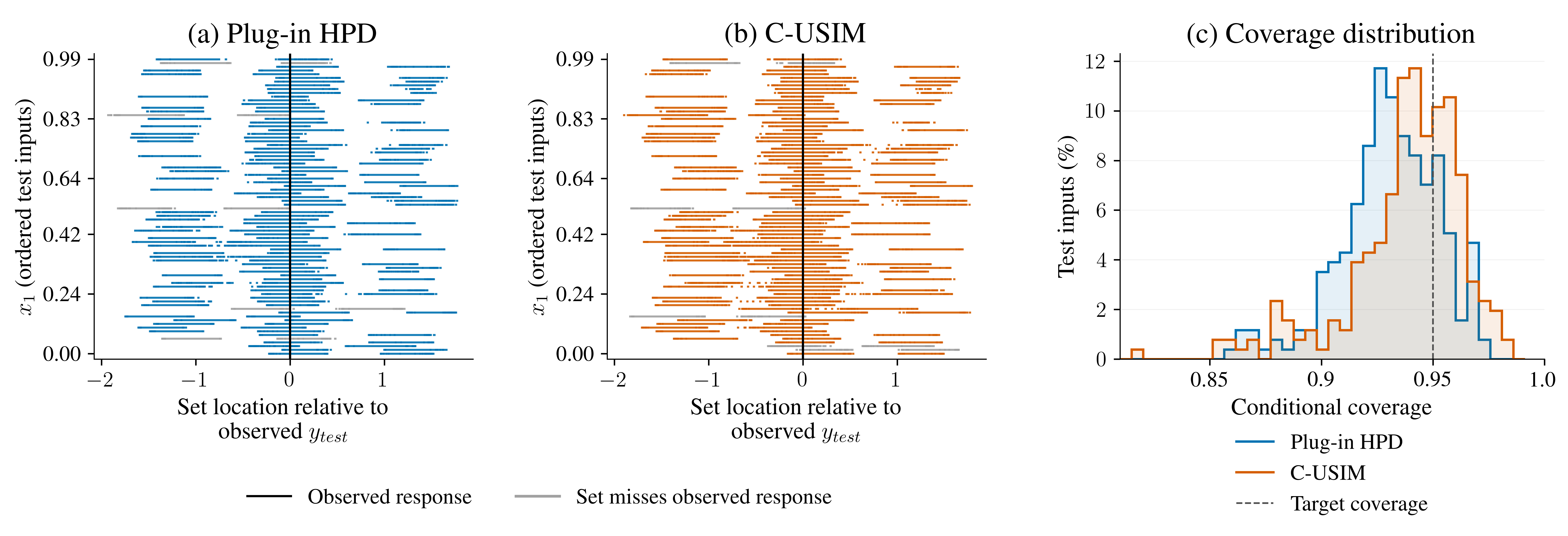}
\par
\includegraphics[width=\linewidth,height=0.21\textheight,keepaspectratio]{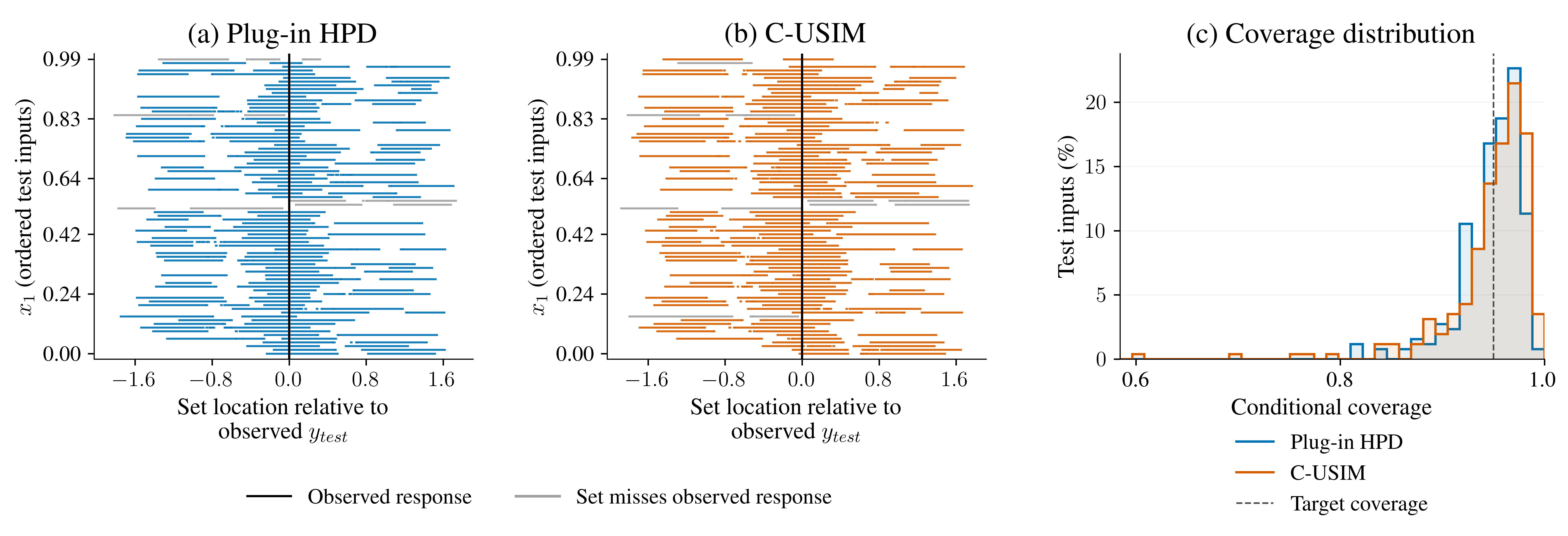}
\caption{\ref{dgp:bm-3}, seed 100. Plotting details: Appendix~\ref{app:synthetic-results}. }
\label{fig:synthetic-notebook-bm-3}
\end{figure}

\begin{table}[!t]
\centering
\caption{Synthetic comparison under a fixed budget of 1,536 labels, averaged over ten seeds per model and case. Each arrow denotes Plug-in HPD(1536) $\rightarrow$ C-USIM(512+1024). Coverage is in percent and CCAD in percentage points (pp). Cases are defined in Appendix~\ref{app:synthetic-dgps}.}

\label{tab:synthetic-calibration}
\small
\begin{tabular}{llccc}
\hline
Model & Example & Coverage (\%) & CCAD (pp) & Mean length \\
\cline{3-5}
& & \multicolumn{3}{c}{Plug-in HPD(1536) $\rightarrow$ C-USIM} \\
\hline
TabPFN & \ref{dgp:1d-1} & $91.92 \rightarrow 95.28$ & $3.572 \rightarrow 2.159$ & $0.571 \rightarrow 0.717$ \\
 & \ref{dgp:1d-2} & $89.77 \rightarrow 94.74$ & $5.384 \rightarrow 2.186$ & $0.754 \rightarrow 1.290$ \\
 & \ref{dgp:1d-3} & $91.63 \rightarrow 94.91$ & $3.513 \rightarrow 1.708$ & $0.483 \rightarrow 0.580$ \\
 & \ref{dgp:md-1} & $89.05 \rightarrow 95.33$ & $5.986 \rightarrow 2.201$ & $0.915 \rightarrow 2.100$ \\
 & \ref{dgp:md-2} & $88.24 \rightarrow 94.81$ & $6.854 \rightarrow 1.642$ & $2.274 \rightarrow 3.313$ \\
 & \ref{dgp:md-3} & $89.64 \rightarrow 95.07$ & $5.533 \rightarrow 2.054$ & $1.706 \rightarrow 2.706$ \\
 & \ref{dgp:bm-1} & $94.51 \rightarrow 94.85$ & $1.227 \rightarrow 1.502$ & $1.231 \rightarrow 1.275$ \\
 & \ref{dgp:bm-2} & $92.43 \rightarrow 95.01$ & $2.642 \rightarrow 1.594$ & $0.969 \rightarrow 1.137$ \\
 & \ref{dgp:bm-3} & $93.18 \rightarrow 94.85$ & $2.057 \rightarrow 1.588$ & $1.122 \rightarrow 1.245$ \\
\hline
TabICL & \ref{dgp:1d-1} & $94.19 \rightarrow 95.31$ & $2.961 \rightarrow 3.385$ & $0.617 \rightarrow 0.747$ \\
 & \ref{dgp:1d-2} & $93.48 \rightarrow 94.93$ & $2.741 \rightarrow 2.632$ & $1.044 \rightarrow 1.287$ \\
 & \ref{dgp:1d-3} & $92.93 \rightarrow 95.13$ & $3.186 \rightarrow 3.168$ & $0.535 \rightarrow 0.725$ \\
 & \ref{dgp:md-1} & $93.08 \rightarrow 95.43$ & $2.930 \rightarrow 2.644$ & $1.402 \rightarrow 1.901$ \\
 & \ref{dgp:md-2} & $93.42 \rightarrow 95.01$ & $2.220 \rightarrow 1.572$ & $2.759 \rightarrow 3.109$ \\
 & \ref{dgp:md-3} & $93.11 \rightarrow 95.17$ & $2.679 \rightarrow 2.266$ & $2.087 \rightarrow 2.609$ \\
 & \ref{dgp:bm-1} & $97.86 \rightarrow 94.91$ & $2.902 \rightarrow 2.046$ & $2.097 \rightarrow 1.476$ \\
 & \ref{dgp:bm-2} & $94.88 \rightarrow 95.18$ & $1.953 \rightarrow 2.543$ & $1.238 \rightarrow 1.340$ \\
 & \ref{dgp:bm-3} & $95.80 \rightarrow 94.92$ & $1.980 \rightarrow 2.467$ & $1.457 \rightarrow 1.462$ \\
\hline
\end{tabular}
\end{table}

\clearpage

\begin{figure}[htbp]
\centering
\includegraphics[width=\linewidth]{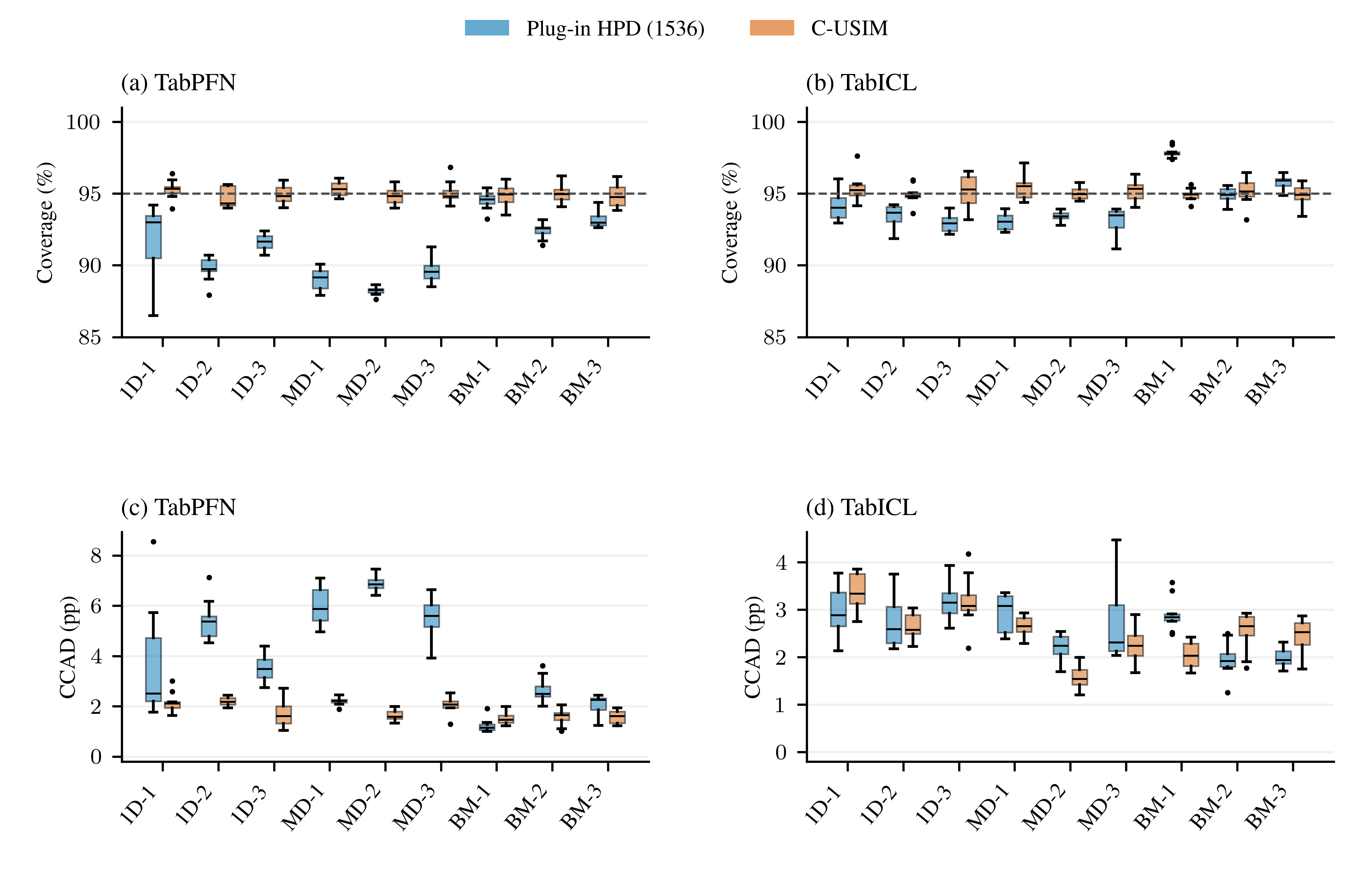}
\caption{Synthetic coverage (a,b) and CCAD (c,d) for all nine mechanisms across ten seeds for Plug-in HPD(1536) and C-USIM(512+1024); methods and label budgets follow Table~\ref{tab:synthetic-calibration}. Boxes show Q1--Q3 and medians, with 1.5-IQR whiskers and all outliers. Dashed lines mark 95\% coverage. CCAD is in percentage points on model-specific scales. These seed distributions are not confidence intervals.}
\label{fig:experiment-calibration}
\end{figure}

\subsubsection{Journal SJR}
\label{app:sjr-results}
Table~\ref{tab:tabpfn-real-coverage} and Figure~\ref{fig:sjr-calibration} report the fixed-budget comparison between Plug-in HPD(1536) and C-USIM for Journal SJR.

\begin{figure}[htb]
\centering
\includegraphics[width=\linewidth,height=0.47\textheight,keepaspectratio]{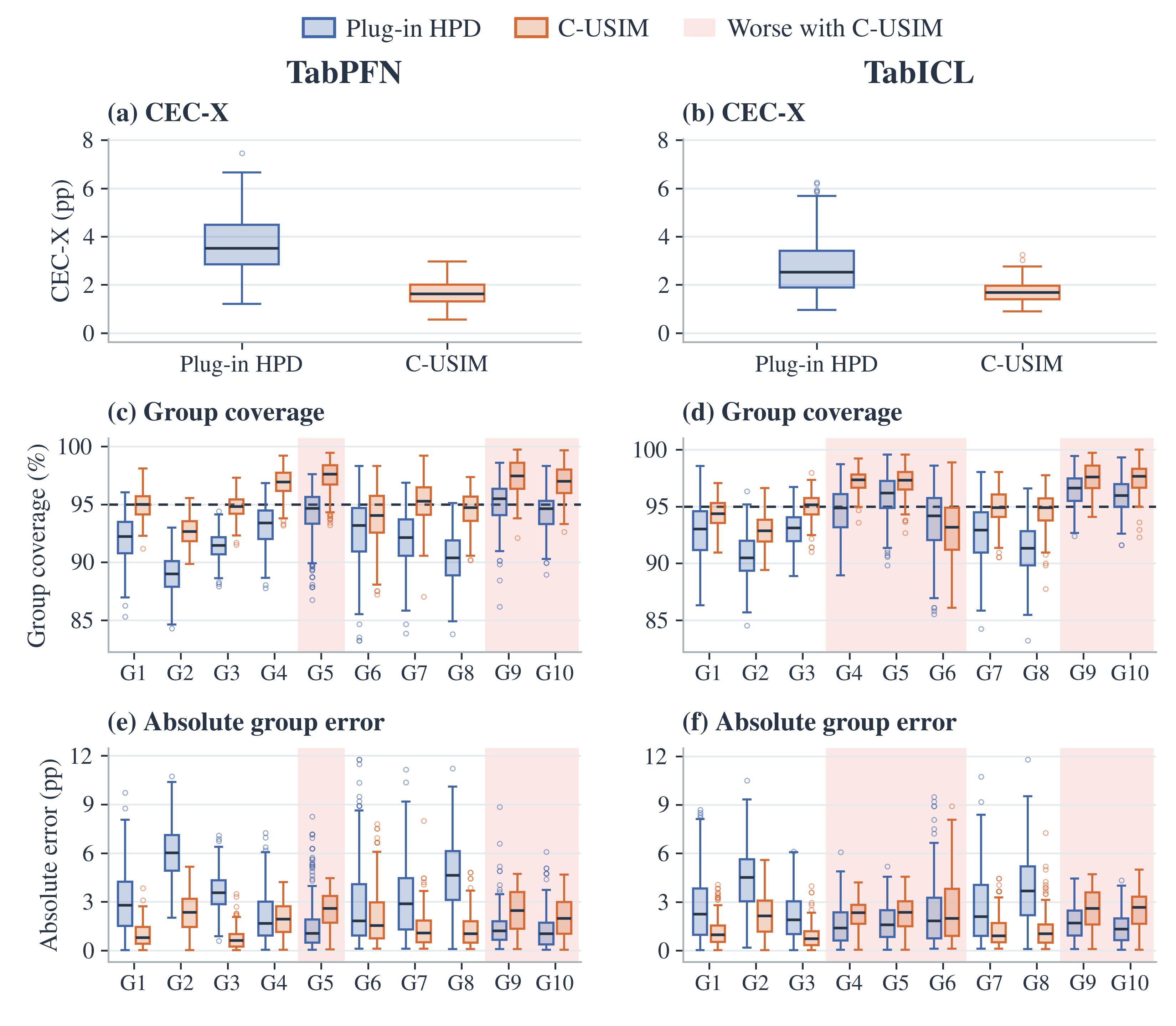}
\caption{Journal SJR across 200 seeds for Plug-in HPD(1536) and C-USIM(512+1024); methods and label budgets follow Table~\ref{tab:tabpfn-real-coverage}. (a,b) CEC-X, weighted by group size; (c,d) group coverage; (e,f) absolute group error. Errors are computed within each seed. Boxes show Q1--Q3 and medians, with 1.5-IQR whiskers and all outliers, not confidence intervals. Red bands mark higher mean absolute group error with C-USIM; dashed lines mark 95\% coverage. Errors are in percentage points. CEC-X shares a scale across models; group plots use model-specific scales with visual padding below zero.}
\label{fig:sjr-calibration}
\end{figure}

\paragraph{Variation across groups and groupings under a fixed label budget.} In the representative grouping, mean absolute error decreases in 7 of the ten groups for TabPFN and 5 for TabICL; 62.20\% and 55.35\% of seed--group pairs, respectively, move closer to 95\% coverage. Across the 30 groupings, C-USIM reduces absolute coverage error in 47.56--70.60\% of seed--group pairs for TabPFN and 40.64--61.00\% for TabICL. All nonempty groups remain in the reported averages. These aggregate gains coexist with higher errors in some covariate groups and longer mean prediction sets.

\begin{table}[htbp]
\centering
\caption{Journal SJR under a fixed budget of 1,536 labels at 95\% target coverage, averaged over 200 seeds. Each arrow denotes Plug-in HPD(1536) $\rightarrow$ C-USIM(512+1024). Absolute marginal and group gaps and CEC-X are computed within each seed before averaging, and are reported in percentage points (pp).}
\label{tab:tabpfn-real-coverage}
\label{tab:tabicl-real-cecx}
\small
\begin{tabular}{lcc}
\hline
Measure & TabPFN & TabICL \\
\cline{2-3}
& \multicolumn{2}{c}{Plug-in HPD(1536) $\rightarrow$ C-USIM} \\
\hline
Marginal coverage (\%) & $91.50 \rightarrow 94.83$ & $92.93 \rightarrow 94.87$ \\
Mean marginal gap (pp) & $3.499 \rightarrow 0.553$ & $2.241 \rightarrow 0.611$ \\
Mean group gap (pp) & $2.979 \rightarrow 1.760$ & $2.458 \rightarrow 1.856$ \\
CEC-X (pp) & $3.698 \rightarrow 1.662$ & $2.782 \rightarrow 1.707$ \\
Mean set length & $1.149 \rightarrow 1.453$ & $1.255 \rightarrow 1.540$ \\
\hline
\end{tabular}
\end{table}

\subsubsection{Additional real-world datasets}
\label{app:additional-realdata}

\paragraph{Datasets and experimental settings.} We add JP Anime from CARTE \citep{kim2024carte} and Allstate Claims Severity from OpenML (dataset 42571, version 1).\footnote{Allstate source: \url{https://www.openml.org/d/42571}.} JP Anime predicts the supplied natural logarithm of the anime score from ten features: genres, type, episodes, producers, studios, source, duration, content rating, start date, and end date. Eight features are categorical; episodes and duration are numeric. The prepared cohort retains the first row for each distinct feature profile. We use the source response without applying another logarithm, and report prediction-set lengths on this log-score scale. Allstate predicts the supplied claim loss from 116 categorical and 14 continuous anonymized features. Claim identifiers are excluded from the predictors, and repeated feature profiles with distinct claim identifiers remain separate observations. No response transformation is applied to Allstate.

\begin{figure}[htb]
\centering
\includegraphics[width=\linewidth]{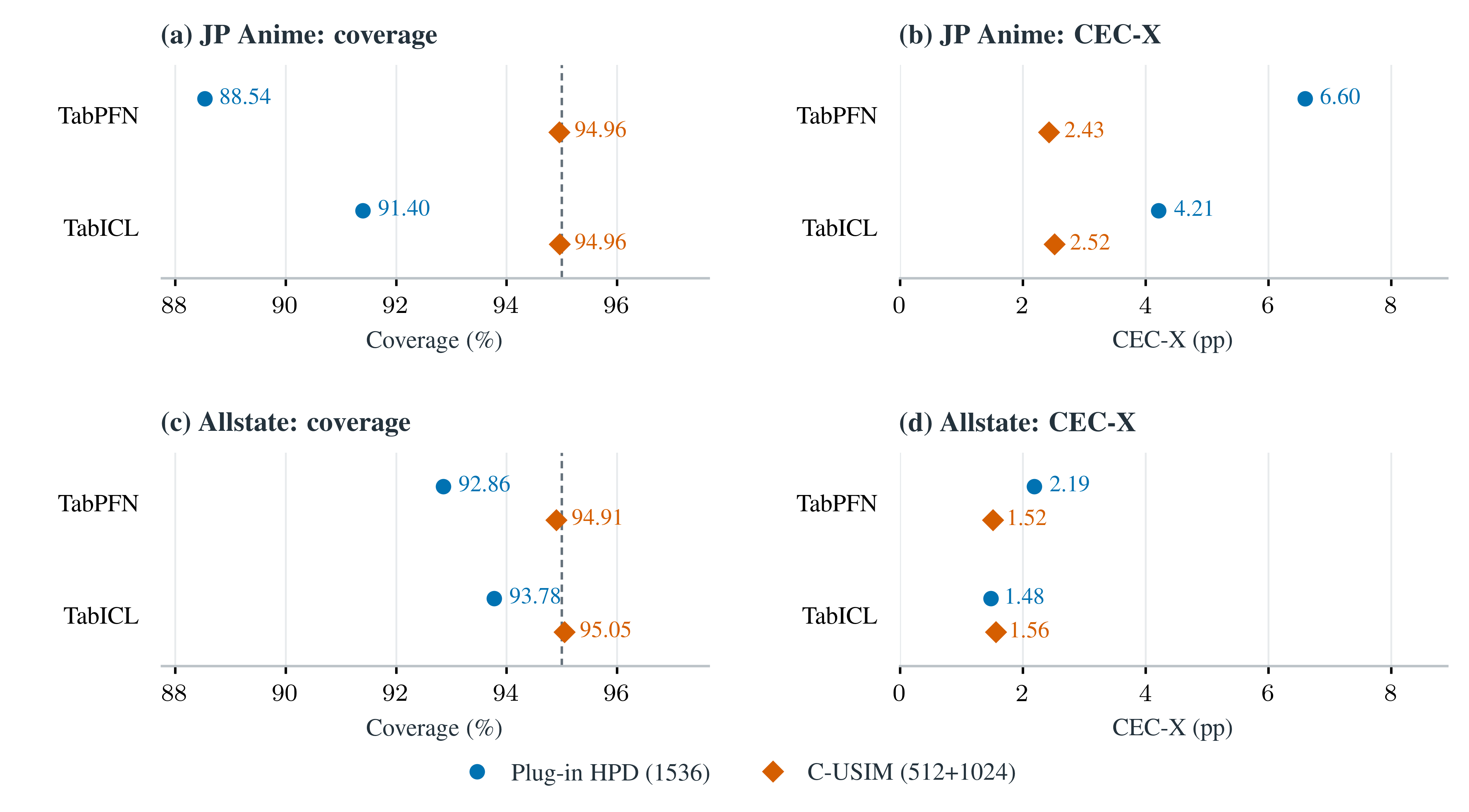}
\caption{Additional real-world examples under a total budget of 1,536 labels. Points are means over 50 seeds, not confidence intervals. (a,b) JP Anime; (c,d) Allstate Claims Severity. Dashed lines mark 95\% coverage. CEC-X uses the representative covariate grouping and is computed within each seed before averaging. Settings and detailed results are in Table~\ref{tab:additional-realdata-settings} and Appendix~\ref{app:additional-realdata}.}
\label{fig:additional-realdata}
\end{figure}

Table~\ref{tab:additional-realdata-settings} gives the additional sample counts. Both datasets use 50 run seeds, 12100--12149, fixed before the new evaluation, with the same model and label-budget settings as Journal SJR. For each fitted context, plug-in HPD processes the full test set in one prediction call, and C-USIM processes all 1,024 calibration covariates together with that test set. Calibration scores are computed anew for each fitted context and joint query table. For model inference, missing numeric values are imputed using context medians and missing categorical values receive a dedicated token. Groups use the same 30 combinations of group counts and clustering seeds as Journal SJR. Numeric variables are median-imputed and standardized, while categorical variables are mode-imputed and one-hot encoded; all preprocessing and cluster centers are fitted on the 512 validation covariates only. The representative grouping remains $K=10$ with clustering seed 1717.

\begin{table}[htbp]
\centering
\caption{Additional real-world datasets.}
\label{tab:additional-realdata-settings}
\small
\begin{tabular}{lrrrrl}
\hline
Dataset & Rows & Features & Categorical & Test rows & Response scale \\
\hline
JP Anime & 15,351 & 10 & 8 & 5,372 & $\ln(\mathrm{Score})$ \\
Allstate & 188,318 & 130 & 116 & 9,731 & Claim loss \\
\hline
\end{tabular}
\end{table}

Figure~\ref{fig:additional-realdata} summarizes the results, and Table~\ref{tab:additional-realdata-budget} reports the comparison between Plug-in HPD(1536) and C-USIM(512+1024) over 50 seeds per model. For both models on both datasets, C-USIM reduces mean absolute marginal coverage error to approximately 0.54--0.60 percentage points. On JP Anime, CEC-X decreases from 6.598 to 2.429 percentage points for TabPFN and from 4.211 to 2.521 for TabICL. On Allstate, it decreases from 2.186 to 1.517 for TabPFN. Mean prediction-set length increases in all four comparisons; lengths are measured on the dataset-specific response scales in Table~\ref{tab:additional-realdata-settings}.

\paragraph{Variation across groups and groupings.} In the representative grouping, mean absolute group error decreases in 7 of ten groups for TabPFN and 5 for TabICL on JP Anime, and in 8 and 5 groups, respectively, on Allstate. Across all 30 groupings, CEC-X decreases for both models on JP Anime and for TabPFN on Allstate. The Allstate TabICL comparison is more sensitive to grouping: CEC-X decreases in 17 of the 30 groupings. In the representative grouping, its CEC-X changes from 1.481 to 1.561 percentage points, while mean group error changes from 1.540 to 1.523. The group plots in Figures~\ref{fig:jp-anime-plugin512} and~\ref{fig:allstate-plugin512} retain all ten groups and place both plug-in baselines alongside C-USIM.

\begin{table}[htbp]
\centering
\caption{Additional real-world datasets under a fixed total budget of 1,536 labels, averaged over 50 seeds. Each arrow denotes Plug-in HPD(1536) $\rightarrow$ C-USIM(512+1024). Marginal and group errors are computed within each seed before averaging; group measures use the representative grouping. Length uses the first 256 fixed test inputs.}
\label{tab:additional-realdata-budget}
\small
\begin{tabular}{llcc}
\hline
Dataset & Measure & TabPFN & TabICL \\
\cline{3-4}
& & \multicolumn{2}{c}{Plug-in HPD (1536) $\rightarrow$ C-USIM} \\
\hline
JP Anime & Marginal coverage (\%) & $88.54 \rightarrow 94.96$ & $91.40 \rightarrow 94.96$ \\
 & Mean marginal gap (pp) & $6.461 \rightarrow 0.590$ & $3.598 \rightarrow 0.540$ \\
 & Mean group gap (pp) & $6.224 \rightarrow 2.770$ & $4.235 \rightarrow 2.837$ \\
 & CEC-X (pp) & $6.598 \rightarrow 2.429$ & $4.211 \rightarrow 2.521$ \\
 & Mean set length & $0.447 \rightarrow 0.565$ & $0.494 \rightarrow 0.539$ \\
\hline
Allstate & Marginal coverage (\%) & $92.86 \rightarrow 94.91$ & $93.78 \rightarrow 95.05$ \\
 & Mean marginal gap (pp) & $2.142 \rightarrow 0.601$ & $1.226 \rightarrow 0.540$ \\
 & Mean group gap (pp) & $2.240 \rightarrow 1.515$ & $1.540 \rightarrow 1.523$ \\
 & CEC-X (pp) & $2.186 \rightarrow 1.517$ & $1.481 \rightarrow 1.561$ \\
 & Mean set length & $5862.19 \rightarrow 6831.03$ & $6195.59 \rightarrow 6883.54$ \\
\hline
\end{tabular}

\end{table}

\subsection{Split-ratio sensitivity}
\label{app:split-ratio}

The sweep follows the synthetic protocol in Appendix~\ref{app:settings}, with allocations and seeds listed in Table~\ref{tab:experiment-settings}. Within each mechanism and seed, all allocations and both models share a labeled pool and independent test inputs: the first $n_{\mathrm{train}}$ pool observations form the context and the remainder form the calibration set.

\begin{table}[htbp]
\centering
\caption{Split-ratio sensitivity across all nine mechanisms, with 1,536 total labels. Entries are means over the same 50 seeds. Gap is the per-seed absolute deviation of marginal coverage from 95\%; gap and CCAD use percentage points (pp).}
\label{tab:split-ratio-all-cases}
\small
\begin{tabular}{lcccc}
\hline
& \multicolumn{2}{c}{TabPFN} & \multicolumn{2}{c}{TabICL} \\
Case & Gap (pp) & CCAD (pp) & Gap (pp) & CCAD (pp) \\
\cline{2-5}
& \multicolumn{4}{c}{1229:307 ($\simeq8{:}2$) $\rightarrow$ 512:1024 ($1{:}2$)} \\
\hline
\ref{dgp:1d-1} & $1.033 \rightarrow 0.554$ & $2.110 \rightarrow 2.152$ & $0.911 \rightarrow 0.440$ & $2.830 \rightarrow 3.381$ \\
\ref{dgp:1d-2} & $0.964 \rightarrow 0.497$ & $1.998 \rightarrow 2.093$ & $1.013 \rightarrow 0.561$ & $2.370 \rightarrow 2.541$ \\
\ref{dgp:1d-3} & $0.857 \rightarrow 0.546$ & $1.231 \rightarrow 1.379$ & $1.098 \rightarrow 0.552$ & $2.486 \rightarrow 3.104$ \\
\ref{dgp:md-1} & $1.026 \rightarrow 0.583$ & $2.271 \rightarrow 2.218$ & $1.119 \rightarrow 0.650$ & $2.569 \rightarrow 2.634$ \\
\ref{dgp:md-2} & $0.924 \rightarrow 0.500$ & $1.626 \rightarrow 1.543$ & $0.969 \rightarrow 0.640$ & $1.676 \rightarrow 1.606$ \\
\ref{dgp:md-3} & $1.223 \rightarrow 0.624$ & $1.784 \rightarrow 1.931$ & $1.223 \rightarrow 0.562$ & $2.045 \rightarrow 2.155$ \\
\ref{dgp:bm-1} & $1.083 \rightarrow 0.474$ & $1.365 \rightarrow 1.146$ & $0.903 \rightarrow 0.516$ & $1.720 \rightarrow 1.962$ \\
\ref{dgp:bm-2} & $0.919 \rightarrow 0.537$ & $1.395 \rightarrow 1.434$ & $0.840 \rightarrow 0.626$ & $1.915 \rightarrow 2.351$ \\
\ref{dgp:bm-3} & $0.851 \rightarrow 0.592$ & $1.335 \rightarrow 1.413$ & $1.047 \rightarrow 0.525$ & $2.208 \rightarrow 2.425$ \\
\hline
\end{tabular}
\end{table}

Coverage and CCAD follow Appendix~\ref{app:evaluation-measures}, with conditional coverage computed using the known CDF rather than response sampling. Mean-function RMSE is $\{256^{-1}\sum_i[\hat m(X_i)-m(X_i)]^2\}^{1/2}$, where $\hat m$ is the mean of the reconstructed predictive density and $m$ is the true conditional mean. Metrics are computed within each seed before averaging. Table~\ref{tab:split-ratio-all-cases} reports all nine mechanisms, including cases with worse CCAD.

\subsection{Supplementary comparison with identical predictive densities}
\label{app:plugin512}

The main experiments compare Plug-in HPD(1536) with C-USIM(512+1024) under the same total label budget. Here, we additionally compare Plug-in HPD(512) with C-USIM to examine the effect of calibrating the cutoff while holding the reconstructed predictive densities fixed. This supplementary comparison uses different numbers of response labels: Plug-in HPD(512) uses only the 512 context labels, whereas C-USIM additionally uses 1,024 held-out calibration responses.

\paragraph{Settings and shared predictions.} Within each run, Plug-in HPD(512) and C-USIM share the context observations, model randomness, query covariates, and reconstructed densities. The plug-in regions use the boundary density level needed to attain at least 95\% model mass; C-USIM uses the 974th of the 1,024 calibration scores. Both use the finite-density reconstruction in Appendix~\ref{app:density-estimation} and the evaluation measures in Appendix~\ref{app:evaluation-measures}. The data splits and seeds follow Table~\ref{tab:experiment-settings}, with the additional real-world datasets specified in Table~\ref{tab:additional-realdata-settings}. Synthetic comparisons use the same 1,000 conditional response draws at each test input across all three configurations.

For Journal SJR, both 512-context configurations use the same joint table of 1,024 calibration and 9,731 test covariates in one prediction call. Plug-in HPD(512) uses the test densities from this shared output without using the calibration responses. Coverage uses all 9,731 test rows, and mean prediction-set length uses the first 256 inputs in the fixed test order for all three configurations. The inference and density reconstruction steps are shared; calibration changes the cutoff used to extract the set without fitting an additional model.

\subsubsection{Synthetic}

\begin{table}[!t]
\centering
\caption{Synthetic comparison on all nine mechanisms, using identical predictive densities from a 512-observation context. Each arrow denotes Plug-in HPD(512) $\rightarrow$ C-USIM(512+1024), averaged over ten seeds. C-USIM additionally uses 1,024 calibration responses.}
\label{tab:synthetic-plugin512}
\small
\begin{tabular}{llccc}
\hline
Model & Example & Coverage (\%) & CCAD (pp) & Mean length \\
\cline{3-5}
& & \multicolumn{3}{c}{Plug-in HPD(512) $\rightarrow$ C-USIM} \\
\hline
TabPFN & \ref{dgp:1d-1} & $92.95 \rightarrow 95.28$ & $2.947 \rightarrow 2.159$ & $0.618 \rightarrow 0.717$ \\
 & \ref{dgp:1d-2} & $90.24 \rightarrow 94.74$ & $5.074 \rightarrow 2.186$ & $0.805 \rightarrow 1.290$ \\
 & \ref{dgp:1d-3} & $92.11 \rightarrow 94.91$ & $3.268 \rightarrow 1.708$ & $0.516 \rightarrow 0.580$ \\
 & \ref{dgp:md-1} & $88.49 \rightarrow 95.33$ & $6.569 \rightarrow 2.201$ & $0.841 \rightarrow 2.100$ \\
 & \ref{dgp:md-2} & $88.42 \rightarrow 94.81$ & $6.714 \rightarrow 1.642$ & $2.281 \rightarrow 3.313$ \\
 & \ref{dgp:md-3} & $89.77 \rightarrow 95.07$ & $5.747 \rightarrow 2.054$ & $1.783 \rightarrow 2.706$ \\
 & \ref{dgp:bm-1} & $94.88 \rightarrow 94.85$ & $1.429 \rightarrow 1.502$ & $1.274 \rightarrow 1.275$ \\
 & \ref{dgp:bm-2} & $92.26 \rightarrow 95.01$ & $2.908 \rightarrow 1.594$ & $0.985 \rightarrow 1.137$ \\
 & \ref{dgp:bm-3} & $93.64 \rightarrow 94.85$ & $1.950 \rightarrow 1.588$ & $1.175 \rightarrow 1.245$ \\
\hline
TabICL & \ref{dgp:1d-1} & $92.77 \rightarrow 95.31$ & $4.313 \rightarrow 3.385$ & $0.643 \rightarrow 0.747$ \\
 & \ref{dgp:1d-2} & $91.98 \rightarrow 94.93$ & $3.915 \rightarrow 2.632$ & $0.970 \rightarrow 1.287$ \\
 & \ref{dgp:1d-3} & $92.74 \rightarrow 95.13$ & $4.058 \rightarrow 3.168$ & $0.617 \rightarrow 0.725$ \\
 & \ref{dgp:md-1} & $92.58 \rightarrow 95.43$ & $3.611 \rightarrow 2.644$ & $1.343 \rightarrow 1.901$ \\
 & \ref{dgp:md-2} & $93.57 \rightarrow 95.01$ & $2.222 \rightarrow 1.572$ & $2.833 \rightarrow 3.109$ \\
 & \ref{dgp:md-3} & $93.41 \rightarrow 95.17$ & $3.123 \rightarrow 2.266$ & $2.278 \rightarrow 2.609$ \\
 & \ref{dgp:bm-1} & $98.26 \rightarrow 94.91$ & $3.395 \rightarrow 2.046$ & $2.178 \rightarrow 1.476$ \\
 & \ref{dgp:bm-2} & $95.19 \rightarrow 95.18$ & $2.621 \rightarrow 2.543$ & $1.329 \rightarrow 1.340$ \\
 & \ref{dgp:bm-3} & $95.62 \rightarrow 94.92$ & $2.418 \rightarrow 2.467$ & $1.544 \rightarrow 1.462$ \\
\hline
\end{tabular}

\end{table}

Table~\ref{tab:synthetic-plugin512} compares Plug-in HPD(512) and C-USIM on all nine mechanisms, using the same reconstructed densities and conditional response draws in each run. Mean CCAD decreases from 4.067 to 1.848 percentage points for TabPFN and from 3.297 to 2.525 for TabICL when averaged equally across the nine mechanisms. Mean CCAD decreases in 16 of the 18 model--mechanism combinations, and mean prediction-set length increases in 16 of the 18 combinations. Figure~\ref{fig:synthetic-plugin512} places this comparison alongside Plug-in HPD(1536), showing how the two context sizes affect the uncalibrated regions and how calibration changes coverage and length at the 512 context.

\begin{figure}[!t]
\centering
\includegraphics[width=\linewidth]{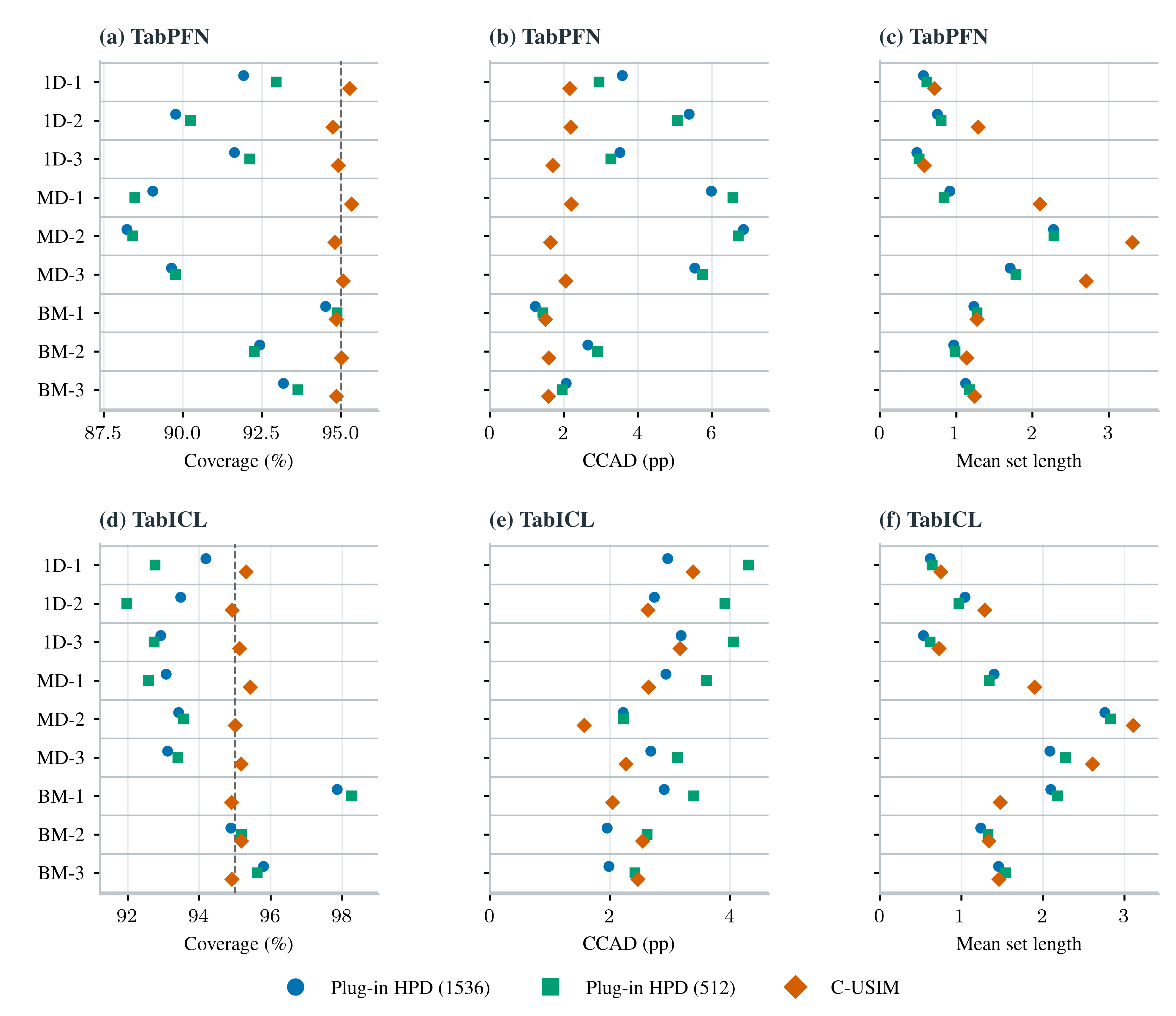}
\caption{Synthetic coverage, CCAD, and prediction-set length for all nine mechanisms and all three configurations. Points are means over ten seeds, not confidence intervals. Horizontal lines separate mechanisms. The two 512-context methods share predictive densities. Dashed lines mark 95\% coverage. All interval components contribute to length, excluding gaps.}
\label{fig:synthetic-plugin512}
\end{figure}

For BM-1--3, calibration lowers mean CCAD in four of the six case--model pairs; the exceptions are TabPFN BM-1 and TabICL BM-3. In TabICL BM-1, it brings mean coverage from 98.26\% to 94.91\% and reduces mean set length from 2.178 to 1.476 with the predictive densities held fixed. These results use the ten-seed Monte Carlo protocol above, separately from the 50-seed exact-CDF split-ratio sweep.

This comparison also clarifies the TabICL result discussed in Appendix~\ref{app:synthetic-results} on \ref{dgp:1d-1}: mean CCAD decreases from 4.313 to 3.385 percentage points after calibration, whereas Plug-in HPD(1536) has a lower CCAD of 2.961. Thus, the increase relative to the larger-context baseline cannot be attributed to threshold calibration alone.

\newpage

\subsubsection{Journal SJR}

\begin{table}[t]
\centering
\caption{Journal SJR using identical predictive densities from a 512-observation context, averaged over 200 seeds. Each arrow denotes Plug-in HPD(512) $\rightarrow$ C-USIM(512+1024). Both configurations use the same joint calibration/test query table; only C-USIM uses the 1,024 calibration responses. Group errors use the representative grouping. Coverage uses all 9,731 test rows; length uses the first 256 inputs in the fixed test order.}
\label{tab:sjr-plugin512}
\small
\begin{tabular}{lcc}
\hline
Measure & TabPFN & TabICL \\
\cline{2-3}
& \multicolumn{2}{c}{Plug-in HPD(512) $\rightarrow$ C-USIM} \\
\hline
Marginal coverage (\%) & $91.06 \rightarrow 94.83$ & $91.72 \rightarrow 94.87$ \\
Mean marginal gap (pp) & $3.964 \rightarrow 0.553$ & $3.431 \rightarrow 0.611$ \\
Mean group gap (pp) & $3.846 \rightarrow 1.760$ & $3.630 \rightarrow 1.856$ \\
CEC-X (pp) & $4.333 \rightarrow 1.662$ & $3.863 \rightarrow 1.707$ \\
Mean set length & $1.261 \rightarrow 1.453$ & $1.359 \rightarrow 1.540$ \\
\hline
\end{tabular}

\end{table}

\begin{figure}[t]
\centering
\includegraphics[width=\linewidth]{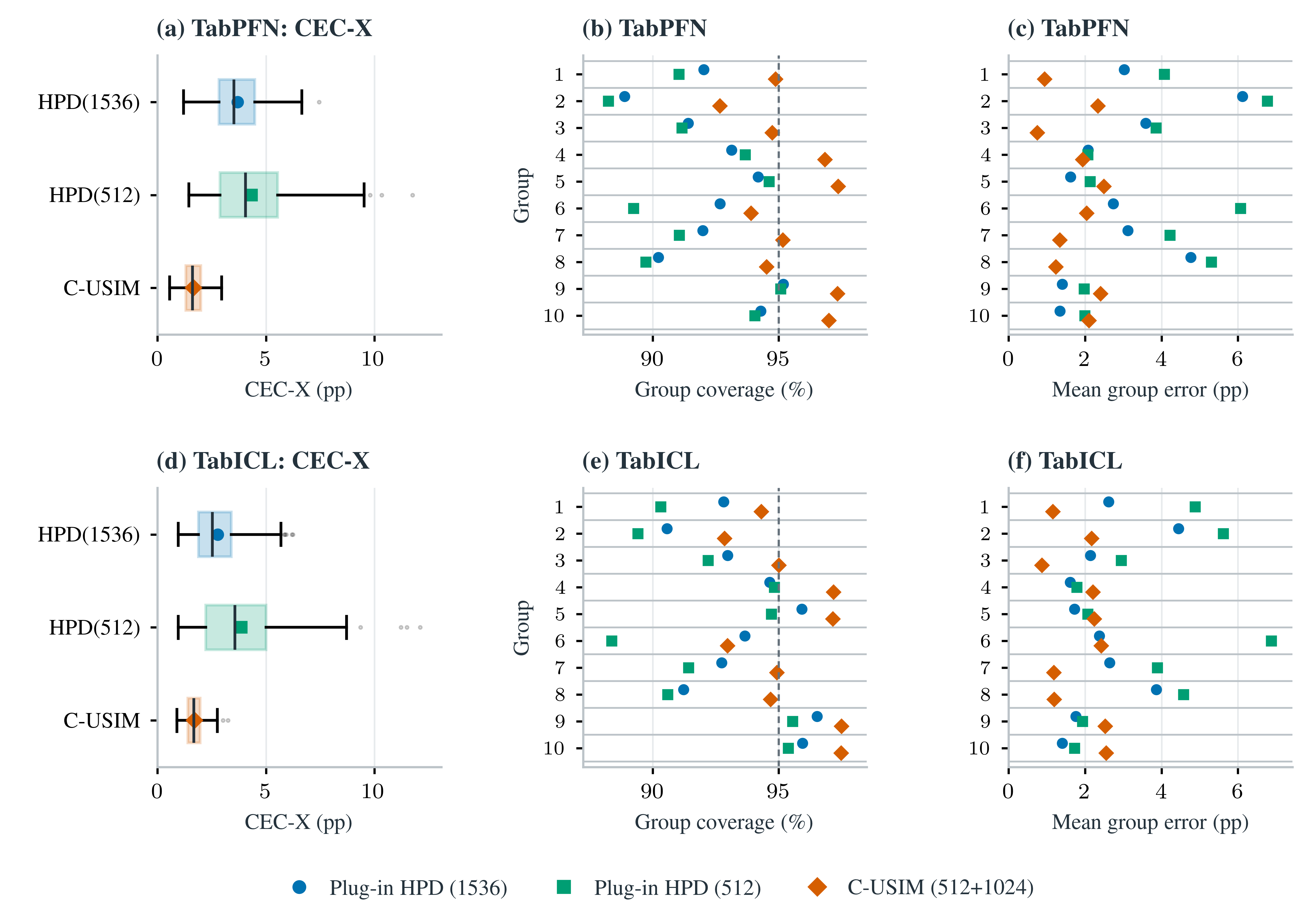}
\caption{Journal SJR with both plug-in baselines and C-USIM, using the representative grouping. (a,d) CEC-X over 200 seeds: boxes show quartiles and medians, whiskers extend to 1.5 IQR, all outliers are retained, and colored markers denote means. (b,e) Mean coverage in each group. (c,f) Mean within-seed absolute group error. Horizontal lines in the group plots separate groups. These summaries are not confidence intervals. Dashed lines mark 95\% coverage; all ten groups are retained.}
\label{fig:sjr-plugin512}
\end{figure}

Table~\ref{tab:sjr-plugin512} reports the results over 200 seeds per model. In the representative grouping, mean CEC-X decreases from 4.333 to 1.662 percentage points for TabPFN and from 3.863 to 1.707 for TabICL. Relative to Plug-in HPD(512), C-USIM lowers mean CEC-X in all 30 groupings for each model. In the representative grouping, mean absolute group error decreases from 3.846 to 1.760 percentage points for TabPFN and from 3.630 to 1.856 for TabICL. Mean error decreases in 7 of the ten groups for TabPFN and 6 for TabICL; 66.05\% and 61.50\% of seed--group pairs, respectively, move closer to the target. Across all groupings, these fractions range from 46.38\% to 70.00\% for TabPFN and from 43.26\% to 64.00\% for TabICL. All nonempty groups remain in the averages and Figure~\ref{fig:sjr-plugin512} shows all ten representative groups. The discussion in Appendix~\ref{app:sjr-results} also applies here. Mean prediction-set length increases from 1.261 to 1.453 for TabPFN and from 1.359 to 1.540 for TabICL, as reported in Table~\ref{tab:sjr-plugin512}.

\subsubsection{Additional real-world datasets}

Table~\ref{tab:additional-realdata-plugin512} compares Plug-in HPD(512) with C-USIM using identical predictive densities. Both methods use the same joint calibration/test table for each seed; only C-USIM uses the 1,024 calibration responses. On JP Anime, mean CEC-X decreases from 4.681 to 2.429 percentage points for TabPFN and from 3.103 to 2.521 for TabICL. On Allstate, the corresponding changes are 2.667 to 1.517 and 2.277 to 1.561. CEC-X decreases in all 30 groupings for each of the four dataset--model pairs, and mean prediction-set length increases in each pair.

The Allstate TabICL result illustrates the distinction between the two comparisons: calibrating the same 512-context densities lowers CEC-X, while comparison with Plug-in HPD(1536) also changes the predictor and gives the smaller, grouping-sensitive differences described in Appendix~\ref{app:additional-realdata}.

\begin{table}[htbp]
\centering
\caption{Additional real-world datasets using identical predictive densities from a 512-observation context, averaged over 50 seeds. Each arrow denotes Plug-in HPD(512) $\rightarrow$ C-USIM(512+1024). Only C-USIM uses the additional 1,024 calibration responses. Group measures use the representative grouping; response scales and test sizes follow Table~\ref{tab:additional-realdata-settings}.}
\label{tab:additional-realdata-plugin512}
\small
\begin{tabular}{llcc}
\hline
Dataset & Measure & TabPFN & TabICL \\
\cline{3-4}
& & \multicolumn{2}{c}{Plug-in HPD (512) $\rightarrow$ C-USIM} \\
\hline
JP Anime & Marginal coverage (\%) & $91.05 \rightarrow 94.96$ & $94.01 \rightarrow 94.96$ \\
 & Mean marginal gap (pp) & $3.974 \rightarrow 0.590$ & $1.853 \rightarrow 0.540$ \\
 & Mean group gap (pp) & $4.697 \rightarrow 2.770$ & $3.330 \rightarrow 2.837$ \\
 & CEC-X (pp) & $4.681 \rightarrow 2.429$ & $3.103 \rightarrow 2.521$ \\
 & Mean set length & $0.486 \rightarrow 0.565$ & $0.522 \rightarrow 0.539$ \\
\hline
Allstate & Marginal coverage (\%) & $92.43 \rightarrow 94.91$ & $93.03 \rightarrow 95.05$ \\
 & Mean marginal gap (pp) & $2.568 \rightarrow 0.601$ & $2.001 \rightarrow 0.540$ \\
 & Mean group gap (pp) & $2.646 \rightarrow 1.515$ & $2.279 \rightarrow 1.523$ \\
 & CEC-X (pp) & $2.667 \rightarrow 1.517$ & $2.277 \rightarrow 1.561$ \\
 & Mean set length & $5935.33 \rightarrow 6831.03$ & $6145.43 \rightarrow 6883.54$ \\
\hline
\end{tabular}

\end{table}

\begin{figure}[p]
\centering
\includegraphics[width=\linewidth,height=0.39\textheight,keepaspectratio]{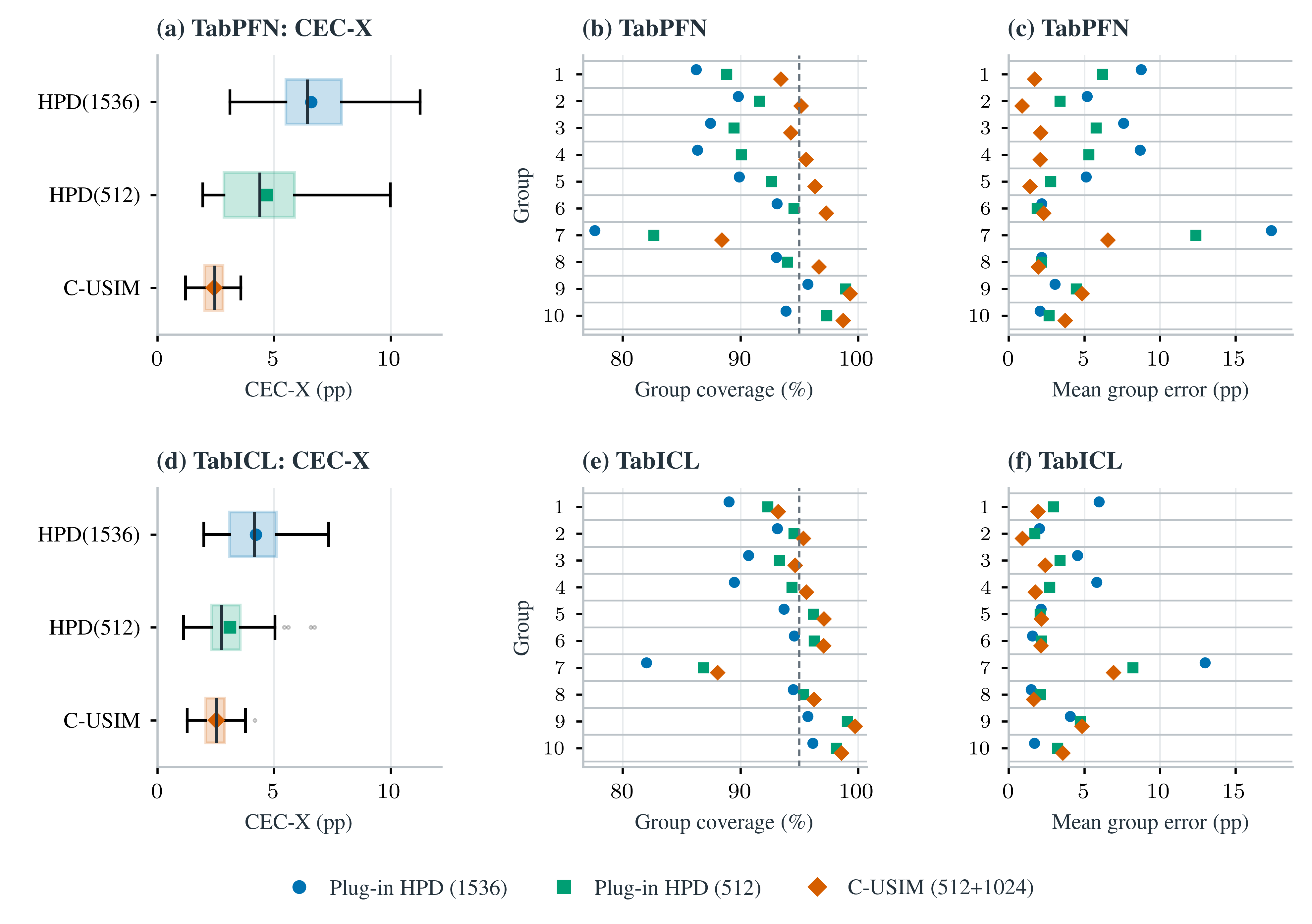}
\caption{JP Anime with both plug-in baselines and C-USIM, using the representative grouping. (a,d) CEC-X over 50 seeds: boxes show quartiles and medians, whiskers extend to 1.5 IQR, all outliers are retained, and colored markers denote means. (b,e) Mean group coverage. (c,f) Mean within-seed absolute group error. Horizontal lines separate groups; group 9 has 15 test observations. Dashed lines mark 95\% coverage. These summaries are not confidence intervals.}
\label{fig:jp-anime-plugin512}
\end{figure}

\begin{figure}[p]
\centering
\includegraphics[width=\linewidth,height=0.39\textheight,keepaspectratio]{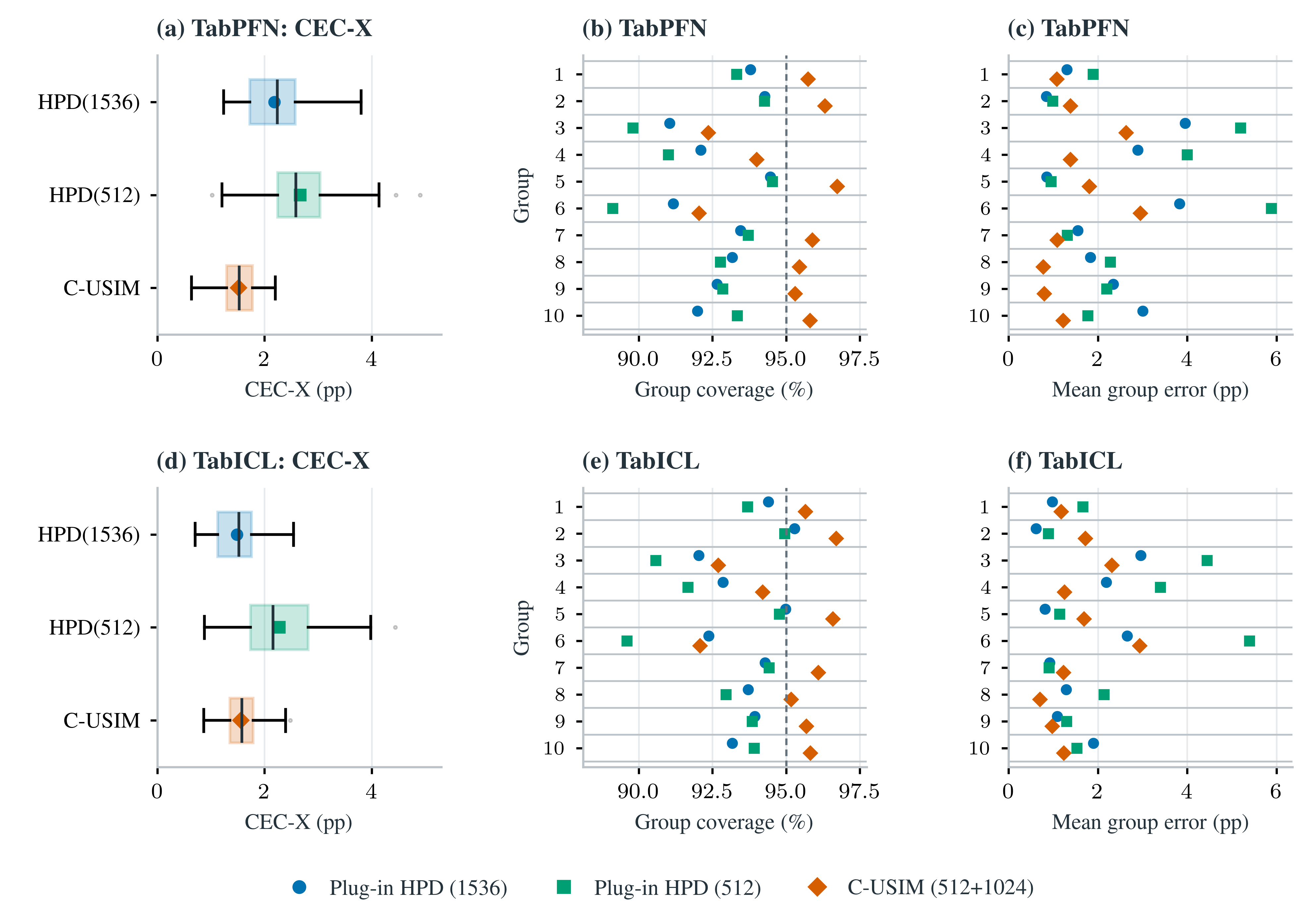}
\caption{Allstate Claims Severity with both plug-in baselines and C-USIM, using the representative grouping. (a,d) CEC-X over 50 seeds: boxes show quartiles and medians, whiskers extend to 1.5 IQR, all outliers are retained, and colored markers denote means. (b,e) Mean group coverage. (c,f) Mean within-seed absolute group error. Horizontal lines separate all ten groups; dashed lines mark 95\% coverage. These summaries are not confidence intervals.}
\label{fig:allstate-plugin512}
\end{figure}

\clearpage


\subsubsection{Comparison with Other CP Methods}
\label{app:calibrated-baselines}

Table~\ref{tab:synthetic-calibration-rev} compares C-USIM with CQR, CHR, and CTI on the nine synthetic settings, using both TabPFN and TabICL. Within each TFM and run, these four calibrated methods share identical predictive outputs and the same finite-histogram representation, with 512 context observations, 1,024 calibration observations, and 256 test observations. We use ten seeds per setting and a target coverage of 95\%. Plug-in HPD(1536) is included as a reference under the same total label budget; it uses all 1,536 labels as context and therefore does not share the calibrated methods' predictive distributions. All HPD(1536) runs use the same ten seeds and test inputs as the calibrated methods, with the context and calibration observations combined as context.

For CTI \citep{luo2025cti}, we use the authors' official implementation\footnote{\url{https://github.com/luo-lorry/CTI}} with $K=100$ interquantile intervals. We supply quantiles and finite support endpoints from the shared TFM histogram while retaining the implementation's score precision, interpolated calibration threshold, and interval selection. The reported metrics average over all test inputs, including empty predictions.

For all five methods, conditional coverage is estimated using the same 1,000 responses drawn from the known conditional distribution at each test input, following the Monte Carlo evaluation in Table~\ref{tab:synthetic-calibration}. Coverage and CCAD are computed across test inputs before averaging over seeds, and total set length includes all components while excluding the gaps between them.

Under this shared protocol, C-USIM achieves broadly comparable conditional coverage error and prediction-set length to CQR and CHR across the 1D and MD settings with both TFMs, with trade-offs depending on the setting. In BM-1--3, C-USIM produces shorter mean total prediction sets than CQR and CHR for both TFMs, while all three calibrated methods maintain mean coverage close to the 95\% target. The mean total set-length reductions across these six TFM--bimodal combinations are 15.4--56.4\% relative to CQR and 15.6--57.0\% relative to CHR. These results illustrate the benefit of allowing disconnected sets in these bimodal examples: C-USIM retains separated high-density regions while excluding the intervening low-density gaps that contiguous intervals spanning those regions must include, without additional model fitting or TFM inference.

The comparison with CTI, which also allows disconnected prediction sets, reveals substantial length inflation in two TabPFN settings at $K=100$. On \ref{dgp:1d-2} and \ref{dgp:md-1}, CTI's mean lengths are 5.588 and 22.112, respectively, compared with 1.290 and 2.100 for C-USIM---approximately 4.3 and 10.5 times as long. These two CTI entries are shown in bold in Table~\ref{tab:synthetic-calibration-rev}. In the bimodal settings, mean lengths are more comparable: CTI produces shorter sets with TabPFN, while C-USIM produces shorter sets with TabICL. C-USIM has lower mean CCAD in all six bimodal case--model pairs, with nearly equal values on TabPFN BM-1.

CTI's score construction and resolution requirements also introduce two limitations in our setting. First, its interquantile-width score is generally non-pivotal even in the oracle setting. C-USIM uses a density-rank score designed to be pivotal in the oracle setting: when the true conditional density is used and the induced density-score distribution is continuous, this construction yields the same uniform score distribution at every input; see Section~\ref{sec:conditional-coverage} for the motivation for pivotal scores. In contrast, CTI's interquantile-width score generally remains non-pivotal even with the true conditional quantiles, because its conditional distribution depends on the local response scale and shape. A shared cutoff therefore need not equalize conditional coverage across inputs. This distinction is particularly relevant to our heteroscedastic BM-3 example, where covariate-dependent mixture weights induce changes in conditional variance. Consistent with this motivation, C-USIM achieves lower mean CCAD than CTI on BM-3: 1.588 versus 1.629 pp with TabPFN and 2.467 versus 2.933 pp with TabICL, while both methods maintain mean marginal coverage close to 95\%.

Second, CTI introduces an additional resolution parameter, the number of interquantile intervals $K$, whose choice affected prediction-set lengths in our experiments, particularly when coarse tail intervals spanned wide support boundaries. Tuning this parameter across candidate values entails additional effort and computational cost for repeated calibration, prediction-set construction, and evaluation. C-USIM operates directly on the existing TFM density representation without this additional resolution parameter.

%

{
\setlength{\tabcolsep}{1pt}
\begin{table}
\caption{Comparison on nine synthetic settings, averaged over ten seeds per model and case. HPD denotes Plug-in HPD(1536); C-USIM, CHR, CQR, and CTI use 512 context and 1,024 calibration observations. CTI uses 100 interquantile intervals. Bold CTI lengths mark the inflated values discussed in the text. Within each metric, values are ordered as HPD, C-USIM, CHR, CQR, and CTI. All methods use the same 1,000 conditional Monte Carlo responses per test input. Coverage is in percent, CCAD in percentage points (pp), and length includes all prediction-set components. All nine settings are defined in Appendix~\ref{app:synthetic-dgps}.}
\label{tab:synthetic-calibration-rev}

\centering
\fontsize{8}{10}\selectfont
\begin{tabular}{ll|ccccc|ccccc|ccccc}
\hline
{\scriptsize Model} & {\scriptsize Example}
& \multicolumn{5}{c|}{\scriptsize Coverage (\%)}
& \multicolumn{5}{c|}{\scriptsize CCAD (pp)}
& \multicolumn{5}{c}{\scriptsize Mean length} \\
\cline{3-17}
& & {\scriptsize HPD} & {\scriptsize C-USIM} & {\scriptsize CHR} & {\scriptsize CQR} & {\scriptsize CTI}
& {\scriptsize HPD} & {\scriptsize C-USIM} & {\scriptsize CHR} & {\scriptsize CQR} & {\scriptsize CTI}
& {\scriptsize HPD} & {\scriptsize C-USIM} & {\scriptsize CHR} & {\scriptsize CQR} & {\scriptsize CTI} \\
\hline
TabPFN & \ref{dgp:1d-1}
& 91.920 & 95.277 & 95.475 & 95.158 & 95.263
& 3.572 & 2.159 & 2.196 & 3.294 & 1.845
& 0.571 & 0.717 & 1.195 & 0.692 & 0.658 \\
& \ref{dgp:1d-2}
& 89.775 & 94.744 & 94.713 & 94.629 & 94.760
& 5.384 & 2.186 & 2.469 & 2.210 & 2.429
& 0.754 & 1.290 & 1.372 & 1.205 & \textbf{5.588} \\
& \ref{dgp:1d-3}
& 91.629 & 94.909 & 95.454 & 95.165 & 95.226
& 3.513 & 1.708 & 2.015 & 1.726 & 2.154
& 0.483 & 0.580 & 0.778 & 0.656 & 0.561 \\
& \ref{dgp:md-1}
& 89.050 & 95.333 & 95.572 & 95.609 & 95.535
& 5.986 & 2.201 & 2.433 & 2.556 & 2.742
& 0.915 & 2.100 & 2.600 & 1.913 & \textbf{22.112} \\
& \ref{dgp:md-2}
& 88.237 & 94.807 & 95.014 & 94.975 & 94.812
& 6.854 & 1.642 & 1.903 & 1.829 & 2.192
& 2.274 & 3.313 & 3.379 & 3.102 & 3.135 \\
& \ref{dgp:md-3}
& 89.642 & 95.070 & 94.843 & 95.110 & 95.184
& 5.533 & 2.054 & 2.397 & 2.133 & 2.226
& 1.706 & 2.706 & 2.716 & 2.549 & 2.590 \\
\cline{2-17}
& \ref{dgp:bm-1}
& 94.510 & 94.851 & 95.231 & 95.276 & 95.039
& 1.227 & 1.502 & 1.639 & 1.518 & 1.503
& 1.231 & 1.275 & 2.964 & 2.923 & 1.239 \\
& \ref{dgp:bm-2}
& 92.428 & 95.014 & 95.222 & 95.184 & 95.107
& 2.642 & 1.594 & 1.573 & 1.667 & 1.627
& 0.969 & 1.137 & 1.835 & 1.636 & 1.077 \\
& \ref{dgp:bm-3}
& 93.176 & 94.853 & 94.982 & 94.785 & 94.879
& 2.057 & 1.588 & 1.714 & 1.652 & 1.629
& 1.122 & 1.245 & 1.791 & 1.681 & 1.197 \\
\hline
TabICL & \ref{dgp:1d-1}
& 94.190 & 95.313 & 95.308 & 95.288 & 95.142
& 2.961 & 3.385 & 3.393 & 3.638 & 3.779
& 0.617 & 0.747 & 0.746 & 0.751 & 0.755 \\
& \ref{dgp:1d-2}
& 93.480 & 94.927 & 94.928 & 94.802 & 94.700
& 2.741 & 2.632 & 2.631 & 2.460 & 2.511
& 1.044 & 1.287 & 1.287 & 1.257 & 1.256 \\
& \ref{dgp:1d-3}
& 92.929 & 95.133 & 95.103 & 95.319 & 95.248
& 3.186 & 3.168 & 3.172 & 2.615 & 2.956
& 0.535 & 0.725 & 0.747 & 0.757 & 0.662 \\
& \ref{dgp:md-1}
& 93.078 & 95.434 & 95.434 & 95.442 & 95.375
& 2.930 & 2.644 & 2.644 & 2.539 & 2.625
& 1.402 & 1.901 & 1.901 & 1.880 & 1.869 \\
& \ref{dgp:md-2}
& 93.424 & 95.008 & 95.008 & 94.905 & 94.818
& 2.220 & 1.572 & 1.572 & 1.592 & 1.708
& 2.759 & 3.109 & 3.109 & 3.087 & 3.086 \\
& \ref{dgp:md-3}
& 93.113 & 95.173 & 95.173 & 95.106 & 95.077
& 2.679 & 2.266 & 2.266 & 2.267 & 2.299
& 2.087 & 2.609 & 2.609 & 2.601 & 2.594 \\
\cline{2-17}
& \ref{dgp:bm-1}
& 97.857 & 94.907 & 95.154 & 95.120 & 95.008
& 2.902 & 2.046 & 2.167 & 2.108 & 2.569
& 2.097 & 1.476 & 2.945 & 2.943 & 1.669 \\
& \ref{dgp:bm-2}
& 94.881 & 95.183 & 95.192 & 95.259 & 95.058
& 1.953 & 2.543 & 2.889 & 2.506 & 2.812
& 1.238 & 1.340 & 1.655 & 1.648 & 1.507 \\
& \ref{dgp:bm-3}
& 95.798 & 94.922 & 95.029 & 94.984 & 94.806
& 1.980 & 2.467 & 2.678 & 2.536 & 2.933
& 1.457 & 1.462 & 1.731 & 1.727 & 1.520 \\
\hline
\end{tabular}
\end{table}
}

\end{document}